\documentclass[11pt,letterpaper]{article}

\usepackage{mathtools}
\usepackage{amssymb}
\usepackage{amsthm}
\DeclareMathOperator*{\argmin}{arg\,min}
\DeclareMathOperator*{\argmax}{arg\,max}

\usepackage{graphicx}
\usepackage{geometry}
\usepackage{cases}
\usepackage{mathtools}
\usepackage{bm}

\usepackage{booktabs}
\usepackage{multicol}
\usepackage{multirow}
\usepackage{makecell}
\usepackage{longtable}
\usepackage{pdflscape}
\usepackage{subfigure}
\usepackage{algorithm}
\usepackage{algorithmicx}

\usepackage{algpseudocode}

\usepackage{tfrupee}

\usepackage[hidelinks]{hyperref}
\usepackage{cleveref}

\usepackage{crossreftools}
\usepackage{natbib}

\usepackage{extarrows}

\newcommand{\RomanNum}[1]{\uppercase\expandafter{\romannumeral #1}}

\usepackage{dsfont}

\let\hat\widehat
\let\tilde\widetilde

\def\given{{\,|\,}}

\newcommand{\cD}{\mathcal{D}}

\newcommand{\cF}{\mathcal{F}}
\newcommand{\cG}{\mathcal{G}}
\newcommand{\cH}{\mathcal{H}}

\newcommand{\cN}{\mathcal{N}}
\newcommand{\cO}{\mathcal{O}}

\newcommand{\cS}{{\mathcal{S}}}

\newcommand{\cX}{\mathcal{X}}

\newcommand{\HH}{\mathbb{H}}

\newcommand{\PP}{\mathbb{P}}

\newcommand{\E}{{\mathbb E}}

\newcommand{\Tr}{\mathop{\text{tr}}\kern.2ex}

\DeclareMathOperator{\Var}{{\rm Var}}

\newcommand{\ud}{{\mathrm{d}}}

\newcommand{\RR}{\mathbb{R}}

\newcommand{\Sur}{{\text{Sur}}}

\theoremstyle{plain}
\newtheorem{theorem}{Theorem}[section]
\newtheorem{lemma}{Lemma}[section]
\newtheorem{proposition}[lemma]{Proposition}

\theoremstyle{definition}

\newtheorem{assumption}{Assumption}[section]

\theoremstyle{remark}

\crefname{assumption}{assumption}{assumptions}
\Crefname{assumption}{Assumption}{Assumptions}

\title{From Confounding to Learning: Dynamic Service Fee Pricing on Third-Party Platforms}

\author{
Rui Ai\thanks{MIT. Email: \texttt{\{ruiai, dslevi, fengzhu\}@mit.edu.}} \and
David Simchi-Levi\footnotemark[1] \and
Feng Zhu\footnotemark[1]
}

\begin{document}
\maketitle

\begin{abstract}
We study the pricing behavior of third-party platforms facing strategic agents. Assuming the platform is a revenue maximizer, it observes market features that generally affect demand. Since only transacted quantities and prices can be observed, this presents a general demand learning problem under confounding. Mathematically, we develop an algorithm with optimal regret of $\Tilde{\mathcal{O}}(\sqrt{T}\wedge\sigma_S^{-2})$. Our results reveal that supply-side noise fundamentally affects the learnability of demand, leading to a phase transition in regret. Technically, we show that non-i.i.d. actions can serve as instrumental variables for learning demand. We also propose a novel homeomorphic construction that allows us to establish estimation bounds without assuming star-shapedness, providing the first efficiency guarantee for learning demand with deep neural networks. 
Finally, we use simulations and offline counterfactuals from Talabat and Lyft data to illustrate the potential revenue implications of our approach.
\end{abstract}

\section{Introduction}
Third-party platforms now intermediate millions of transactions every day
across food delivery, ticket resale, ride-hailing, and electric-vehicle
charging~\citep{feldman2018service,pires2025much}. Each of these platforms earns revenue not from the underlying goods but from a \emph{service fee} levied on every transaction---the per-ticket service fee on Ticketmaster, the per-kWh surcharge on PlugShare (e.g., \textyen0.8/kWh at a Beijing station), the order-processing service fee on DoorDash. Setting this fee is a quintessential revenue-management problem: too high and price-sensitive consumers leave for a competitor, too low and the platform's margin erodes~\citep{gao2018platform,hagiu2009two,lin2020platform}.
Setting the fee well requires the platform to know how transacted quantity responds to price --- that is, to learn the demand curve --- yet, as we explain next, the very data the platform collects make this learning problem fundamentally harder than it appears.
\Cref{fig:example} illustrates the phenomenon across three concrete
markets.

\begin{figure}
    \centering
    \includegraphics[width=\linewidth]{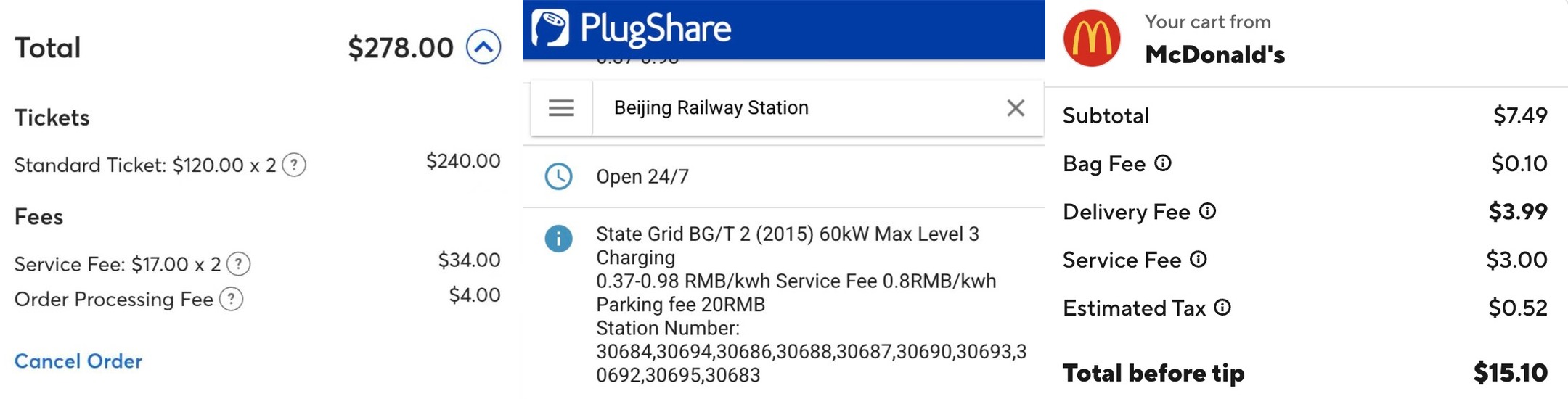}
    \caption{Examples of service fees in practice: ticket marketplaces, electric vehicle charging platforms, and food delivery services.}
    \label{fig:example}
\end{figure}

What makes the platform's pricing problem distinctive --- and intellectually
rich --- is an asymmetry of information. The platform observes the supply side perfectly: DoorDash sees the menus and prices that restaurants list, PlugShare sees the tiered electricity rates posted by station operators, Ticketmaster sees the ask prices of individual sellers. Yet it never observes the \emph{demand curve} of the consumers who eventually transact. All it can record is the realized price $P_t^e$ and quantity $Q_t^e$ at which the market clears, and these two quantities are jointly determined by supply, demand, and the service fee the platform itself just chose. Naively regressing $Q_t^e$ on $P_t^e$ to recover demand thus inherits a textbook \emph{confounding} problem: when supply is stable, the regression can even produce a positive slope,
contradicting basic price theory.

The standard remedy for confounded observational data is the use of
\emph{instrumental variables} (IV)~\citep{angrist1995two,angrist2001instrumental,newey2003instrumental}. In a platform market, however, valid instruments are scarce: there is no natural exogenous shock, no random encouragement design, no policy variation to exploit. Our central methodological observation is that the platform does not need to look outside its own decision loop --- the very \emph{action} it is trying to optimize, the service fee $a_t$ itself, can be used as an instrument. 
Because the platform commits to $a_t$ before the demand shock $\epsilon_{Dt}$ realizes, the conditional orthogonality between $a_t$ and $\epsilon_{Dt}$ holds even though $a_t$ is dynamic and chosen adaptively from the history. 
We call this the \emph{actions-as-instruments} (or self-instrumental) approach. The technical challenge for using a non-i.i.d., adaptive instrument is twofold: identification has to be argued through martingale concentration rather than i.i.d. tools, and the strength of the
instrument is itself a control variable that the platform must actively manage.

Once the platform is allowed to use its own actions as instruments, a second and more surprising phenomenon emerges: the very ability to learn demand depends on the structure of \emph{supply-side} randomness. Let $\sigma_S^2$ denote the variance of the supply shock. We prove that the worst-case regret over $T$ rounds scales as $\widetilde{\Theta}\bigl(\sqrt{T}\wedge\sigma_S^{-2}\bigr)$, exhibiting a sharp phase transition at the critical threshold $\sigma_S^2\simeq1/\sqrt{T}$. Below the threshold, supply fluctuations are too small for the demand curve to be identified from observed cutoff points; the instrument is weak, and the platform must deliberately inject artificial noise of magnitude $t^{-1/4}$ into the service fee to explore, paying $\widetilde{O}(\sqrt{T})$ regret. Above the threshold, the natural variation of supply alone is informative enough to pin down
demand, and regret becomes bounded by a logarithmic term in $T$. 
Operationally, the implication is striking: \emph{market volatility is information, not just noise.} Categories with naturally high supply variance (food delivery, ride-hailing) are far easier to learn from than seemingly stable ones (utilities, subscriptions). In stable markets, platforms may need some controlled source of identifying variation; when this variation is introduced through fees, it should be small, vanishing, auditable, and constrained by the platform's legal and consumer-protection obligations.

Two further considerations make this program both delicate and broadly
applicable. First, buyers are strategic: a consumer who anticipates that the platform is learning may misreport her demand in order to bias future fees in her favor. \citet{amin2013learning,aidynamic} show that when buyers are as patient as the platform ($\gamma=1$), no algorithm can achieve sublinear regret. In the empirically relevant regime $\gamma<1$, however, we show that updating the pricing policy only $O(\log T)$ times --- a low-switching design --- suffices: any future fee reduction obtained through misreporting is exponentially discounted, and the incentive to manipulate vanishes, leaving a regret with only polynomial dependence on $1/(1-\gamma)$. Second, modern platforms predict demand using high-capacity models such as deep neural networks (DNNs), whose function classes are typically not star-shaped, the standard regularity condition under which non-parametric IV theory operates~\citep{wainwright2019high,rakhlin2022mathematical,yu2022strategic}. We close this gap by constructing a homeomorphism into an auxiliary star-shaped class 
whose effective dimension is at most one larger than that of the original class,
so that our regret guarantee applies, for the first time, to ERM-with-IV demand learning over arbitrary multilayer perceptrons (MLP); the effective dimension grows quadratically in network width and linearly in depth, in line with the practitioner's intuition that deep but narrow networks are statistically efficient. 
Finally, we complement the theory with synthetic experiments and an offline counterfactual using micro-level transaction logs from \textsc{Talabat}, a major Egyptian food-delivery platform. Under a fixed-average-fee constraint, the exercise illustrates potential revenue upside from redistributing service fees across the day.

\paragraph{Contributions.} Our results can be summarized as follows.

\textbf{(i)} We propose \emph{actions as instruments}, a framework that turns the platform's own service fee into a self-instrumental variable for dynamic demand learning under confounding, and prove identification under a non-i.i.d., adaptively chosen instrument via martingale concentration.

\textbf{(ii)} We establish the minimax-optimal regret rate
$\widetilde{\Theta}(\sqrt{T}\wedge\sigma_S^{-2})$ and uncover a phase
transition at $\sigma_S^2\simeq 1/\sqrt T$ with matching upper and lower
bounds; interestingly, the lower bound of supply noise rests on a new construction that isolates the information channel carried by the demand noise, and is, to our knowledge, the first such construction in the dynamic pricing literature.

\textbf{(iii)} We develop a homeomorphism that bypasses the star-shapedness assumption, extending the ERM-with-IV framework to arbitrarily deep neural networks 
with at most a one-dimensional increase in effective dimension.


\textbf{(iv)} We use public food-delivery and ride-hailing data to construct offline counterfactuals illustrating the potential revenue upside of dynamic service-fee redistribution under fixed-average-fee constraints.

\subsection{Related Work}
\paragraph{Instrumental Variables in Machine Learning.} Instrumental variable is a powerful tool in econometrics \citep{angrist1995two,angrist2001instrumental,baltagi1986estimating}. \citet{nambiar2019dynamic} uses IV to learn the demand curve under model misspecification, inspired by \citet{petrin2010control,phillips2015effectiveness}.
In recent years, a series of articles have emerged that use the IV method to estimate machine learning parameters. \citet{liao2021instrumental,yu2022strategic} utilize instruments to learn a nearly-optimal policy via offline data. \citet{chen2022well,uehara2024future} adopt a non-parametric instrumental variable framework to do off-policy evaluation. \citet{chen2023estimating} considers a time-varying instrument and \citet{singh2020kernel} extends the kernel methods to scenarios with confounders. The work most closely related to our paper is \citet{aidynamic}, whereas it only considers a specific linear regression with a closed form. We extend the model to general function classes including deep neural networks, and prove the empirical success on learning deep features on a large-scale food-delivery dataset.

\paragraph{Pricing with Strategic Buyers.}
There is a burgeoning amount of literature on pricing with strategic buyers. \citet{amin2013learning} first proves that no algorithm can achieve sublinear regret when buyers are as patient as sellers. \citet{deng2020robust} considers pricing in dynamic mechanism design with less patient buyers. It obtains sublinear regret in contextual auctions. \citet{golrezaei2019dynamic} study optimal reserve design problem original from \citet{myerson1981optimal} facing strategic bidders while \citet{mohri2015revenue} considers a corresponding revenue optimization problem. \citet{golrezaei2023incentive} considers pricing with strategic buyers under non-parametric market noise, and \citet{ai2022reinforcement} extends to Markov decision process (MDP) pricing models. \citet{mohri2014learning,kanoria2020dynamic,epasto2018incentive,amin2014repeated} also study such issues under different information structures and depict different scenarios in real markets.

\paragraph{Demand Learning under Uncertainty.}
Demand learning is a hot topic in microeconomics, management science, and operations research~\citep{besbes2011minimax,cope2007bayesian,kleinberg2003value,carvalho2005learning,araman2009dynamic,lin2006dynamic,caro2007dynamic}. \citet{lobo2003pricing} considers a linear demand model and invents a ``price-dithering'' policy to add perturbation. \citet{den2014simultaneously} invents the controlled variance pricing idea and \citet{broder2012dynamic} scales exploration by adding $t^{-1/4}$ in the absence of strategic behavior and supply interaction, which has similar connotations to our approach. 
Besides, \citet{nambiar2019dynamic} considers confounders in demand learning in the absence of strategic buyers. 

\paragraph{Nonparametric Regression and Covering Number.} 
Nonparametric regression can partially explain the empirical success of deep learning, but it typically assumes that the identifiable class has star-shapedness~\citep{bartlett2005local,wainwright2019high}. \citet{yu2022strategic} uses algorithmic instruments to learn under a reinforcement learning framework, relying on the underlying MDP having a star-shaped structure. \citet{hu2025sample} extends this line of work to online reinforcement learning with confounders and proposes a nonparametric instrumental variable method, along with the lines of \citet{angrist1995two,angrist2001instrumental,newey2003instrumental,ai2003efficient}.
An important component of nonparametric regression is the covering number. Recently, several studies have characterized the complexity of deep neural networks~\citep{shen2023exploring,ou2024covering}, providing a statistical explanation for their strong performance.

\section{Problem Formulation}\label{sec:pre}
We study revenue management from the perspective of a third-party online platform that intermediates transactions between sellers and a strategic
buyer over $T$ rounds. The model has three actors---sellers, the platform,
and a buyer --- whose interactions we describe in turn.

\paragraph{Sellers.}
Over $T$ rounds, sellers first announce their supply before the platform
sets a fee or the buyer transacts. On PlugShare, the platform discloses the tiered electricity rates of each charging station; on Ticketmaster, sellers post the number of tickets they offer along with their ask prices.
Aggregating across sellers yields an (inverse) market supply curve, which we model in
inverse form as
\[
    P_{St} = \alpha_0 + \alpha_1 Q + \epsilon_{St},
\]
where $\epsilon_{St}$ is independent zero-mean, $\eta_S$-sub-Gaussian noise with variance $\sigma_S^2$. We assume sellers report supply truthfully: because their offers are committed before the platform and buyer act, no strategic manipulation of the supply curve confers any advantage. For example, on Ticketmaster, a seller must post the ask price before the ticket enters the market; once a buyer is willing to pay, the platform automatically matches the order. This committed-supply structure is the institutional norm in the intermediated marketplaces we study: sellers post binding offers through the platform's listing interface before any fee-mediated transaction can occur, and these offers cannot be revised once the matching engine receives them. The same pattern holds on PlugShare, where station operators register tiered electricity rates in advance, and on food-delivery platforms such as DoorDash, where merchants submit menus and prices through the portal before orders are routed.

\paragraph{The Platform.}
After observing supply, the platform sets a unit \emph{service fee}
$a_t \in \RR_{\ge 0}$ --- the per-kWh fee on PlugShare or the per-ticket fee on Ticketmaster. 
For instance, the charging-station operator on PlugShare posts tiered electricity rates while each kWh incurs the same fixed service fee on the platform side.
Following the behavior-based pricing literature~\citep{hart1988contract, salant1989inducing,
aviv2008optimal, aviv2019responsive}, we assume the platform commits to its fee structure before transactions occur, since commitment is known to
dominate adaptive deviation in revenue. The platform's revenue at time $t$
is $\Pi_t = a_t \cdot Q_t^e$, where $Q_t^e$ is the realized transacted
quantity defined below.

\paragraph{The Buyer.}
We study the buyer's market behavior that total purchase quantity is related to price. For example, when ticket prices are low, she tends to
purchase multiple tickets to go to games with family members, and when food is discounted, she tends to order more items in a single food delivery order. The platform never observes the buyer's demand curve, but does observe a feature vector $x_t \in \RR^d$ before setting $a_t$ --- weather and traffic on PlugShare, click-through rates and dwell time on Ticketmaster. Conditional on $x_t$, we model the buyer's (inverse) demand curve as
\[
    P_{Dt} = \beta Q + f(x_t) + \epsilon_{Dt},
\]
where $\epsilon_{Dt}$ is an independent zero-mean, strongly log-concave innovation with variance $\sigma_D^2$. The feature vector $x_t$ is interpreted broadly as the platform's demand-relevant information before pricing: any component of demand that is predictable from past observations, current market features, or the currently observed supply-side state is absorbed into $x_t$ and hence into $f(x_t)$. Thus $\epsilon_{Dt}$ represents only the residual demand innovation that is not predictable at the time the platform sets the service fee.
To rule out market-feasibility pathologies we assume $\alpha_1 \ge 0 > \beta$ and that
all realized prices and quantities are non-negative; otherwise we truncate at zero, which does not affect any of our analyses (see \Cref{app:market}).

We end our discussion of market participants with some {\it rationale for assuming a linear relationship between price and quantity}. First, given the relative stability of the market, transacted quantities fluctuate within a narrow range, and empirical observations reveal a strong locally linear relationship. In this context, the coefficients $\alpha_1$ and $\beta$ effectively capture local deviations. Second, \citet{besbes2015surprising} shows that under certain conditions, the consequences of model misspecification due to assuming linearity are significantly less severe than commonly expected. Third, the linear form is widely adopted in the literature on demand learning~\citep{bu2020online, qiang2016dynamic, keskin2014dynamic, zhu2020demands, bach2024adventures}, instrumental variables~\citep{nambiar2019dynamic, aidynamic}, and machine learning~\citep{bu2022context}. In \Cref{sec:ext} we discuss how our methodology can be extended to general nonlinear models, and in \Cref{sec:ride} we provide additional numerical experiments on a Lyft dataset to showcase the power of our methodology on nonlinear models.

\paragraph{Strategic Behavior and Market Observation.}
Unlike sellers, the buyer can manipulate her purchasing decisions to
influence future fees. Examples abound: a Ticketmaster customer who finds a price too high may wait for a discount; an Amazon shopper may delay until Black Friday even when willing to pay today. We capture this by letting the buyer act strategically against the platform.
She chooses a purchase trajectory that maximizes her expected cumulative discounted surplus, where for the $Q$-th unit, the buyer's payment is $P_{St}(Q) + a_t$ and her willingness to pay is $P_{Dt}(Q,x_t)$,
\[
    \E\!\left[\sum_{\tau=t}^{T} \gamma^{\tau-t}\, \Sur_\tau\right],
    \qquad
    \Sur_t = \int_0^{Q_t^e}\!\big(P_{Dt}(Q,x_t)-P_{St}(Q)-a_t\big)\,\ud Q,
\]
where $\gamma \in [0,1)$ is the buyer's discount rate. Her best response
induces a (round-$t$) behaved demand curve which we denote $P'_{Dt}(Q, x_t)$; absent strategic behavior yields $P'_{Dt} \equiv P_{Dt}$, and in general $P'_{Dt}$ is the optimal best response of the buyer to the platform's policy. The platform observes only the cutoff price $P_t^e$
and quantity $Q_t^e$, defined as the market-clearing pair
\begin{equation}\label{eq:equ}
    P_t^e \;:=\; P_{St}(Q_t^e)+a_t \;=\; P'_{Dt}(Q_t^e,x_t).
\end{equation}
Although the platform records every transaction price, only the highest one, i.e., the cutoff price at which the last unit clears, lies on both the supply and demand curves (adjusted by service fee, hereinafter) and therefore carries demand information; lower transaction prices reveal only weaker constraints already implied by the cutoff. For instance, a buyer purchased two low-price tickets at a certain price and one high-price ticket at the cutoff price $P_t^e$. However, the price of the low-price tickets only reveals that the buyer’s willingness to pay for quantity two exceeds that price. This information is subsumed by the buyer’s willingness to purchase the third ticket at $P_t^e$. From the supply perspective, all these transaction prices lie on the supply curve. Only the highest transaction price, say the cutoff price, is at the intersection of demand and supply, providing informative demand information.

The patience asymmetry between the platform and the buyer is essential: the platform is patient with a discount rate of 1 while the buyer discounts at $\gamma<1$. \citet{amin2013learning} proves that no algorithm achieves sublinear regret when the buyer is as patient as the platform from the technical perspective. The assumption is also empirically realistic: platforms optimize over long horizons of user retention, whereas individual buyers face immediate consumption needs (see also \citealp{drutsa2017horizon, golrezaei2019dynamic}).

\paragraph{Benchmark and Regret.}
As a benchmark, we use the clairvoyant revenue, in which the platform
observes both the supply and the expected demand curves. Under such an oracle, the buyer has no incentive to misreport, so $P'_{Dt}\equiv P_{Dt}$, and the optimal service fee is
\[
    a_t^\ast \;=\; \argmax_{a_t \ge 0} \;
    \E\!\big[\,a_t \cdot Q_t^e(P_{St},P_{Dt},a_t)\,\big],
\]
where the expectation is over $\epsilon_{Dt}$ given $x_t$. We define the
cumulative regret over $T$ rounds as
\begin{equation}\label{eq:regret}
    \text{Regret}(T) \;=\; \sum_{t=1}^{T}
    \E\!\left[\,a_t^\ast \cdot Q_t^e(P_{St},P_{Dt},a_t^\ast)
    \;-\; a_t \cdot Q_t^e(P_{St},P'_{Dt},a_t)\,\right].
\end{equation}
Minimizing regret thus requires two coupled tasks: identifying a fee
$a_t$ close to $a_t^\ast$, and structuring the policy so that the buyer's
best response $P'_{Dt}$ remains close to the truthful $P_{Dt}$.


\paragraph{Information available to the platform.}
At round $t$, the platform first observes the history
$\cH_{t-1}=\{h_\tau\}_{\tau=1}^{t-1}$, the current feature vector $x_t$, and the current supply curve $(\alpha_0+\epsilon_{St},\alpha_1)$.
The service fee is then chosen as a measurable function with respect to the sigma-algebra generated by $(\cH_{t-1},x_t,\alpha_0+\epsilon_{St},\alpha_1)$.
Only after this fee is committed does the residual demand innovation
$\epsilon_{Dt}$ realize and the platform observes
\[
h_t=((\alpha_0+\epsilon_{St},\alpha_1),x_t,P_t^e,Q_t^e).
\]
Service fee $a_t$ can depend arbitrarily on past observations, including past cutoff prices and quantities that were themselves affected by past demand shocks. What it rules out is only the use of information
that predicts the current residual innovation beyond what has already been included in $x_t$ and the pre-pricing market state. Hence, $a_t$ satisfies exogeneity only with respect to the contemporaneous demand shock $\epsilon_{Dt}$, while it may depend on past service fees in an arbitrarily complex way.
Throughout, the buyer is assumed to know the platform's learning policy in advance, as the platform can earn more revenue by committing to a pricing strategy~\citep{golrezaei2019dynamic}; if the policy randomizes, she observes the underlying mechanism but not its specific realization.

\paragraph{Preview of main results.}
For ease of exposition, we collect all formal regularity conditions in
\Cref{sec:regret}. The following informal statement summarizes the rate we
will prove; subsequent sections develop the methodology and full theorems. Throughout, $\cO(\cdot)$ and
$\widetilde{\cO}(\cdot)$ hide only absolute constants and logarithmic terms respectively.
\begin{theorem}[Informal preview]\label{thm:informal}
Under mild regularity conditions, our algorithm attains
\[
    \mathrm{Regret}(T) \;=\; \tilde{\cO}\!\left(\sqrt{T}\,\wedge\,\sigma_S^{-2}\right)
\]
up to logarithmic factors in $T$ and polynomial dependence in
$1/(1-\gamma)$, and this rate is \emph{minimax-optimal}. Moreover, the regret exhibits a sharp \emph{phase transition} at the critical level
$\sigma_S^2 \simeq 1/\sqrt{T}$: above the threshold, supply randomness
alone identifies demand and regret is logarithmic in $T$; below it, the
platform must explicitly explore, paying $\sqrt{T}$ regret.
\end{theorem}

\Cref{sec:method} develops the methodology behind this rate,
\Cref{sec:regret} states and proves the formal theorems, and
\Cref{sec:practice}--\ref{sec:exp} discuss practical extensions and
empirical validation.

\section{Methodology: From Confounding to Learning}
\label{sec:method}
This section bridges the formal model of \Cref{sec:pre} and the regret
guarantees of \Cref{sec:regret}. We first exhibit the confounding moment
condition that defeats naive regression (\Cref{sec:why-fail}), then
introduce the action-as-instrument idea that resolves it
(\Cref{sec:aai-idea}), develop the three engineering refinements needed
to make the method work in our dynamic, strategic setting
(\Cref{sec:tricks}), and finally state the algorithm
(\Cref{sec:algo}). For ease of exposition, we abstract away from buyer strategic behavior in \Cref{sec:why-fail}--\Cref{sec:aai-idea} and
reintroduce it in \Cref{sec:tricks}.

\subsection{Why naive regression fails: the confounding moment condition}
\label{sec:why-fail}

For the platform, the only observable and informative variables are the cutoff price $P_t^e$ and quantity $Q_t^e$, both of which are influenced by market noise. As a result, simple estimation methods---such as nonparametric regression---are susceptible to confounding, leading to biased estimates.
Specifically, we have the moment condition
\[
   \E\!\left[(P_t^e - \beta Q_t^e - f(x_t))\, Q_t^e\right] \;\neq\; 0,
\]
because $\epsilon_{Dt}$ is correlated with the transacted quantity
$Q_t^e$. To see why this is fatal, note that a standard nonparametric
regression would solve
\[
   \min_{\beta,\, f \in \cF}\;
   \sum_{t=1}^{T}\!\big(P_t^e - \beta Q_t^e - f(x_t)\big)^2,
\]
whose first-order condition with respect to $\beta$ reads
\[
   \sum_{t=1}^{T}\!\big(P_t^e - \beta Q_t^e - f(x_t)\big)\,Q_t^e \;=\; 0,
\]
in direct contradiction with the moment condition above. The bias is
not a finite-sample artifact: it does not vanish as $T \to \infty$.
Worse, when supply variation is small, the resulting estimate of
$\beta$ can even change sign, producing a positive demand slope that
contradicts elementary price theory.

\subsection{Action as instrument}
\label{sec:aai-idea}

To repair the moment condition, observe that the feature vector $x_t$
is uncorrelated with the demand shock $\epsilon_{Dt}$, since
$\epsilon_{Dt}$ captures exactly the variation in the outcome that is
left unexplained by $x_t$. This yields the structural equation
\begin{equation}\label{eq:structural}
   \E\!\left[(P_t^e - \beta Q_t^e - f(x_t))\, \tilde f(x_t)\right]
   \;=\; \E\!\left[\epsilon_{Dt}\, \tilde f(x_t)\right] \;=\; 0
   \qquad \forall\, \tilde f \in \cF.
\end{equation}
\Cref{eq:structural} provides identifying restrictions with dimension corresponding to that of $\cF$, which is one less than the dimension of the unknowns $(\beta, f)$.
We therefore need one additional moment condition.

\paragraph{The platform's own action as an instrument.}
The missing moment is supplied by the platform itself. The service fee
$a_t$ is chosen before the demand shock $\epsilon_{Dt}$ realizes; hence we have the \emph{one-step-ahead orthogonality} that 
\begin{equation}\label{eq:IV}
   \E\!\left[(P_t^e - \beta Q_t^e - f(x_t))\, a_t\right]\;=\; \E\!\left[\epsilon_{Dt}\, a_t\right] \;=\;  \E\!\left[
      a_t\,
      \E\!\left[\epsilon_{Dt}\mid
      \cH_{t-1},x_t,\alpha_0+\epsilon_{St},\alpha_1
      \right]
   \right]\;=\; 0,
\end{equation}
which, together with \Cref{eq:structural}, identifies both $\beta$ and
$f$. We refer to this as the \emph{action-as-instrument}
(\textsc{AaI}) or \emph{self-instrumental} identification. Crucially, $a_t$ is \emph{not} a classical IV: it is endogenously chosen by the
platform, adaptively as a function of the past, and therefore neither independent nor identical.

\paragraph{Why the non-i.i.d.\ instrument still works.}
Three obstacles in our setting threaten the validity of $a_t$ as an instrument, and each must be---and is---addressed. First, the policy generating $a_t$ must itself evolve, since the platform's goal is to find the optimal fee, so subsequent fees depend on previous ones; we control this dependence using a martingale concentration argument with one-step-ahead orthogonality in place of i.i.d. tools. Second, when supply noise is small the estimator suffers from weak instruments, so we inject artificial exploration noise into $a_t$ whose magnitude must vanish along the
trajectory in order not to inflate regret (\Cref{sec:tricks}). Third, the buyer's strategic behavior couples $a_t$ to the cutoff price through misreporting, so $a_t$ acts as a mediated rather
than purely exogenous variable; we neutralize this by updating the policy only $\cO(\log T)$ times, which limits cross-episode contamination and shrinks the within-episode strategic distortion to a level controlled by \Cref{prop:bound_strategy}. Together, these three considerations turn the platform's own decision variable into a usable instrument under mild conditions.

\subsection{Three engineering refinements}
\label{sec:tricks}
\Cref{eq:structural,eq:IV} pinpoint identification, but a working algorithm needs three further refinements: (i) a
doubly-robust correction to control estimator variance and accelerate convergence,
(ii) adaptive exploration with vanishing magnitude to handle weak
instruments without inflating regret, and (iii) a low-switching update schedule to defeat strategic misreporting.

\paragraph{Doubly-robust correction.}
Building on doubly-robust machine learning~\citep{imbens2015causal,chernozhukov2018double}, we replace
\Cref{eq:IV} by the more robust structural equation
\begin{equation}\label{eq:IV_robust}
   \E\!\left[(P_t^e - \beta Q_t^e - f(x_t))\,
            (a_t - \E[a_t \given x_t])\right] \;=\; 0.
\end{equation}
As long as either $P_t^e - \beta Q_t^e - f(x_t)$ or
$a_t - \E[a_t \given x_t]$ is accurately estimated,
\Cref{eq:IV_robust} compensates for first-order errors in the other.
Combining empirical risk minimization with this IV restriction yields
the iterates
\begin{align*}
   \hat\beta &\;\leftarrow\;
      \frac{\E[(P_t^e - f(x_t))(a_t - \E[a_t \given x_t])]}
           {\E[Q_t^e\,(a_t - \E[a_t \given x_t])]},\\[2pt]
   \hat f &\;\leftarrow\;
      \argmin_{f \in \cF}\,
      \E\!\left[(P_t^e - \beta Q_t^e - f(x_t))^2\right].
\end{align*}
We initialize $\hat f = 0$ and alternate the two updates using
empirical expectations. Ablation experiments in \Cref{app:ablation}
confirm that the doubly-robust step yields estimators with markedly
smaller variance than the plain IV update.

\paragraph{Adaptive exploration with vanishing noise.}
The estimator $\hat\beta$ has finite variance only when
$a_t - \E[a_t \given x_t]$ has nonzero second moment---i.e., when supply noise alone provides sufficient exploration. When supply is too
stable (in particular, when $\sigma_S^2 = 0$ all observed cutoff points lie on the supply curve), the platform itself must explore by
perturbing the fee. A naive constant perturbation of order $\Theta(1)$
would contribute $\Theta(T)$ to cumulative regret, while no perturbation leaves the IV weak. The right balance is a vanishing
perturbation of magnitude $t^{-1/4}$, which we will show is the unique
rate that simultaneously controls weak-IV variance and exploration
loss; this trade-off gives rise to the $\sqrt{T}$ branch of the phase
transition analyzed in \Cref{sec:regret}.

\paragraph{Low switching for strategic buyers.}
A strategic buyer can reduce her current surplus by misreporting demand in order to bias future fees in her favor. The misreport pays off only if the platform's fee policy reacts quickly to the most recent
observations. We therefore update the parameters $(\hat\beta, \hat f)$ only $\cO(\log T)$ times over $T$ rounds, doubling the length of the historical window each time. We prove that the future surplus benefit of any present misreport is exponentially attenuated; combined with the rare updates, the strategic distortion of $P'_{Dt}$ around $P_{Dt}$ remains controlled in every episode.

\subsection{The \texttt{AaI} and \texttt{AAaI} algorithms}
\label{sec:algo}

We now combine the three refinements above into the algorithm
\texttt{AaI} (Action-as-Instrument, \Cref{alg:AaaIV}), which we will
analyze in \Cref{sec:regret}. \texttt{AaI} proceeds in doubling
episodes $m = 1, 2, \ldots$. At the beginning of each episode, it
re-estimates $(\hat\alpha_0, \hat\beta, \hat f)$ from the dataset collected in the previous episode via the doubly-robust IV update; within the episode, it chooses the service fee through the subroutine \texttt{Act} (\Cref{alg:act}), which augments the greedy revenue maximizer
$\hat a^\ast$ with a Gaussian perturbation of variance
$1/\sqrt{|\cD|+1}$. The perturbation is active only when the supply
side is degenerate; whether to activate it is decided either by the
known value of $\sigma_S^2$ or, when $\sigma_S^2$ is unknown in
practice, by the wrapper algorithm \texttt{AAaI} described next.

\begin{algorithm}[htbp]
   \caption{\texttt{AaI} Algorithm}
   \label{alg:AaaIV}
\begin{algorithmic}
   \State {\bfseries Input:} $T$.
   \State {\bfseries Initialization:} $\cD\gets\emptyset$, $t\gets 0$ and $(\hat \alpha_0,\hat\beta,\hat f)\gets(0,0,0)$.
   \For{episode $m=1, 2, \cdots$}
        \State Estimate unknowns:
        \begin{equation*}
        \begin{aligned}
            \hat\alpha_0 & = \frac{1}{|\cD|}\sum_{i=1}^{|\cD|}\left(\alpha_0+\epsilon_{Si}\right)\\
            \hat\beta & =\frac{\sum_{i=1}^{|\cD|}(P_{i}^e-\hat f(x_i))(\hat f(x_i)-2a_i-\hat\alpha_0)}{\sum_{i=1}^{|\cD|}Q_i^e(\hat f(x_i)-2a_i-\hat\alpha_0)} \\
            \hat f & = \frac{1}{|\cD|}\argmin_{f\in\cF}\sum_{i=1}^{|\cD|}(P_{i}^e-\hat\beta Q_i^e-f(x_i))^2
        \end{aligned}
        \end{equation*}
        \State Reinitialize dataset $\cD\gets\emptyset$.
        \For{$\tau=t+1$ {\bfseries to} $t+2^m$}
        \State Observe the supply curve $P_{S\tau}=\alpha_0+\alpha_1 Q+\epsilon_{S\tau}$ and feature $x_\tau$.
        \State Form estimate for the demand curve $\hat P_{D\tau}=\hat f(x_\tau)+\hat \beta Q$.
        \State Select the action $a_\tau\leftarrow\texttt{Act}(P_{S\tau},\hat P_{D})$ using \Cref{alg:act}.
        \State Observe $o_\tau=(P^e_\tau,Q^e_\tau,x_\tau,\alpha_0+\epsilon_{S\tau},a_\tau)$.
        \State Update $\cD\leftarrow\cD\cup o_\tau$.
   \EndFor
   \State Renumber the index in the dataset $\cD$.
   \State Update round index: $t\leftarrow t+2^m$.
   \EndFor
\end{algorithmic}
\end{algorithm}

\begin{algorithm}[htbp]
   \caption{\texttt{Act} Algorithm}
   \label{alg:act}
\begin{algorithmic}
    \State {\bfseries Input:} $P_S$ and $P_D$.
    \State Calculate $\hat a^*\leftarrow\argmax_{a\geq 0}\ a\cdot \hat Q^e(a)$, where $\hat Q^e(a)$ is decided by $P_{S}(\hat Q^e)+a=\hat P_D(\hat Q^e)$.
    \State Generate an independent noise term $\epsilon\sim\cN(0,\frac{1}{\sqrt{\text{size}(\cD)+1}})$.
    \State Set service fee $a\leftarrow \hat a^*+\epsilon\cdot\mathbf 1\{\sigma_S=0\}$, when $\sigma_S^2$ is known (\Cref{sec:upper+root});
    \State Set service fee $a\leftarrow \hat a^*+\epsilon\cdot\mathbf 1\{\HH=\HH_0\}$, when $\sigma_S^2$ is unknown (\Cref{sec:adaptive}).
    \State {\bfseries Output:} $a$.
\end{algorithmic}
\end{algorithm}

\paragraph{The wrapper algorithm \texttt{AAaI}.}
In practice, the supply variance $\sigma_S^2$ is rarely known a priori.
We therefore wrap \texttt{AaI} in a unified algorithm \texttt{AAaI}
(Adaptive-AaI, \Cref{alg:Ada-AaaIV}). \texttt{AAaI} spends
$T_0 \eqsim \Theta(\log T)$ rounds drawing random fees in order to
test the hypothesis $\HH_0:\sigma_S^2 \lesssim 1/\sqrt{T}$ versus
$\HH_1:\sigma_S^2 \gtrsim 1/\sqrt{T}$, then dispatches the remaining
$T - T_0$ rounds to \texttt{AaI} with the perturbation activated only
when $\HH = \HH_0$. Because the test budget is logarithmic in $T$, the
overhead of estimating $\sigma_S^2$ adds at most an additive
$\Theta(\log T)$ regret on top of the oracle that knows $\sigma_S^2$
in advance. The formal guarantees for \texttt{AaI} and \texttt{AAaI},
together with their matching lower bounds and the resulting phase
transition, are stated in \Cref{sec:regret}.

\begin{algorithm}[htbp]
   \caption{\texttt{AAaI} Algorithm}
   \label{alg:Ada-AaaIV}
\begin{algorithmic}
   \State {\bfseries Input:} $T$.
   \For{$t=1$ {\bfseries to} $T_0\eqsim\Theta(\log T)$}
   \State Choose a random service fee and observe $\alpha_0+\epsilon_{St}$.
   \EndFor
   \State {\bfseries Hypothesis Test:} $\HH_0:\sigma_S^2\lesssim\cO(\frac{1}{\sqrt{T}})$ and $\HH_1:\sigma_S^2\gtrsim\Omega(\frac{1}{\sqrt{T}})$, denoting the result as $\HH$.
   \State Conduct $\texttt{AaI}(T-T_0)$.
\end{algorithmic}
\end{algorithm}

\section{Main Results: Regret Bounds and Phase Transition}
\label{sec:regret}

This section makes the rate previewed in \Cref{thm:informal} formal. We
organize the analysis by whether the supply variance $\sigma_S^2$ is
available to the platform: \Cref{sec:upper+constant} treats the known-$\sigma_S$ regime, where matching upper and lower bounds yield
$\widetilde{\Theta}(1)$ regret with positive supply noise and
$\widetilde{\Theta}(\sqrt{T})$ regret in the absence of supply noise;
\Cref{sec:adaptive} treats the unknown-$\sigma_S$ regime, in which the two rates connect through a sharp phase transition at
$\sigma_S^2 \simeq 1/\sqrt{T}$, accompanied by a matching minimax
lower bound that exploits the information channel of demand noise.

Throughout this section, we assume without loss of generality that the
horizon $T$ is known: if not, the standard doubling
trick~\citep{auer2002adaptive,besson2018doubling} converts any
horizon-aware bound into a horizon-free bound of the same order.

\paragraph{General feature vectors.}
With the rapid advancement of machine learning, modern platforms model
demand using high-capacity, high-dimensional predictors. We allow $f$
to lie in an arbitrary function class $\cF$ subject only to a covering
number bound, following standard machine-learning
conventions~\citep{jin2020provably,kong2021online,foster2021statistical,jin2021bellman}.
For any $f$ we use $\|f\|_\infty=\sup_{x\in\cX}|f(x)|$ and, for fixed
$x_1,\ldots,x_n$, the empirical norm
$\|f\|_n=\sqrt{\frac{1}{n}\sum_{i=1}^n f(x_i)^2}$ together with the
corresponding inner product $\langle\cdot,\cdot\rangle_n$. Vectors are
equipped with the usual $\ell_2$ and $\ell_\infty$ norms.

\begin{assumption}[Boundedness]\label{ass:context}
$x_t\in\cX\subseteq[0,1]^d$ is i.i.d.\ generated. The function class
$\cF\subseteq\{f:\RR^d\to[-B,B]\}$ has covering number
$N(\cF,\|\cdot\|_\infty,\epsilon)\lesssim\cO((1/\epsilon)^{\textit{dim}})$,
and we call the smallest such $\textit{dim}$ the \emph{effective
dimension}.
\end{assumption}

We do not impose convexity or any form of star-shapedness on $\cF$; in particular, our framework accommodates deep neural networks
\citep{bedi2019deep, law2019tourism}, as we develop in \Cref{sec:practice}. 

\subsection{Regret bounds when supply variance is known}
\label{sec:upper+constant}
\label{sec:upper+root}

We first analyse the platform's performance when the supply variance
$\sigma_S^2$ is available. Two qualitatively different regimes appear,
separated by whether or not natural supply noise alone provides enough
exploration to identify demand.

\paragraph{Positive supply variance: logarithmic regret.}
When $\sigma_S^2>0$, the supply curve fluctuates randomly around
$\alpha_0+\alpha_1 Q$, and these fluctuations alone provide a
non-trivial instrument for the moment condition \Cref{eq:IV_robust}.
The platform need not perturb its fee, and \Cref{alg:AaaIV} achieves
logarithmic-order regret in $T$.

\begin{theorem}[$\sigma_S>0$]\label{thm:partially}
Under \Cref{ass:context}, with probability at least $1-1/T$,
\Cref{alg:AaaIV} achieves regret at most
\[
\cO\!\left(
   dim\log(dim) + dim\log^2 T
   + \frac{\eta_S^2\log^3 T}{\sigma_S^4}
   + \frac{\eta_S^2\log^2 T}{\sigma_S^4(1-\gamma)^2}
   + \frac{\eta_S^2\log T}{\sigma_S^4(1-\gamma)^3}
\right)
\]
against any buyer with discount rate $\gamma\in[0,1)$.
\end{theorem}

The regret depends only logarithmically on the horizon $T$, indicating
that the platform converges to a near-optimal fee with vanishing
per-round loss. On the demand side, regret grows linearly in the
effective dimension $\textit{dim}$; this rate is
unavoidable, since even for a linear model with $\textit{dim}=d$ the
minimax $L_2$ prediction error is at least
$d/T$~\citep{yang1999information,chernozhukov2024applied}.

Regarding the buyer's discount factor $\gamma$, our market model is
unbounded so the $1/(1-\gamma)$ lower bound of
\citet{amin2013learning} does not apply; instead, regret exhibits only
a polynomial dependence on $1/(1-\gamma)$. In practice
\citet{yao2012determining} estimates $\gamma\approx 0.9$ on a weekly
basis; combined with the logarithmic dependence on $T$, the loss in revenue from buyer strategic behavior remains modest.

\paragraph{Vanishing supply variance: root-$T$ regret.}
When $\sigma_S^2=0$ a more serious obstacle emerges:
$a_t^\ast - \E[a_t^\ast\given x_t]=0$, so the denominator of the IV
estimator $\hat\beta$ collapses and its variance diverges.
Equivalently, all observed cutoff points lie exactly on the supply
curve, so demand cannot be identified without further exploration.
\texttt{Act} (\Cref{alg:act}) injects mean-zero Gaussian noise of
variance $1/\sqrt{|\cD|+1}$ into the service fee---of magnitude
$t^{-1/4}$ along the trajectory---trading off exploration loss against
weak-instrument variance. The resulting bound is sublinear in $T$.

\begin{theorem}[$\sigma_S=0$]\label{thm:no_noise}
Under \Cref{ass:context}, with probability at least $1-1/T$,
\Cref{alg:AaaIV} achieves regret at most
\[
\cO\!\left(
   dim\log(dim) + dim\log^2 T
   + \log^2 T\sqrt{T}
   + \frac{\log T\sqrt{T}}{1-\gamma}
   + \frac{\log^2 T}{(1-\gamma)^2}
   + \frac{\log T}{(1-\gamma)^3}
\right)
\]
against any buyer with discount rate $\gamma\in[0,1)$.
\end{theorem}

The dominant term is now $\widetilde{\cO}(\sqrt{T})$, a sharp contrast
to the logarithmic dependence of \Cref{thm:partially}. Comparing the two bounds reveals a central message of this paper: \emph{supply-side noise can in fact facilitate the learning of demand}. The bound also
contains an $\widetilde{\cO}(\sqrt{T}/(1-\gamma))$ term arising from
strategic behavior; in our proof, we additionally show that whenever a
uniform upper bound $\bar Q$ is available on the realized quantity, this term is dominated by the others, leaving the influence of buyer misreporting moderate.

A natural question is whether the $\sqrt{T}$ rate is tight, or whether
some clever algorithm could recover the logarithmic rate of the
positive-variance case. The next theorem rules this out: the
$\sqrt{T}$ rate is fundamental to the noise-free supply regime.

\begin{theorem}\label{thm:lower_no_noise}
When $\sigma_S^2=0$, any algorithm incurs worst-case expected regret
at least $\Omega(\sqrt{T})$.
\end{theorem}

In summary, when $\sigma_S^2$ is known the minimax regret is
$\widetilde{\Theta}(1)$ if $\sigma_S^2>0$ and
$\widetilde{\Theta}(\sqrt{T})$ if $\sigma_S^2=0$. The two regimes are
qualitatively different but share the same algorithmic skeleton,
distinguished only by whether \texttt{Act} activates the perturbation.
We next examine how regret interpolates between them when
$\sigma_S^2$ is not known to the platform.

\subsection{Phase transition under unknown supply variance}
\label{sec:adaptive}

\Cref{thm:partially}--\Cref{thm:lower_no_noise} together identify two
regimes: at one extreme ($\sigma_S^2$ a positive constant) regret is
logarithmic, at the other extreme ($\sigma_S^2=0$) regret is
$\sqrt{T}$. In practice, the platform rarely knows $\sigma_S^2$ exactly,
and the natural question is how regret scales with $\sigma_S^2$ along
the continuum between these two limits. We show that this dependence is
governed by a single critical scale $\sigma_S^2 \simeq 1/\sqrt{T}$,
producing a sharp phase transition.

To rule out pathological cases driven by heavy-tailed sub-Gaussian
parameters, we impose---only for the results in this subsection---a
mild Chernoff-type condition tying $\eta_S$ to $\sigma_S$.
It merely says multiplicative rescaling of the noise does not affect the dependence
between regret and supply variance, so the restriction is innocuous.

\begin{assumption}[Chernoff bound]\label{ass:eta}
$\eta_S \lesssim \cO(\sigma_S)$.
\end{assumption}

\Cref{ass:eta} is satisfied by most common distributions, including
Gaussian and bounded distributions; characterizing the precise regret
dependence on $\eta_S$ is a promising direction for future work. Under
this assumption, our wrapper algorithm \texttt{AAaI}
(\Cref{alg:Ada-AaaIV}) attains the optimal rate without prior
knowledge of $\sigma_S^2$.

\begin{theorem}[Unknown $\sigma_S$]\label{thm:upper_ada}
Under \Cref{ass:context} and \Cref{ass:eta}, with probability at least
$1-1/T$, \Cref{alg:Ada-AaaIV} achieves regret at most
\[
\cO\!\left(\sqrt{T}\,\log^2 T \;\wedge\; \frac{\log^3 T}{\sigma_S^2}\right).
\]
\end{theorem}

The rate exhibits a distinct phase transition. When
$\sigma_S^2 \lesssim \cO(1/\sqrt{T})$, the regret scales as
$\widetilde{\cO}(\sqrt{T})$, and the artificial perturbation of
magnitude $t^{-1/4}$ is essential to balance the two terms in the
bound. When $\sigma_S^2$ exceeds the critical threshold $1/\sqrt{T}$, natural supply fluctuations alone provide sufficient exploration:
artificial noise becomes unnecessary and regret decreases hyperbolically
as $1/\sigma_S^2$. One might wonder why \emph{larger} supply noise does
not worsen the platform's revenue; the reason is that the platform
observes the supply realization \emph{before} committing to its fee, so
it can adapt to the realized $\epsilon_{St}$ rather than pay for it.

The $\sqrt{T} \wedge 1/\sigma_S^2$ rate of \Cref{thm:upper_ada} matches
an information-theoretic lower bound, ignoring logarithmic factors.

\begin{theorem}\label{thm:lower}
Assume Gaussian supply noise. The worst-case expected regret of any
algorithm is at least
\[
\Omega\!\left(\sqrt{T} \;\wedge\; \frac{\log T}{\sigma_S^2(1+\max\{0,\log(1/\sigma_S)\})}\right).
\]
\end{theorem}

\paragraph{Demand noise as a hidden information channel.}
The construction behind \Cref{thm:lower} departs from standard practice
in the dynamic-pricing lower-bound literature. Traditional KL-divergence
and Van Trees constructions modify the expected demand $f$ while keeping
the demand noise $\epsilon_{Dt}$ unchanged. In our setting, this strategy
fails: the demand noise leaks information about $f$ through the variance
of the cutoff price $P_t^e$ and quantity $Q_t^e$, so any change in $f$
is detectable second-order even when first-order means coincide. We
isolate this leakage by simultaneously adjusting $\beta$ and the
magnitude of $\epsilon_{Dt}$, producing two instances whose both first-order and second-order differences vanish. The
resulting indistinguishability gives the lower bound. To our knowledge,
this is the first time the information contained in demand noise has
been used as a tool in dynamic-pricing lower-bound proofs; we view it
as a methodological contribution in itself, complementary to the
algorithmic contributions of the upper bounds.

\paragraph{The phase transition picture.}
\Cref{thm:upper_ada,thm:lower} together establish that the minimax
regret is
\[
\widetilde{\Theta}\!\left(\sqrt{T} \,\wedge\, \sigma_S^{-2}\right),
\]
exhibiting a sharp phase transition at the critical scale
$\sigma_S^2 \simeq 1/\sqrt{T}$ (\Cref{fig:phase}). Below the threshold
the regret is flat at $\sqrt{T}$; above it, regret decreases as
$1/\sigma_S^2$.

\begin{figure}[ht]
\centering
\includegraphics[width=0.6\linewidth]{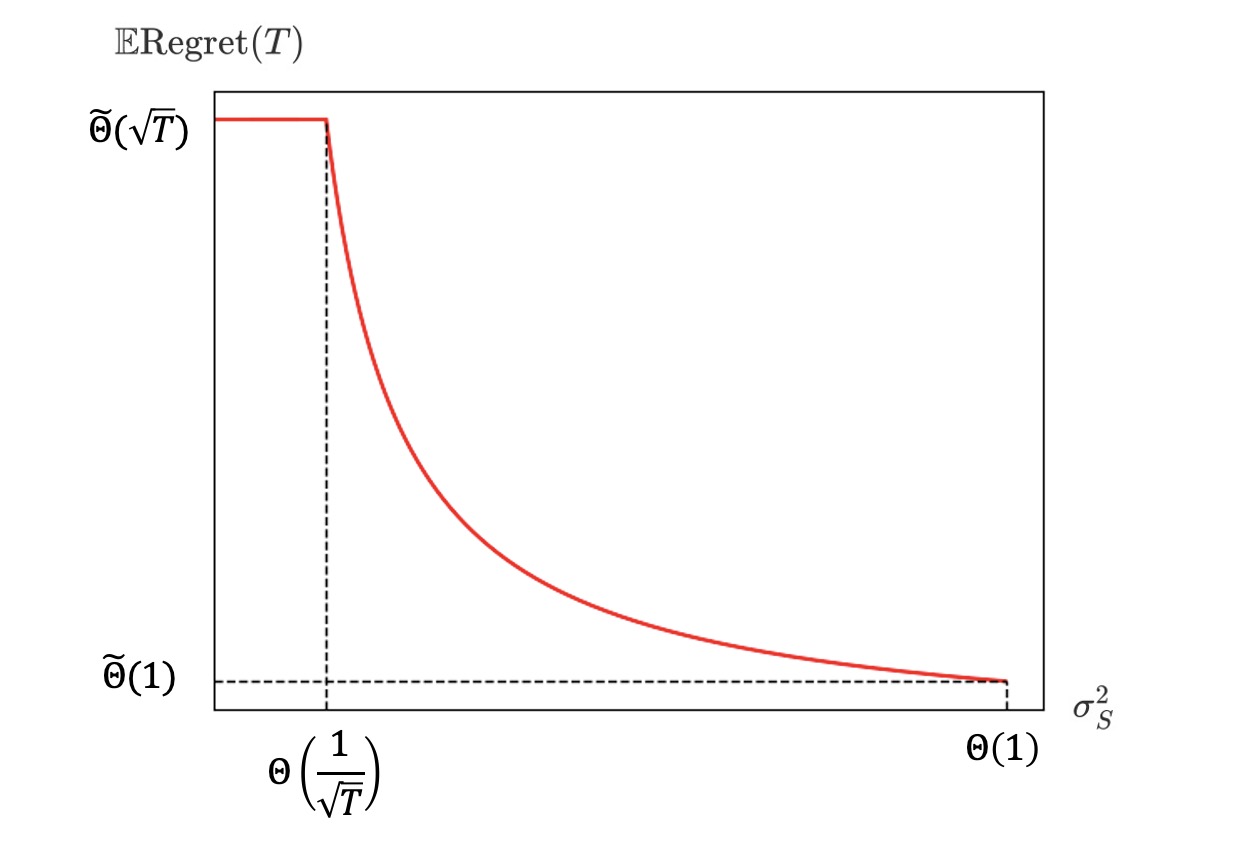}
\caption{Phase transition with respect to $\sigma_S^2$.}
\label{fig:phase}
\end{figure}

\paragraph{Why does noise help learning?}
The phase transition admits a clean economic reading: in our setting,
\emph{supply} randomness is informative for demand because the platform
chooses the service fee \emph{after} observing the supply curve, so the noise widens the natural exploration footprint without paying any revenue cost. Mathematically, a larger supply variance avoids the
weak-instrument problem and shrinks the variance of the IV estimator
$\hat\beta$. \emph{Demand} randomness, in contrast, does not help the
platform's regret: it contributes a constant variance term that the platform cannot manipulate, and our regret bound scales proportionally
with the magnitude of demand noise. The general lesson is that, in dynamic pricing, sources of market randomness should be disentangled
into the part that reveals identifying information (here $\sigma_S^2$)
and the part that adds residual variance (here $\sigma_D^2$); the platform's exploration budget should be calibrated to the former, not
the latter.

\paragraph{Practical implication.}
Categories with high natural supply volatility (food delivery, ride-hailing) deliver near-zero regret without any deliberate exploration, while categories with stable supply (subscriptions, utilities) require the platform to inject vanishing fee perturbations of magnitude $t^{-1/4}$. When the variance of the perturbation in \texttt{Act} is keyed directly to the dataset size, a single implementation handles both regimes without manual tuning---the \texttt{AAaI} wrapper merely needs to decide once whether $\sigma_S^2$ falls above or below the critical scale.
In real-world deployment, such potential perturbations should be implemented only within pre-specified price bands, audited ex ante and ex post, and constrained by consumer protection, non-discrimination, and any sector-specific pricing regulations. Platforms may also replace customer-facing randomization by less intrusive sources of variation whenever available, such as pre-announced pilots, randomized discounts, randomized fee waivers, or
experiments confined to legally permitted tariff windows.

\section{Practical Considerations}
\label{sec:practice}

The regret bounds of \Cref{sec:regret} hold for any function class
that satisfies \Cref{ass:context}. To make the result usable in
practice, we compute the effective dimension $\textit{dim}$ for two
function classes that span the spectrum from interpretable to highly
expressive: linear models (\Cref{sec:linear}) and multilayer
perceptrons (\Cref{sec:MLP}). We then sketch how the methodology
extends beyond the additive model $\beta Q + f(x_t)$ to fully
nonlinear demand $f(Q, x_t)$ via deep IV regression
(\Cref{sec:ext}).

Before turning to the two function classes, we note that the
\emph{effective dimension} $\textit{dim}$ in \Cref{ass:context} captures the complexity of the demand--feature relationship. Mathematically, highly irregular demand functions are unlearnable in
the worst case, but in practice demand depends on features in a structured, smooth way. A larger $\textit{dim}$ reduces model bias but inflates estimation variance for a fixed sample size, and our regret bounds scale linearly in $\textit{dim}$ (\Cref{thm:partially,thm:no_noise,thm:upper_ada}). The two examples below show that $\textit{dim}$ remains tractable for both linear and deep predictors.

\subsection{Linear demand}
\label{sec:linear}

When $\cF$ is the class of linear functions, every $f\in\cF$
corresponds to a parameter $\theta_f\in\RR^d$ via
$f(x)=\langle x,\theta_f\rangle$. We take $\cX = [0,1]^d$.

\begin{theorem}\label{thm:linear_class}
Assume $\|\theta_f\|_2 \le B/\sqrt{d}$. Then \Cref{ass:context} is
satisfied with $\textit{dim} = d$.
\end{theorem}

\Cref{thm:linear_class} recovers the rates of
\citet{nambiar2019dynamic} and \citet{aidynamic} but---importantly---removes their assumption on the condition number of the feature covariance matrix. In quantitative-trading or large-scale ad-pricing settings, where companies routinely deploy thousands of features, near-collinearity between features is the rule rather than the exception, and a condition-number assumption is rarely satisfied in
practice. Our regret bound holds for any ERM solution and depends only on the covering number, providing a strictly more practical guarantee. 

\subsection{Deep neural networks: bypassing star-shapedness}
\label{sec:MLP}
We next instantiate \Cref{ass:context} for deep neural networks. Let $\cF$ be the class of $L$-layer MLP
networks
\[
   f(x) \;=\; h_L \circ \mathrm{ReLU} \circ h_{L-1}
              \circ \cdots \circ \mathrm{ReLU}
              \circ h_2 \circ \mathrm{ReLU} \circ h_1(x),
\]
where $h_\ell(x) = W_\ell x + b_\ell$ and the layer dimensions
$d_0 = d,\, d_1,\ldots, d_{L-1},\, d_L = 1$ satisfy
$\max_\ell d_\ell \le W$. Let
$b = \max\{\|W_1\|_\infty, \ldots, \|W_L\|_\infty,
          \|b_1\|_\infty, \ldots, \|b_L\|_\infty\}$
denote the largest weight in the network; in practice $b$ is
controlled by an explicit weight penalty.

\begin{theorem}\label{thm:MLP}
Assume $b \ge 1$ and $((W+1)b)^L \le B$. Then \Cref{ass:context} is
satisfied with $\textit{dim} \lesssim \cO(W^2 L)$. Furthermore, if
$W, L \ge 60$, $\textit{dim} \simeq \Theta(W^2 L)$.
\end{theorem}

\Cref{thm:MLP} explains the empirical success of deep networks in demand prediction. The condition $((W+1)b)^L \le B$ is the standard output bound used throughout the covering-number literature~\citep{shen2023exploring,ou2024covering} and is enforceable through any common weight regularizer. The bound is \emph{linear in depth and quadratic in width}, consistent with the practitioner's preference for deep but narrow networks~\citep{he2016deep}. Combined with the linear dependence of regret on $\textit{dim}$ in \Cref{thm:partially,thm:no_noise,thm:upper_ada}, we conclude that scaling the network changes the regret only polynomially---our framework does not suffer from the curse of model complexity and is suitable for large-scale deployment.

The main technical difficulty in applying our IV analysis to modern high-capacity predictors, such as MLPs, is the star-shapedness requirement in standard nonparametric IV theory. 
The classical statistical-learning route to studying deep models in a nonparametric framework requires the centered class $\cF - f := \{g - f : g \in \cF\}$ to be \emph{star-shaped} around the origin, that is, $\lambda(\cF - f)\subseteq \cF - f$ for every $\lambda \in [0,1]$~\citep{wainwright2019high, yu2022strategic}. This regularity is mild for $\cF$ itself when the network has a free scalar at its last layer: rescaling the output multiplies the entire function. After
centering around an unknown $f$, however, star-shapedness almost never survives---there is no reason for $\lambda f' + (1-\lambda) f$ to lie in $\cF$ for two arbitrary networks $f, f'$. Imposing convexity on $\cF$ would restore the property but exclude essentially all deep architectures. The assumption thus rules out the very models our framework aims to support (\Cref{fig:star}).

\begin{figure}[ht]
\centering
\includegraphics[width=0.45\linewidth]{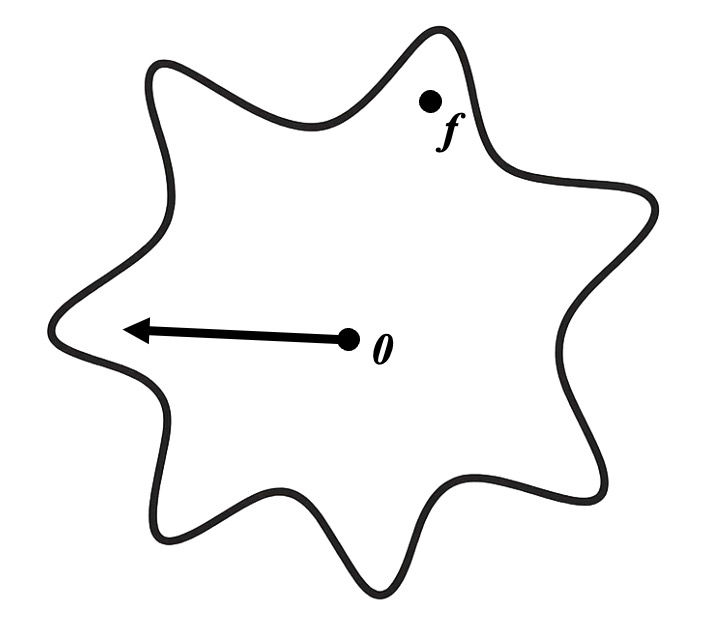}
\caption{$\cF$ is only star-shaped around $0$.}
\label{fig:star}
\end{figure}

\paragraph{A homeomorphism that recovers star-shapedness.}
We resolve this difficulty by constructing an auxiliary function class that is homeomorphic to $\lambda\cF + (1-\lambda) f$. This auxiliary class is not directly accessible, since $f$ is unknown; our algorithm does not rely on access to it, and we invoke the construction purely as a bridge for the statistical analysis. The auxiliary class has two properties that drive the proof. First, its provably effective dimension increases from $\textit{dim}$ to $\textit{dim} + 1$---only one extra dimension relative to $\cF$. Intuitively, one can think of the auxiliary class as duplicating the network $\cF$ (doubling its width) and adding a new connecting layer with weights $\lambda$ and $1 - \lambda$ on the two copies. Second, the auxiliary class is star-shaped with respect to $f$, even though we do \emph{not} require $\cF$ itself to be star-shaped, thereby extending the boundaries of traditional statistical learning at the cost of only a one-dimensional increase in complexity. Replacing the star-shapedness condition with this one-dimensional overhead, we obtain an upper bound that bypasses the unrealistic star-shaped assumption and provides the first theoretical justification for ERM-with-IV demand learning over arbitrarily deep networks.

We end the section by addressing a separate practical concern about DNN: whether the ERM oracle invoked in \Cref{alg:AaaIV} can be realized in practice. Real implementations approximate ERM by stochastic gradient descent; the resulting optimization error~\citep{meng2017generalization} can be controlled in over-parameterized regimes, where gradient descent provably converges to a global minimum~\citep{cao2020generalization,qiu2020robust}, and is in any case small under standard initialization schemes such as He initialization~\citep{he2015delving}. Our experiments in \Cref{sec:exp} use the same gradient-based pipeline that platforms deploy in production.

\subsection{General demand learning}
\label{sec:ext}
The additive model $P_{Dt} = \beta Q + f(x_t) + \epsilon_{Dt}$ is
rich enough for our main results, but in some applications the price--quantity relationship cannot be cleanly separated. Suppose demand follows the fully nonlinear specification $P_{Dt} = f(Q, x_t) + \epsilon_{Dt}$ without any further structural restriction. Confounding between $P_t^e$ and $Q_t^e$ persists, and the conditional-moment identity of \citet{newey2003instrumental}
reads
\[
   \E[P_t^e \given x_t, a_t]
   \;=\; \int f(Q_t^e, x_t)\, dF(Q_t^e \given x_t, a_t),
\]
where $F(\cdot \given x_t, a_t)$ is the conditional CDF of $Q_t^e$
given $(x_t, a_t)$. A natural two-stage procedure follows: first
approximate $F(Q_t^e \given x_t, a_t)$ with a generative neural
network, then plug in via Monte Carlo and learn $f(\cdot, \cdot)$
with a second network. This is precisely the Deep IV training
pipeline of \citet{hartford2017deep}.

The three obstacles to treating $a_t$ as a strict instrument were
already discussed in \Cref{sec:aai-idea}: the policy generating
$a_t$ evolves over time, $a_t$ requires a vanishing artificial
perturbation when the supply is degenerate, and $a_t$ is mediated by the
buyer's strategic behavior. The same three remedies---martingale
concentration, $t^{-1/4}$ exploration, and $\cO(\log T)$-switching
updates---transfer to the nonlinear specification unchanged, since
they act on $a_t$'s temporal structure rather than on the form of
$f$. \Cref{alg:AaaIV} and its wrapper \Cref{alg:Ada-AaaIV} therefore
extend verbatim, with the only substantive change being that the
inner ERM step now fits a Deep IV network rather than a
linear-in-$Q$ predictor.

Deep IV regression typically lacks rigorous theoretical guarantees. \citet{xu2020learning} tries to establish a rate, but only
for generalized linear networks with $\widetilde\cO(T^{-1/4})$
convergence. Our paper is the first to provide regret bounds for
general deep networks in additive models, and obtaining bounds for
fully nonseparable $f(Q, x_t)$ remains a promising open direction.
\Cref{sec:ride} provides a numerical study on a Lyft ride-hailing
dataset that illustrates the empirical robustness of this extension.

\section{Experiments}
\label{sec:exp}

We now turn from theory to practice. 
The numerical experiments in this section have three goals: (i) to illustrate the regret regimes of \Cref{thm:partially,thm:no_noise} on synthetic data, including the phase transition predicted by \Cref{thm:upper_ada} and
\Cref{thm:lower};
(ii) to verify the methodology of ``actions as
instruments'' on real micro-level transaction data from a large
food-delivery platform; and 
(iii) to illustrate the potential revenue implications of dynamic service-fee redistribution.
\Cref{sec:exp+phase} addresses (i), and
\Cref{sec:exp+Talabat} addresses (ii) and (iii). Ablation studies
isolating the contribution of the doubly-robust correction
(\Cref{eq:IV_robust}) are deferred to \Cref{app:ablation}; a
parallel empirical study on Lyft ride-hailing data, which exercises
the nonlinear extension of \Cref{sec:ext}, is reported in
\Cref{sec:ride}.

\subsection{Phase transition: regret bounds with respect to supply
volatility}
\label{sec:exp+phase}

\paragraph{Known supply variance.}
We first illustrate the two regret regimes of
\Cref{sec:upper+constant} through simulation.
We take $x_t \in \cX = [0,1]^5$ and let $\cF$ be a three-layer multilayer
perceptron with two hidden layers of dimension $10$. The supply
curve is $P_{St} = Q + \epsilon_{St}$ with
$\epsilon_{St} \sim \cN(0, \sigma_S^2)$, and the demand curve is
$P_{Dt} = f(x_t) - Q + \epsilon_{Dt}$ for some $f \in \cF$. We
consider two regimes: $\sigma_S^2 = 1$, which falls into the
positive-volatility case of \Cref{thm:partially}, and
$\sigma_S^2 = 0$, which falls into the noise-free case of
\Cref{thm:no_noise}. For computational convenience, we replace the
ERM step in \Cref{alg:AaaIV} by \texttt{Adam} optimizer, mirroring standard deep-learning practice and demonstrating that our framework runs
with off-the-shelf optimizers. 
We set $T = 100{,}000$ and repeat each run 10 times under the same implementation, reporting 95\% confidence intervals; full implementation details are in \Cref{app:exp}.

\begin{figure}[!ht]
\centering
\begin{minipage}[t]{0.465\textwidth}
\centering
    \includegraphics[width=\linewidth]{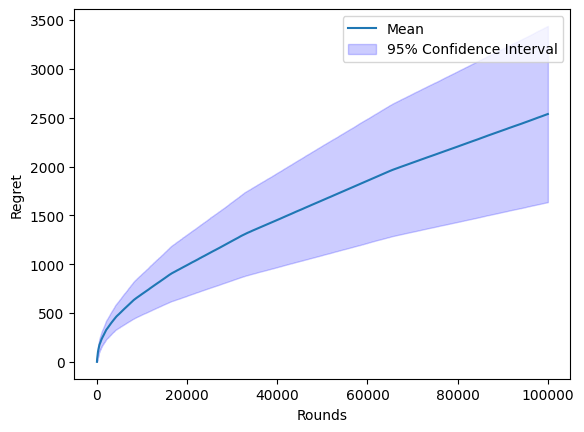}
    \caption{Regret of \Cref{alg:AaaIV} when $\sigma_S^2 = 1$.}
    \label{fig:log}
\end{minipage}
\hspace{0.5cm}
\begin{minipage}[t]{0.485\textwidth}
\centering
    \includegraphics[width=\linewidth]{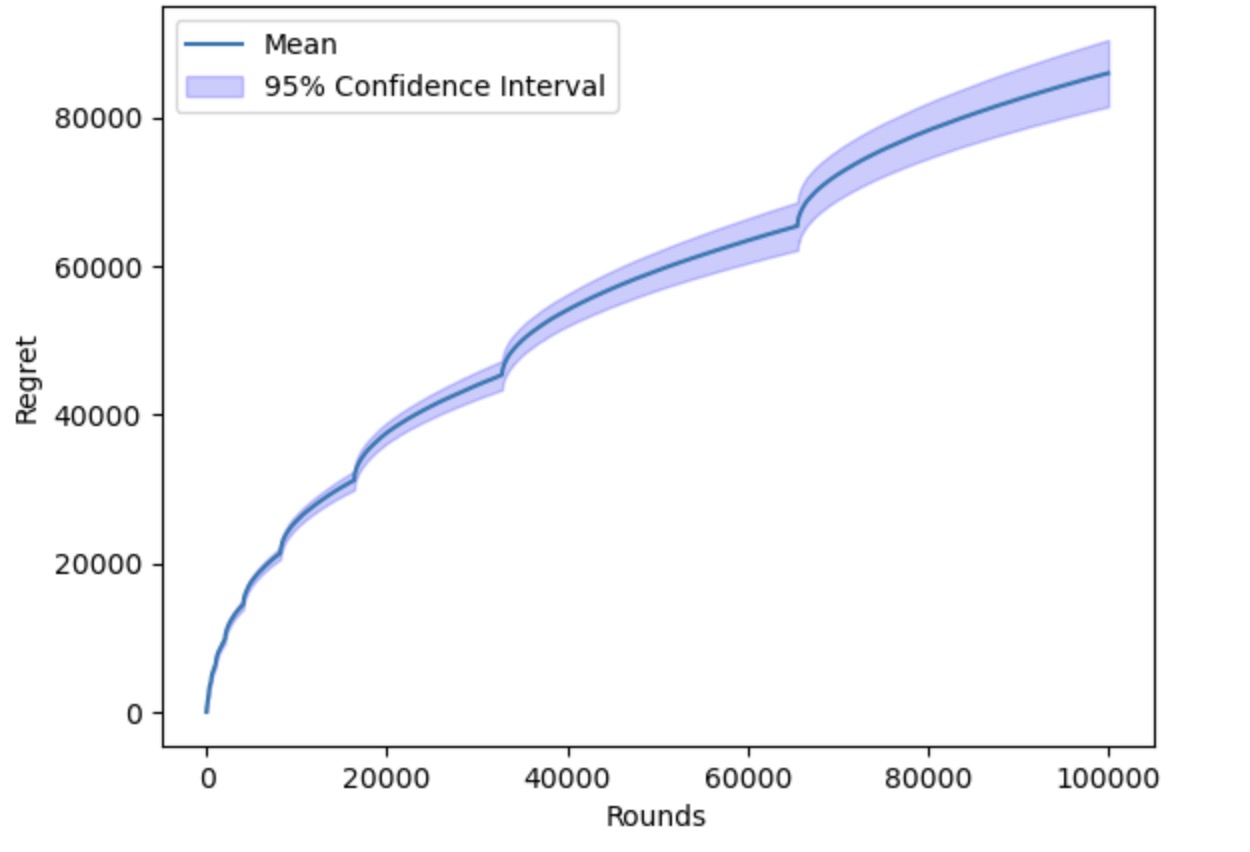}
    \caption{Regret of \Cref{alg:AaaIV} when $\sigma_S^2 = 0$.}
    \label{fig:sqrt}
\end{minipage}
\end{figure}

\Cref{fig:log,fig:sqrt} show a pronounced separation between the two regimes. With positive supply variance, regret remains below 5\% of its noise-free counterpart over the full horizon, matching the $\log T$ vs.\ $\sqrt T$ contrast predicted by \Cref{thm:partially} and \Cref{thm:no_noise}. In the noise-free case, regressing the average regret of \Cref{fig:sqrt} on the time horizon yields $\widehat{\textrm{Regret}}(T) \approx 137\, T^{0.56}$, close to the theoretical $\sqrt T$ rate in \Cref{thm:no_noise} and \Cref{thm:lower_no_noise}.


\paragraph{Unknown supply variance.}
Next we examine how regret depends on $\sigma_S^2$ when the platform
does not know its value, exercising \Cref{alg:Ada-AaaIV}. We vary
$\sigma_S^2$ over $[0,1]$ on a uniform grid of 200 points, with
$T = 10{,}000$ and a hypothesis-testing budget $T_0 = 100$. 
The results in \Cref{fig:phase_sigmaS} display a finite-sample transition: below the threshold (instances classified under $\HH_0$), regret is roughly flat in $\sigma_S^2$, whereas above it (instances under $\HH_1$), regret declines monotonically. The LOWESS fit~\citep{cleveland1979robust} makes this plateau-to-decline pattern explicit. The observed transition occurs at the theoretically predicted scale $\sigma_S^2\eqsim\Theta(1/\sqrt{T})$, with a finite-sample constant that depends on the testing rule, consistent with the rate $\widetilde{\Theta}(\sqrt T \wedge 1/\sigma_S^2)$ in \Cref{thm:upper_ada} and \Cref{thm:lower}.


\begin{figure}[!ht]
    \centering
    \includegraphics[width=0.5\linewidth]{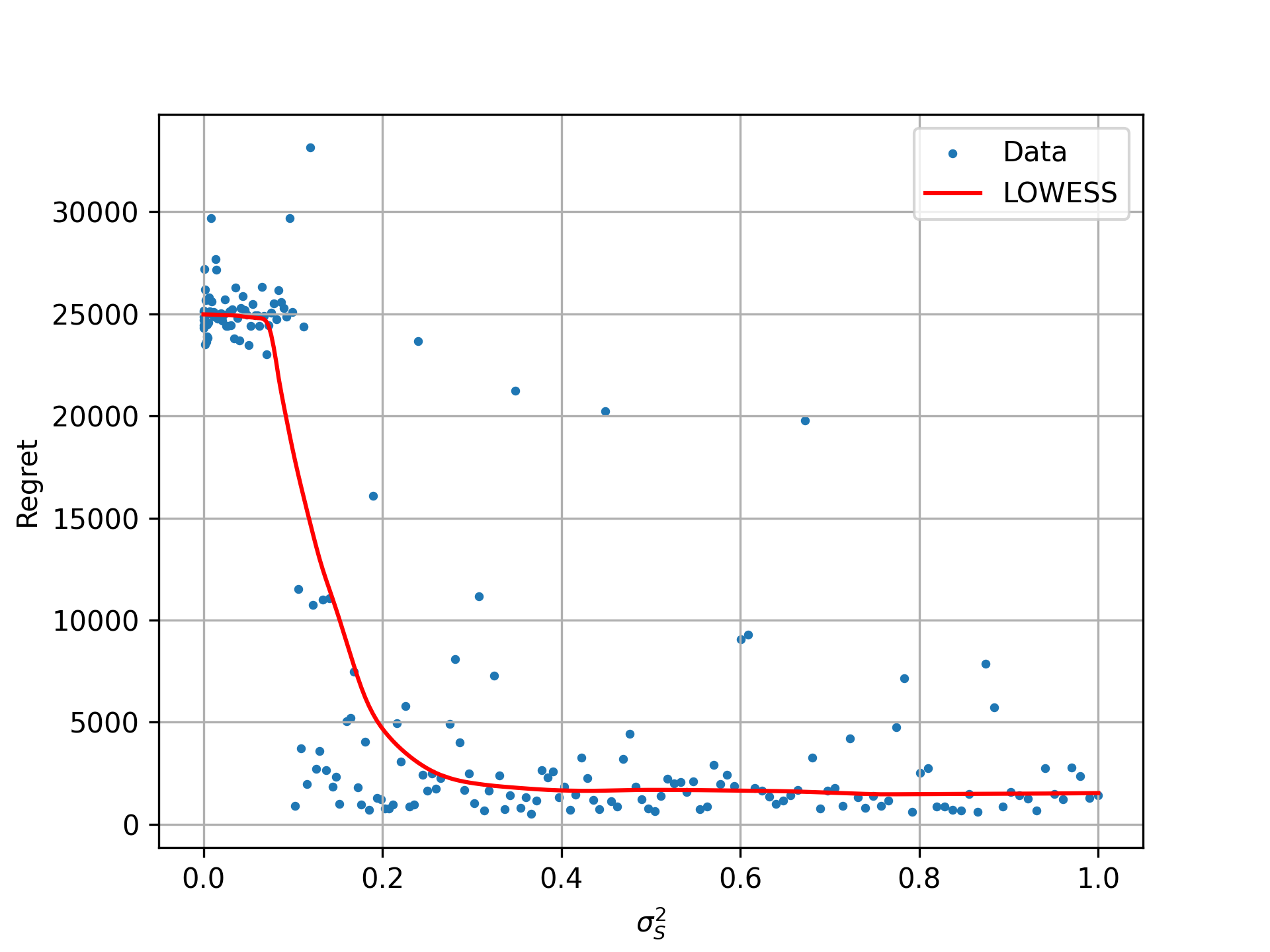}
    \caption{Phase transition and its LOWESS fitting.}
    \label{fig:phase_sigmaS}
\end{figure}

\subsection{Actions as instruments: an empirical study on the Talabat
platform}
\label{sec:exp+Talabat}

We now evaluate the methodology of actions-as-instruments on micro-level transaction data from \emph{Talabat}, a leading food-delivery platform operating in 21 Egyptian cities. The dataset covers \emph{September 1, 2021 through January 1, 2022} and contains detailed order-level snapshots\footnote{\url{https://github.com/hassan-mustafa/Talabat-Delivery-Hero-Case-Study.}}.
Throughout, we restrict the hypothesis class $\cF$ to the linear
specification of \Cref{sec:linear}, 
\[
\mathcal{F} = \{ f : f(x_t) \text{ is linear in } x_t \},
\]

\paragraph{Data construction and feature design.}
We aggregate successful orders at the hourly level to approximate
the market transacted quantity $Q_t^e$. The price is
defined as the total amount a customer pays for an order, and the
\emph{service fee} is the difference between this payment and the
tax-inclusive value of goods before discounts. The service fee
comprises platform fees (E\pounds8 per food order as of May 2026),
delivery-partner compensation, restaurant charges, and coupon
adjustments. Although Talabat also offers six smaller categories
(pharmacy, cosmetics, electronics, flowers, groceries, pet
supplies), food deliveries account for $\approx 83.5\%$ of all
transactions. Feature vectors are simple hourly averages: for
example, if 30\% of orders in an hour originate from Alexandria and 70\% from Cairo, the Alexandria dummy takes the value $0.3$. We include
dummy variables for discounts and first-time customers, and
continuous variables for delivery time and drop-off distance. After
dropping observations with missing values, the final dataset
contains 2{,}639 hourly observations with 58 features. Order
volumes exhibit strong intraday cyclicality, with a minimum between
4--7~a.m.\ and a peak around dinner time
(\Cref{fig:food_time}); additional preprocessing details are in
\Cref{app:emp}.

\begin{figure}[!ht]
\begin{minipage}[t]{0.48\textwidth}
    \centering
    \includegraphics[width=\linewidth]{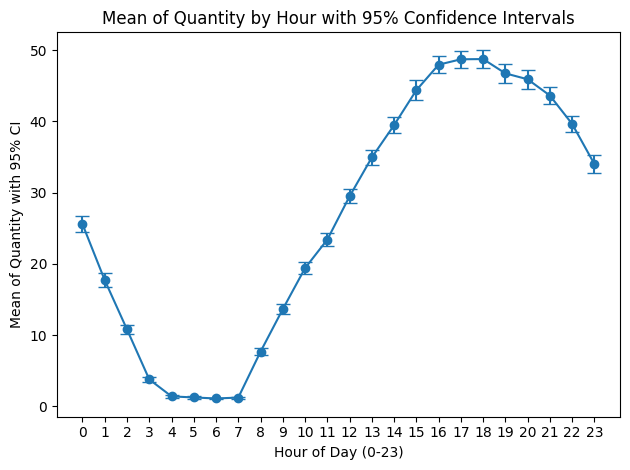}
    \caption{Relationship between transacted quantity and time of day.}
    \label{fig:food_time}
\end{minipage}\hfill
\begin{minipage}[t]{0.48\textwidth}
    \centering
    \includegraphics[width=\linewidth]{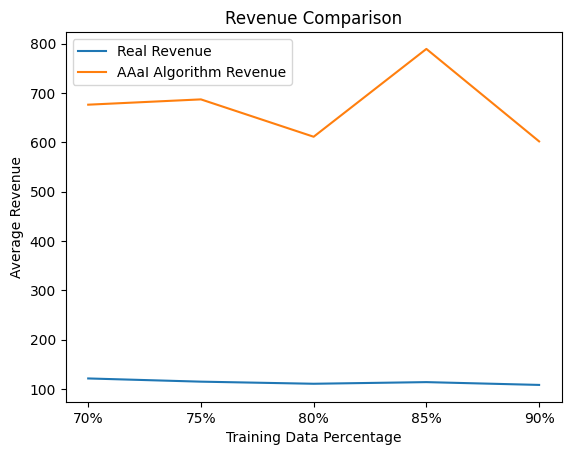}
    \caption{Offline counterfactual revenue via ``actions as instruments''.}
    \label{fig:food_per}
\end{minipage}
\end{figure}

\paragraph{Empirical strategy: explore-then-commit on chronological
splits.}
We partition the data chronologically: the first 80\% is used to
estimate the demand function and the remaining 20\% is used for
counterfactual evaluation. 
The explore-then-commit design provides a chronological held-out offline evaluation: it estimates the demand function on the training set, applies the learned model to adjust service fees on the test set, and compares model-implied counterfactual revenue with observed revenue.
On the training set, a naive OLS
regression of the cutoff price on transacted quantity (without an
instrument) yields an insignificant slope,
$\hat\beta_{\textrm{OLS}} = -0.02$ ($p = 0.81$), illustrating exactly the confounding pathology diagnosed in
\Cref{sec:why-fail}. 
Applying the estimation module of \Cref{alg:AaaIV}, treating the service fee as an instrument under the maintained conditional-orthogonality assumption, produces $\hat\beta = -7.18$ with a first-stage F-statistic of $442.5$. The statistic indicates strong instrument relevance conditional on the included features. Given the training-set average quantity (29.61) and average price (E\pounds113.28), the resulting estimate corresponds to a price elasticity of demand of approximately $-0.53$, that is, under the maintained specification, a 1\% price increase is associated with a 0.53\% decrease in demand, consistent with the relatively low elasticity expected for staple categories such as meal deliveries.


\paragraph{Revenue evaluation: dynamic redistribution at fixed mean fee.}
Because supply-side data are not available in the dataset, we choose hourly service fees to maximize model-implied revenue while holding the test-set \emph{mean service fee} fixed at its observed level (E\pounds4.90). This design isolates the \emph{redistributive} component of dynamic pricing from changes in the average fee level. Under the fitted linear market-clearing model and the fixed-average-fee constraint, average hourly counterfactual revenue is E\pounds611.45 rather than the observed E\pounds110.65, a $\approx 5.5\times$ model-implied difference (\Cref{fig:food_per}); because quantity is affine in the service fee, the unchanged mean fee also keeps mean transacted quantity unchanged. The policy reallocates fees toward peak-demand hours, where estimated local elasticity is lower, and away from off-peak periods. 
Taken together, this counterfactual quantifies the potential revenue upside from dynamic redistribution within the fitted model and fixed-average-fee constraint.

\paragraph{Robustness and limitations.}
We vary the training-test split between 70\% and 90\%; the model-implied revenue differences are qualitatively similar across these configurations (\Cref{fig:food_per}). Two interpretation considerations remain. First, the service-fee field bundles platform commissions with restaurant-side charges, so the headline E\pounds4.90 average overstates the platform-only fee. Second, unobserved supply-side variation, for example, additional delivery capacity coming online during dinner hours, could correlate with both transacted quantity and the service fee, affecting the estimated elasticity. 
We therefore view the $5.5\times$ figure as an optimistic, model-implied counterfactual under the fixed-average-fee constraint. Taken together, the exercise illustrates the potential revenue upside of dynamic redistribution in the estimated model.

\section{Conclusion}
\label{sec:conclusion}

This paper studies dynamic service-fee pricing on third-party
platforms, a problem that combines confounded observations,
strategic buyers, and high-capacity demand models. Three findings form the core of our contribution. First, the platform's own
service fee, despite being chosen adaptively, can be
used as an instrument for the demand curve; combined with a
doubly-robust correction, vanishing $t^{-1/4}$ exploration, and
$\cO(\log T)$ low-switching updates (\Cref{alg:AaaIV},
\Cref{alg:Ada-AaaIV}), this yields a working IV-based demand
learner in a fully closed-loop pricing setting. 
Second, the regret is governed by the variance of supply-side noise, with a sharp phase transition at $\sigma_S^2 \simeq 1/\sqrt{T}$: above this
threshold supply randomness alone identifies demand and regret is
logarithmic in $T$, while below it the platform must explicitly
explore at the $\sqrt{T}$ rate (\Cref{thm:partially},
\Cref{thm:no_noise}, \Cref{thm:upper_ada}, \Cref{thm:lower}). The lower bound rests on a new construction that isolates the
information channel carried by demand noise---an innovative device in the dynamic-pricing literature. 
Together, the upper- and lower-bound analyses give a clean economic reading of the phase transition: \emph{market volatility is information, not just noise}.
Third, a homeomorphism construction allows the ERM-with-IV framework to extend from linear models to arbitrarily deep neural networks 
at the cost of at most one additional dimension of complexity (\Cref{thm:MLP}), providing the first regret guarantee for DNN-based demand learning
under confounding.
We further apply the methodology to micro-level transaction data from \textsc{Talabat} (\Cref{sec:exp+Talabat}) and
\textsc{Lyft} (\Cref{sec:ride}), showing that actions-as-instruments yields an economically sensible demand curve where naive regression is uninformative, and that elasticity-driven fee redistribution has model-implied revenue potential under fixed-average-fee counterfactuals.


This work opens up several promising directions for future research. Is the dependence of our algorithm on $\gamma$ optimal? While we provide theoretical guarantees under a separable model, 
we propose an extension and illustrate its implementation for non-separable models in \Cref{sec:ext} and \ref{sec:ride}.
However, establishing theoretical guarantees for pricing with deep neural networks remains challenging due to the lack of a solid theoretical foundation in deep learning. Bridging the gap between empirical success and theoretical understanding under more general models is a valuable avenue for future exploration. 
It is also important to implement such exploration as constrained, auditable experiments, accompanied by regulatory safeguards and robustness checks.
We hope our work offers useful insights into the pricing problem faced by third-party platforms.

\bibliographystyle{plainnat}

\bibliography{Sections/reference}

\newpage
\appendix

\section{Auxiliary Lemmas}
We first state some lemmas to pave the way for incoming proofs.
\begin{lemma}[Corollary 1.7. in \citet{rigollet2023high}]\label{lem:hoeffding}
    Suppose that $X_i, i=1, \ldots, n$ are independent, zero-mean, $\eta$-sub-Gaussian random variables. Then for all $a\in \RR^n$ and $t \geq 0$, we have
\[
\mathbb{P}\left[\left|\sum_{i=1}^n a_i X_i \right|\geq t\right] \leq 2\exp \left\{-\frac{t^2}{2  \eta^2\|a\|_2^2}\right\} .
\]
\end{lemma}

\begin{lemma}[Theorem 1.13 in \citet{rigollet2023high}]\label{lem:subE}
    Suppose that $X_i, i=1, \ldots, n$ are independent, zero-mean, $\eta$-sub-exponential random variables. Then for all $t \geq 0$, we have
    \[
    \mathbb{P}\left[\left|\frac{1}{n}\sum_{i=1}^n  X_i \right|\geq t\right] \leq 2\exp \left\{-\frac{n}{2}\min\{\frac{t^2}{\eta^2},\frac{t}{\eta}\}\right\} .
    \]
    Consequently, when $t\le\eta$,
        \[
    \mathbb{P}\left[\left|\frac{1}{n}\sum_{i=1}^n  X_i \right|\geq t\right] \leq 2\exp \left\{-\frac{nt^2}{2\eta^2}\right\} .
    \]
\end{lemma}
Some literature~\citep{wainwright2019high} uses $(\eta,\alpha)$ to represent a sub-exponential distribution and the bound will become $2\exp \{-\frac{n}{2}\min\{\frac{t^2}{\eta^2},\frac{t}{\alpha}\}\}$.
\begin{lemma}[Theorem 8 in \citet{rakhlin2022mathematical}]\label{lem:theorem8}
    Let $X_1,...,X_n$ be i.i.d. and $X_i\in[-1,1]$. Then for all $t \geq 0$, we have
    \[
    \PP\left[\frac{1}{n}\sum_{i=1}^n X_i-\E X_i\ge \sqrt{\frac{2t\Var(X_i)}{n}}+\frac{t}{3n}\right]\le e^{-t}.
    \]
\end{lemma}
\begin{lemma}[Theorem 3.16 in \citet{wainwright2019high}]\label{lemma:3.16}
    Let $\mathbb{P}$ be any strongly log-concave distribution with parameter $\gamma>0$. Then for any function $f: \mathbb{R}^n \rightarrow \mathbb{R}$ that is L-Lipschitz with respect to Euclidean norm, we have
\[
\mathbb{P}[|f(X)-\mathbb{E}[f(X)]| \geq t] \leq 2 e^{-\frac{\gamma t^2}{4 L^2}}.
\]
\end{lemma}
\begin{lemma}[Lemma 1.12. in \citet{rigollet2023high}]\label{lem:square_subE}
    If $X$ is an $\eta$-sub-Gaussian random variable, then $X^2-\E[X^2]$ is a $16\eta^2$-sub-exponential random variable.
\end{lemma}
\begin{lemma}\label{lem:product+subG}
    Suppose $\eta_X$-sub-Gaussian random variable $X$ and $\eta_Y$-sub-Gaussian random variable $Y$ are independent.  Then, $XY$ is a $\sqrt{3}\eta_X\eta_Y$-sub-exponential random variable.
\end{lemma}
\begin{proof}{Proof}
Since $X$ is $\eta_X$-sub-Gaussian and independent of $Y$, it holds that
    \[
    \E e^{\lambda XY}\le \E e^{\frac{\lambda^2\eta_X^2 Y^2}{2}}.
    \]
    We then calculate $\E e^{sY^2}$. It holds that when $0\le s<\frac{1}{2\eta_Y^2}$
$$
        \E e^{sY^2}=\int_{-\infty}^\infty e^{sy^2}f_Y(y)dy=\int_0^\infty \PP(e^{sy^2}>u)du\le \int_1^\infty 2e^{-\frac{\log u}{2s\eta_Y^2}}du=\frac{2}{\frac{1}{2s\eta_Y^2}-1}\le e^{\frac{2s\eta_Y^2}{1-2s\eta_Y^2}},
$$
    where $f_Y(\cdot)$ is the p.d.f. of $Y$. We use \Cref{lem:hoeffding} with $n=1$ in the first inequality. The second inequality comes from the fact that $x\le e^{\frac{1}{2}x}$. 

    Therefore, it holds that
    \[
    \E e^{\lambda XY}\le \E e^{\frac{\lambda^2\eta_X^2 Y^2}{2}}\le e^{\frac{\lambda^2\eta_X^2\eta_Y^2}{1-\lambda^2\eta_X^2\eta_Y^2}}\le e^{\frac{3\lambda^2\eta_X^2\eta_Y^2}{2}}=e^{\frac{1}{2}\lambda^2(\sqrt{3}\eta_X\eta_Y)^2},
    \]
    when $\lambda^2\le\frac{1}{3\eta_X^2\eta_Y^2}$. Hence, $XY$ is a $\sqrt{3}\eta_X\eta_Y$-sub-exponential random variable.
\end{proof}
\begin{lemma}[Lemma 36 in \citet{rakhlin2022mathematical}]\label{lem:lemma36}
    Let $\cG$ be a class of functions with values in $[0,1]$. Then, with probability at least $1-e^{-t}$ for all $g\in \cG$
    \[
    \E g(X)- \frac{2}{n}\sum_{i=1}^n g(X_i)\lesssim\cO(\bar\delta^2)+\cO\left(\frac{t+\log\log n}{n}\right),
    \]
    where $\bar\delta$ is the critical radius for $\cG$.
\end{lemma}
\begin{lemma}[Theorem 2 in \citet{rakhlin2022mathematical}]\label{lem:inner}
    When $\epsilon$ is a sub-Gaussian random variable, it holds that
\[
\mathbb{E} \sup _{f \in \cF:\|f\|_n \leq u}\langle\epsilon, f\rangle_n \lesssim \cO\left(\frac{1}{\sqrt{n}} \int_{0}^u \sqrt{\log N\left(\mathcal{F}, \|\cdot\|_\infty, x\right)} dx\right).
\]
\end{lemma}
\begin{lemma}[Lemma 37 in \citet{rakhlin2022mathematical}]\label{lem:squareclass}
    For any class $\cF=\{f:\cX\rightarrow[-1,1]\}$ of bounded functions, the critical
radius $\bar\delta$ for the class $\cG=\cF^2$ can be upper bounded by a solution to
\[
\frac{12}{\sqrt{n}}\int_{u/16}^1\sqrt{\log N(\cF,\|\cdot\|_\infty,x/2)}dx\le u/4.
\]
\end{lemma}
Here, We use the fact in \Cref{lem:inner,lem:squareclass} that if $\|f_1-f_2\|_\infty\le \epsilon$, then for any given $x_1,...,x_n$, it holds that $\sqrt{\frac{1}{n}\sum_{i=1}^n (f_1(x_i)-f_2(x_i))^2}\le \epsilon$.

\begin{lemma}[Lemma 2.6 in \citet{tsybakov2004introduction}]\label{lem:KL+ineq}
    For two probability distributions $p$, $q$ over space $(\Omega,\cF)$, it holds that for any $A\in \cF$
    \[p(A)+q(A^c)\ge \frac{1}{2}e^{-KL(p\|q)}.\]
\end{lemma}

\section{Omitted Proof in Section \ref{sec:regret}}
\subsection{Market Propositions}\label{app:market}
We first state some propositions of market cutoff points. 
Since $\epsilon_{Dt}$ is strongly-log-concave, we know it's also sub-Gaussian. We assume that $\epsilon_{Dt}$ is an $\eta_D$-sub-Gaussian random variable.
\begin{proposition}\label{prop:truth}
    Myopically considering the surplus in one round, truthfully behaving as the demand $P_{Dt}$ is the optimal policy for the representative buyer.
\end{proposition}
\begin{proof}{Proof}
    Since $\Sur_t=\int_0^{Q_t^e}(P_{Dt}(Q,x_t)-P_{St}(Q)-a_t)dQ$ and $a_t$ is predetermined, we know the optimal $Q_t^e$ satisfies $P_{Dt}(Q_t^e,x_t)-P_{St}(Q_t^e)-a_t=0$. From \Cref{eq:equ}, we know that $P_{Dt}'=P_{Dt}$ maximizes the surplus, showing truthful behavior is the optimal policy. We consider continuous $Q$ in this paper. For a discrete $Q$, the integral can be replaced by a summation without changing any core conclusion.
\end{proof}

\begin{proposition}\label{prop:equ}
    For any misreported demand $P_{Dt}'$ leading to cutoff price and quantity $P_t^e$ and $Q_t^e$, it's equivalent to assume 
    \[
    P_{Dt}'=\beta Q+f(x_t)+\epsilon_{Dt}+\xi_{Dt},
    \]
    where $\xi_{Dt}=P_t^e-\beta Q_t^e-f(x_t)-\epsilon_{Dt}$.
\end{proposition}
\begin{proof}{Proof}
    Since the platform can only observe the market cutoff point $(P_t^e,Q_t^e)$, we only need to prove that such $\xi_{Dt}$ will yield the needed cutoff. Recall that other transactions provide no extra beneficial information to the platform.
    We obtain the expression of $\xi_{Dt}$ from \Cref{eq:equ} directly.
\end{proof}
We then bound the magnitude of $\xi_{Di}$ in the following proposition.
\begin{proposition}\label{prop:bound_strategy}
    Assume that the platform only updates the algorithm for setting the service fee after $s$ rounds, it holds that
    \[
    |\xi_{Dt}|\le \sqrt{\frac{\gamma (\alpha_1-\beta)(B^2+\eta_D^2)}{(1-\gamma)|\beta|}}.
    \]
\end{proposition}
\begin{proof}{Proof}
We know that $f\in \cF$ has $|f(x_t)|\le B$. From \Cref{prop:truth}, it holds that the expected surplus at each round is at most
\[
\Sur_t\le \E[\frac{1}{2|\beta|}(B+\epsilon_{Dt})^2]\le \frac{1}{2|\beta|}(B^2+\eta_D^2).
\]
The last inequality comes from the fact that $\epsilon_{Dt}$ is $\eta_D$-sub-Gaussian.

Since the platform only updates its algorithm for setting service fee after $s$ rounds, the upper bound of future cumulative surplus influenced by the strategic behavior at time $t$ is at most
\[
\sum_{\tau=t+s}^T \gamma^{\tau-t} \frac{1}{2|\beta|}(B^2+\eta_D^2)\le \frac{(B^2+\eta_D^2)\gamma^s}{2(1-\gamma)|\beta|}.
\]
The immediate surplus loss related to $\xi_{Dt}$ is 
\[
\Delta \Sur_t=-\frac{1}{2}(\Delta Q_t^e)^2(\alpha_1-\beta)=-\frac{1}{2(\alpha_1-\beta)}\xi_{Dt}^2.
\]
Therefore, we need 
\[
\frac{(B^2+\eta_D^2)\gamma^s}{2(1-\gamma)|\beta|}-\frac{1}{2(\alpha_1-\beta)}\xi_{Dt}^2\ge 0,
\]
yielding $|\xi_{Dt}|\le \sqrt{\frac{\gamma (\alpha_1-\beta)(B^2+\eta_D^2)}{(1-\gamma)|\beta|}}\lesssim \cO(1)$ as $s\ge 1$.
\end{proof}
\Cref{prop:bound_strategy} provides us with an upper bound on misreporting. Combined with \Cref{prop:truth}, we know that the buyer will approximately behave according to the true demand.

We now bound the magnitude of the cutoff price and quantity. We call a condition a success event if all previous high-probability inequalities hold.
\begin{proposition}
    Under the success event, with probability at least $1-\delta$, it holds that the cutoff prices and quantities are all smaller than $\cO(\sqrt{\log(\frac{T}{\delta})}+\frac{1}{\sqrt{1-\gamma}})$. So, we can simply choose upper bounds of $P_t^e$ and $Q_t^e$ as $\bar P$ and $\bar Q$ scaling as $\cO(\sqrt{\log(\frac{T}{\delta})}+\frac{1}{\sqrt{1-\gamma}})$.
\end{proposition}
\begin{proof}{Proof}
        We know the price $P_t^e$ is smaller than $f(x_t)+\epsilon_{Dt}+\xi_{Dt}$ as $\beta<0$. Meanwhile, we know that $Q_t^e\le \frac{P_t^e}{|\beta|}$. Since $\epsilon_{Dt}$ is $\eta_D$-sub-Gaussian, it holds that $\PP(\epsilon_{Dt}\ge x)\le e^{-x^2/(2\eta_D^2)}$. Setting this probability to be $\frac{\delta}{2T}$, we know that $ P_t^e\lesssim \cO(B+\sqrt{2\eta_D^2\log(\frac{2T}{\delta})}+\sqrt{\frac{\gamma (\alpha_1-\beta)(B^2+\eta_D^2)}{(1-\gamma)|\beta|}})\lesssim\cO(\sqrt{\log(\frac{T}{\delta})}+\frac{1}{\sqrt{1-\gamma}})$ since $f$ is bounded by $B$ and $|\xi_{Dt}|$ is bounded in \Cref{prop:bound_strategy}.  Similarly, we know that $ Q_t^e\lesssim\cO(\sqrt{\log(\frac{T}{\delta})}+\frac{1}{\sqrt{1-\gamma}})$ as well.
        So, we can choose $\bar P$ and $\bar Q$ scaling as $\cO(\sqrt{\log(\frac{T}{\delta})}+\frac{1}{\sqrt{1-\gamma}})$.
\end{proof}
Besides, we can presume all quantities are non-negative. In practice, when calculated quantities are negative, the market fails and they will be truncated to zero. Since $\max\{0,\cdot\}$ is 1-Lipschitz, such truncation won't affect the validity of auxiliary inequalities.
Finally, we give the closed-form expression of the cutoff and optimal service fee.
\begin{proposition}\label{prop:equ-init}
    For any given $a_t$, the cutoff price and quantity are 
    \[
    P_t^e=\frac{\alpha_1(f(x_t)+\epsilon_{Dt}+\xi_{Dt})-\beta(\alpha+\epsilon_{St}-a_t)}{\alpha_1-\beta}
    ,\
    Q_t^e=\frac{f(x_t)+\epsilon_{Dt}+\xi_{Dt}-\alpha_0-\epsilon_{St}-a_t}{\alpha_1-\beta},
    \]
    respectively.
\end{proposition}
\begin{proof}{Proof}
    Combining \Cref{eq:equ} and \Cref{prop:equ}, we derive the closed-form expression of cutoff $(P_t^e, Q_t^e)$ directly.
\end{proof}
\begin{proposition}\label{prop:opta}
    The optimal service fee given any feature $x_t$ is
    \[
    a_t^*=\frac{f(x_t)-\alpha_0-\epsilon_{St}}{2}.
    \]
\end{proposition}
\begin{proof}{Proof}
    Note that we define the optimal service fee with truthful report $P_{Dt}$ so $\xi_{Dt}=0$. Then, $\E[a_t\cdot Q_t^e]=\frac{(f(x_t)-\alpha_0-\epsilon_{St}-a_t)a_t}{\alpha_1-\beta}$, where the expectation is taken over $\epsilon_{Dt}$. Thus,
    \[
    a_t^*\leftarrow\argmin_a \frac{(f(x_t)-\alpha_0-\epsilon_{St}-a)a}{\alpha_1-\beta}=\frac{f(x_t)-\alpha_0-\epsilon_{St}}{2},
    \]
    which ends the proof. Consequently, for any $a_t$, it holds that the one-round regret is 
    \[
    \E[ a_t^*\cdot Q_t^e(P_{St},P_{Dt},a_t^*)-a_t\cdot Q_t^e(P_{St},P'_{Dt},a_t)]=\frac{1}{\alpha_1-\beta}(a_t^*-a_t)^2-\frac{\xi_{Dt} a_t}{\alpha_1-\beta}.
    \]
\end{proof}

\subsection{Omitted Proof in Section \ref{sec:upper+constant}}
\subsubsection{Lemmas Supporting Theorem \ref{thm:partially}}
After the $m$-th episode, there are $n=2^m$ tuples in the dataset $\cD$, namely $|\cD|=n$. We first study the accuracy of $\hat\beta$ and $\hat f$.
\begin{lemma}\label{lem:beta_noise}
    Let $\hat\beta$ be the estimate of $\beta$ after the $m$-th episode. Under the success event and \Cref{ass:context}, it holds that with probability at least $1-8\delta$,
    \[
    |\hat \beta-\beta|\lesssim\cO\left(\frac{\sqrt{\eta_S^2\log\frac{\log T}{\delta}}}{\sigma_S^2\sqrt{n}}+\frac{\sqrt{\eta_S^2\log \frac{T}{\delta}}}{(1-\gamma)\sigma_S^2n}\right),
    \]
    whenever $2^m\gtrsim\Omega(\frac{\eta_S^2\log\frac{T}{\delta}\log\frac{\log T}{\delta}}{\sigma_S^4}+\frac{\eta_S^2\log\frac{\log T}{\delta}}{(1-\gamma)\sigma_S^4}+\frac{\eta_S\sqrt{\log\frac{T}{\delta}}}{(1-\gamma)\sigma_S^2})$.
\end{lemma}
\begin{proof}{Proof}
From the implementation of \Cref{alg:AaaIV}, it holds that
$$\begin{aligned}
\hat\beta&=\frac{\sum_{i=1}^{n}(P_{i}^e-\hat f(x_i))(\hat f(x_i)-2a_i-\hat\alpha_0)}{\sum_{i=1}^{n}Q_i^e(\hat f(x_i)-2a_i-\hat\alpha_0)}\\
&=\beta+
\underbrace{
\frac{\frac{1}{n}\sum (f(x_i)-\hat f(x_i))(\epsilon_{Si}+\alpha_0-\hat\alpha_0)}{\frac{1}{n}\sum Q_i^e(\epsilon_{Si}+\alpha_0-\hat\alpha_0)}
}_{q_1}
+
\underbrace{\frac{\frac{1}{n}\sum \epsilon_{Di}(\epsilon_{Si}+\alpha_0-\hat\alpha_0)}{\frac{1}{n}\sum Q_i^e(\epsilon_{Si}
+\alpha_0-\hat\alpha_0)}
}_{q_2}
+
\underbrace{
\frac{\frac{1}{n}\sum \xi_{Di}(\epsilon_{Si}+\alpha_0-\hat\alpha_0)}{\frac{1}{n}\sum Q_i^e(\epsilon_{Si}+\alpha_0-\hat\alpha_0)}
}_{q_3}.
\end{aligned}$$
Here, we use the fact that $P_i^e=\beta Q_i^e+f(x_i)+\epsilon_{Di}+\xi_{Di}$ and 
due to \Cref{prop:opta}, it holds that $\hat f(x_i)-2a_i-\hat\alpha_0=\epsilon_{Si}+\alpha_0-\hat\alpha_0$. 

Then, we bound $|\alpha_0-\hat\alpha_0|$. It holds that, with probability at least $1-\delta$, due to \Cref{lem:hoeffding},
\[
|\alpha_0-\hat\alpha_0|=|\frac{1}{n}\sum_{i=1}^n\epsilon_{Si}|\le\sqrt{\frac{2\eta_S^2\log\frac{2\lceil \log_2 T \rceil}{\delta}}{n}},
\]
for all $m\le\lceil \log_2 T \rceil$.

For the numerator of $q_1$, it holds that
\[
|\frac{1}{n}\sum (f(x_i)-\hat f(x_i))(\alpha_0-\hat\alpha_0)|\le 2B\sqrt{\frac{2\eta_S^2\log\frac{2\lceil \log_2 T \rceil}{\delta}}{n}},
\]
since all functions in $\cF$ are uniformly bounded by $B$. Also, it holds that with probability at least $1-\delta$,
\[
|\frac{1}{n}\sum (f(x_i)-\hat f(x_i))\epsilon_{Si}|\le 2B\sqrt{\frac{2\eta_S^2\log\frac{2\lceil \log_2 T \rceil}{\delta}}{n}},
\]
where we use \Cref{lem:hoeffding} and set $a_i=\frac{1}{n}(f(x_i)-\hat f(x_i))$. Here, the convergence rate implicitly depends on the doubly-robust correction.

For the numerator of $q_2$, it holds that with probability at least $1-\delta$,
\[
|\frac{1}{n}\sum \epsilon_{Di}(\alpha_0-\hat\alpha_0)|\le \sqrt{\frac{2\eta_S^2\log\frac{2\lceil \log_2 T \rceil}{\delta}}{n}}
\sqrt{\frac{2\eta_D^2\log\frac{2\lceil \log_2 T \rceil}{\delta}}{n}},
\]
where we use \Cref{lem:hoeffding} for $\epsilon_{Di}$.
For $\frac{1}{n}\sum\epsilon_{Di}\epsilon_{Si}$, from \Cref{lem:product+subG}, we know that $\epsilon_{Di}\epsilon_{Si}$ is $\sqrt{3}\eta_D\eta_S$-sub-exponential. Here we use independence conditional on the filtration, and the same argument applies throughout. With \Cref{lem:subE}, it holds that with probability at least $1-\delta$,
\[
|\frac{1}{n}\sum\epsilon_{Di}\epsilon_{Si}|\le \sqrt{\frac{6\eta_D^2\eta_S^2\log\frac{2\lceil \log_2 T \rceil}{\delta}}{n}},
\]
as long as $n\ge 2\log\frac{2\lceil \log_2 T \rceil}{\delta}$.

For the numerator of $q_3$, we know that for $\xi_{Di}$, the buyer needs to wait for $n-i+1$ rounds for the next update. Hence, from  the proof of \Cref{prop:bound_strategy}, it holds that
\[
    |\xi_{Di}|\le \sqrt{\frac{\gamma^{n-i+1} (\alpha_1-\beta)(B^2+\eta_D^2)}{(1-\gamma)|\beta|}}.
\]
Then, it holds that
\[
|\frac{1}{n}\sum \xi_{Di}(\alpha_0-\hat\alpha_0)|\lesssim \cO(\frac{1}{n(1-\gamma)}\sqrt{\frac{2\eta_S^2\log\frac{2\lceil \log_2 T \rceil}{\delta}}{n}})\lesssim\cO(\frac{\eta_S}{(1-\gamma)n^{3/2}}\sqrt{\log\frac{\log T}{\delta}}).
\]
Moreover, with probability at least $1-\delta$, it holds that $|\epsilon_{Si}|\le\sqrt{2\eta_S^2\log\frac{2T}{\delta}}$ due to \Cref{lem:hoeffding}. Then, we have
\[
|\frac{1}{n}\sum \xi_{Di}\epsilon_{Si}|\lesssim \cO(\frac{\eta_S\sqrt{\log\frac{T}{\delta}}}{n(1-\gamma)}).
\]
Combining these two parts, we know that
\[
|\frac{1}{n}\sum \xi_{Di}(\epsilon_{Si}+\alpha_0-\hat\alpha_0)|\lesssim\cO(\frac{\eta_S\sqrt{\log\frac{T}{\delta}}}{n(1-\gamma)}).
\]

For the denominator, we know that
$$\begin{aligned}
   -\frac{1}{n}\sum Q_i^e(\epsilon_{Si}+\alpha_0-\hat\alpha_0)\ge& -\frac{1}{n}\sum Q_i^e \epsilon_{Si}-\bar Q\sqrt{\frac{2\eta_S^2\log\frac{2\lceil \log_2 T \rceil}{\delta}}{n}}\\
   =&\frac{1}{n}\sum\frac{\epsilon_{Si}^2}{2(\alpha_1-\beta)}-\frac{1}{n}\sum\frac{(2f(x_i)-\hat f(x_i)+2\epsilon_{Di}+2\xi_{Di}-\alpha_0)\epsilon_{Si}}{2(\alpha_1-\beta)}\\
   &-\bar Q\sqrt{\frac{2\eta_S^2\log\frac{2\lceil \log_2 T \rceil}{\delta}}{n}},
\end{aligned}$$
since $Q_i^e$, which is determined by \Cref{prop:equ,prop:opta}, is bounded by $\bar Q\lesssim\cO(\sqrt{\log(\frac{T}{\delta})}+\frac{1}{\sqrt{1-\gamma}})$.

From the above, we know that with probability at least $1-2\delta$ (one $\delta$ is for $\frac{1}{n}\sum (2f(x_i)-\hat f(x_i))\epsilon_{Si}$ and the other is for $\frac{1}{n}\sum \alpha_0\epsilon_{Si}$),
\[
|\frac{1}{n}\sum\frac{(2f(x_i)-\hat f(x_i)+2\epsilon_{Di}+2\xi_{Di}-\alpha_0)\epsilon_{Si}}{2(\alpha_1-\beta)}|\lesssim\cO\left(\sqrt{\frac{\eta_S^2\log\frac{\log T}{\delta}}{n}}+\frac{\sqrt{\eta_S^2\log\frac{2T}{\delta}}}{n(1-\gamma)}\right).
\]

Meanwhile, from \Cref{lem:square_subE}, we know that $\epsilon_{Si}^2-\sigma_S^2$ follows a $16\eta_S$-sub-exponential distribution. Therefore, we know that with probability at least $1-\delta$,
\[
|\frac{1}{n}\sum \epsilon_{Si}^2-\sigma_S^2|\le \sqrt{\frac{512\eta_S^4\log\frac{2\lceil \log_2 T \rceil}{\delta}}{n}},
\]
as long as $n\ge 2\log\frac{2\lceil \log_2 T \rceil}{\delta}$ due to \Cref{lem:subE}.
It then holds that
\[
\frac{1}{n}\sum\frac{\epsilon_{Si}^2}{2(\alpha_1-\beta)}\ge \frac{1}{2(\alpha_1-\beta)}(\sigma_S^2-\sqrt{\frac{512\eta_S^4\log\frac{2\lceil \log_2 T \rceil}{\delta}}{n}}).
\]

Combining these parts, we know that 
\[
-\frac{1}{n}\sum Q_i^e(\epsilon_{Si}+\alpha_0-\hat\alpha_0)\ge\frac{\sigma_S^2}{4(\alpha_1-\beta)},
\]
as long as $n\gtrsim\Omega(\frac{\eta_S^2\log\frac{T}{\delta}\log\frac{\log T}{\delta}}{\sigma_S^4}+\frac{\eta_S^2\log\frac{\log T}{\delta}}{(1-\gamma)\sigma_S^4}+\frac{\eta_S\sqrt{\log\frac{T}{\delta}}}{(1-\gamma)\sigma_S^2})$.

Consequently, it holds that
\[
|\hat\beta-\beta|\lesssim\cO(\frac{\sqrt{\eta_S^2\log\frac{\log T}{\delta}}}{\sigma_S^2\sqrt{n}}+\frac{\sqrt{\eta_S^2\log \frac{T}{\delta}}}{(1-\gamma)\sigma_S^2n}),
\]
as long as $n\gtrsim\Omega(\frac{\eta_S^2\log\frac{T}{\delta}\log\frac{\log T}{\delta}}{\sigma_S^4}+\frac{\eta_S^2\log\frac{\log T}{\delta}}{(1-\gamma)\sigma_S^4}+\frac{\eta_S\sqrt{\log\frac{T}{\delta}}}{(1-\gamma)\sigma_S^2})$. Here, we use the fact that as long as $n\gtrsim\Omega(\log\frac{\log T}{\delta})$, it holds that $\frac{\sqrt{\log\frac{\log T}{\delta}}}{\sqrt{n}}\gtrsim\Omega(\frac{\log\frac{\log T}{\delta}}{n})$.

When $\delta\eqsim\Theta(\frac{1}{T})$, we then obtain that
\[
|\hat\beta-\beta|\lesssim\cO(\frac{\eta_S\sqrt{\log T}}{\sigma_S^2\sqrt{n}}+\frac{\eta_S\sqrt{\log T}}{(1-\gamma)\sigma_S^2 n}),
\]
as long as $n\gtrsim\Omega(\frac{\eta_S^2\log ^2 T}{\sigma_S^4}+\frac{\eta_S^2\log T}{(1-\gamma)\sigma_S^4})$.
\end{proof}

\begin{lemma}\label{lem:bound_f_square}
    Let $\hat f$ be the estimate of $f$ after the $m$-th episode. Under the success event and \Cref{ass:context}, it holds that with probability at least $1-2\delta$,
    \[
    \E(\hat f(x)-f(x))^2\lesssim \cO(\frac{dim}{n}\log n+\frac{\eta_S^2\log\frac{\log T}{\delta}\log\frac{ T}{\delta}}{n\sigma_S^4}+\frac{\eta_S^2\log\frac{\log T}{\delta}}{n\sigma_S^4(1-\gamma)}
    +\frac{\eta_S^2\log^2(\frac{T}{\delta})}{n^2\sigma_S^4(1-\gamma)^2}
    +\frac{\eta_S^2\log\frac{T}{\delta}}{n^2\sigma_S^4(1-\gamma)^3}),
    \]
    whenever $2^m\gtrsim\Omega(dim\log(dim)+\frac{\eta_S^2\log\frac{T}{\delta}\log\frac{\log T}{\delta}}{\sigma_S^4}+\frac{\eta_S^2\log\frac{\log T}{\delta}}{(1-\gamma)\sigma_S^4}+\frac{\eta_S\sqrt{\log\frac{T}{\delta}}}{(1-\gamma)\sigma_S^2})$. Here, the expectation is taken over the distribution of market features.
\end{lemma}
\begin{proof}{Proof}
    We define a new function class $\cG=\{\frac{(\Tilde{f}-f)^2}{4B^2}:\Tilde{f}\in \cF\}$. Since we know all functions in $\cF$ are bounded by $B$, we know that with probability at least $1-\delta$,
    \[
    \E g-\frac{2}{n}\sum g(x_i)\lesssim\cO(\bar \delta(\cG)^2+\frac{\log\frac{\log T}{\delta}+\log\log n}{n})
    \]
    due to \Cref{lem:lemma36}. Here, $\bar\delta(\cG)$ is the critical radius for the class $\cG$.

    We know that $\frac{\Tilde{f}-f}{2B}\in[-1,1]$, then from \Cref{lem:squareclass}, it holds that $\bar\delta(\cG)$ is upper bounded by a solution to
    \[
    \frac{12}{\sqrt{n}}\int_{u/16}^1\sqrt{\log N(\frac{\cF}{2B},\|\cdot\|_\infty,x/2)}dx\le u/4.
    \]
    Here, we use the fact that for any given $f$, $\cF-f$ and $\cF$ have the same covering number. Then, we know that
    \[
    \frac{12}{\sqrt{n}}\int_{u/16}^1\sqrt{\log N(\frac{\cF}{2B},\|\cdot\|_\infty,x/2)}dx\lesssim\cO(\frac{1}{\sqrt{n}}\int_{u/16}^1\sqrt{dim\log\frac{1}{x}}dx)\lesssim\cO(\frac{1}{\sqrt{n}}\sqrt{dim\log\frac{1}{u}}).
    \]
    Therefore, we know that there exists a solution satisfying $u\lesssim\cO(\sqrt{\frac{dim\log n}{n}})$. Hence, we know that 
    \[
    \bar\delta(\cG)\lesssim \cO(\sqrt{\frac{dim\log n}{n}}).
    \]
    Then, we know that 
    \[
    \E g-\frac{2}{n}\sum g(x_i)\lesssim\cO(\frac{dim\log n+\log\frac{\log T}{\delta}}{n}),
    \]
    for all $g\in\cG$, specially $\frac{(\hat f-f)^2}{4B^2}$,
    and we only need to bound $\frac{1}{n}\sum g(x_i)$.

    Note that $\hat f$ is the ERM, it holds that
    \[
    \frac{1}{n}\sum (\hat f(x_i)-f(x_i)-\epsilon_{Di}-\xi_{Di}-(\beta-\hat\beta)Q_i^e)^2\le \frac{1}{n}\sum (-\epsilon_{Di}-\xi_{Di}-(\beta-\hat\beta)Q_i^e)^2.
    \]
    Therefore, we know that
    \[
    \frac{1}{n}\sum(\hat f(x_i)-f(x_i))^2\le \frac{2}{n}\sum (\hat f(x_i)-f(x_i))(\epsilon_{Di}+\xi_{Di}+(\beta-\hat\beta)Q_i^e).
    \]
    First, we know that
    \[
     |\frac{2}{n}\sum(\hat f(x_i)-f(x_i))\xi_{Di}|\lesssim\cO(\frac{B}{n(1-\gamma)}),
    \]
    and due to $ab\le\frac{a^2}{4}+b^2$,
    \[
    |\frac{2}{n}\sum (\hat f(x_i)-f(x_i))(\beta-\hat\beta)Q_i^e|\le \frac{\frac{1}{n}\sum (\hat f(x_i)-f(x_i))^2}{2}+\frac{2}{n}\sum (\beta-\hat\beta)^2\bar Q^2. 
    \]
    It yields that
    $$\begin{aligned}
        & \quad\; \frac{1}{n}\sum(\hat f(x_i)-f(x_i))^2-4\langle \hat f-f,\epsilon_{D}\rangle_n \\
        & \lesssim\cO(\frac{\eta_S^2\log\frac{\log T}{\delta}\log\frac{ T}{\delta}}{n\sigma_S^4}+\frac{\eta_S^2\log\frac{\log T}{\delta}}{n\sigma_S^4(1-\gamma)}
    +\frac{\eta_S^2\log^2(\frac{T}{\delta})}{n^2\sigma_S^4(1-\gamma)^2}
    +\frac{\eta_S^2\log\frac{T}{\delta}}{n^2\sigma_S^4(1-\gamma)^3}
    ).
    \end{aligned}$$
    Here, $\langle\cdot,\cdot\rangle$ is associated with $x_1,...,x_n$ and $\epsilon_{D1},...,\epsilon_{Dn}$.

    We define $q_4=\sup_{h\in\cF-f}8\langle h,\epsilon_D\rangle_n-\|h\|_n^2$. Besides, we construct an auxiliary function class $\cF_f=\{\lambda(\Tilde{f}-f):\Tilde{f}\in\cF\ \text{and }\lambda\in[0,1]\}$. Note that it's homeomorphic to $\lambda \cF+(1-\lambda)f$ as we can always add it by $f$. 

    For every $\Tilde{f}\in\cF$, we assume that there exists $h_{\Tilde{f}}\in\cF$ such that $\|\Tilde{f}-h_{\Tilde{f}}\|_\infty\le\frac{\epsilon}{2}$. Then, considering $S=\{\frac{k\epsilon}{4B}:k\in\{0,1,...,\lfloor\frac{4B}{\epsilon}\rfloor\}\}$, it holds that for every $\lambda\in[0,1]$, there exists $\lambda_S\in S$ such that $|\lambda_S-\lambda|\le\frac{\epsilon}{4B}$. Therefore, we know
    \[
    \|\lambda(\Tilde{f}-f)-\lambda_S(h_{\Tilde{f}}-f)\|_\infty\le |\lambda-\lambda_S|*\|\Tilde{f}-f\|_\infty+\lambda_S\|\Tilde{f}-h_{\Tilde{f}}\|_\infty\le \frac{\epsilon}{4B}*2B+1*\frac{\epsilon}{2}=\epsilon.
    \]
    Consequently, it holds that
    \[
    N(\cF_f,\|\cdot\|_\infty,\epsilon)\le N(\cF,\|\cdot\|_\infty,\frac{\epsilon}{2})*\frac{4B}{\epsilon}\lesssim\cO((\frac{1}{\epsilon})^{dim+1}).
    \]
    Note that it's only slightly larger than $\cF$, namely, the effective dimension becomes $dim+1$. This approximate equivalence allows us to bypass the star-shapedness assumption.

    We define $Z(u,\cF)=\sup_{h\in\cF, \|h\|_n\le u}\langle h,\epsilon_D\rangle_n$ and $G(u,\cF)=\E Z(u,\cF)$, where the expectation is taken over $\epsilon_{D}$. Notice that
    \[
    |\sup_{h\in\cF, \|h\|_n\le u}\langle h,\epsilon_D\rangle_n-\sup_{h\in\cF, \|h\|_n\le u}\langle h,\tilde\epsilon_D\rangle_n|\le \frac{u}{\sqrt n}\sqrt{\sum(\epsilon_{Di}-\tilde\epsilon_{Di})^2}.
    \]
    It tells us that $Z(u,\cF)$ is $\frac{u}{\sqrt{n}}$-Lipschitz with respect to vector $\epsilon_{D}$ and its Euclidean norm. We assume that $\epsilon_D$ is $l$-log-concave, then we know from \Cref{lemma:3.16}
    \begin{equation}\label{eq:GandZ}
            \PP(|Z(u,\cF)-G(u,\cF)|\ge \sqrt{\frac{4u^2\log\frac{2\lceil \log_2 T \rceil}{\delta}}{nl}})\le \frac{\delta}{2\lceil \log_2 T \rceil},
    \end{equation}
    for any fixed class $\cF$.
    We define critical radius $\bar\delta(\cF)$ to be the infimum of all positive solutions satisfying $G(u,\cF)\le\frac{u^2}{2}$. Actually, in the following proof, we can choose any solution as long as $G(\bar\delta(\cF),\cF)\le\frac{(\bar\delta(\cF))^2}{2}$.
    Moreover, since $\cF-f\subseteq \cF_f$, we know that $Z(u,\cF-f)\le Z(u,\cF_f)$ and $G(u,\cF-f)\le G(u,\cF_f)$.

    When $\|h\|_n\le\bar\delta(\cF_f)$, it holds that 
    \[
    q_4\le 8 Z(\bar\delta(\cF_f),\cF-f)\le 8Z(\bar\delta(\cF_f),\cF_f).
    \]
    We know from \Cref{eq:GandZ} that with probability at least $1-\delta$, it holds that
    \[
    Z(\bar\delta(\cF_f),\cF_f)\le G(\bar\delta(\cF_f),\cF_f)+\sqrt{\frac{4\log\frac{2\lceil \log_2 T \rceil}{\delta}}{nl}}\bar\delta(\cF_f).
    \]
    Then, we know that
    \[
    q_4\le 4(\bar\delta(\cF_f))^2+8\sqrt{\frac{4\log\frac{2\lceil \log_2 T \rceil}{\delta}}{nl}}\bar\delta(\cF_f).
    \]
    Otherwise, for the case $\|h\|_n\ge \bar\delta(\cF_f)$, we assume that $r=\frac{\bar\delta(\cF_f)}{\|h\|_n}\in[0,1]$.
    It then holds that
    \[
    q_4=\frac{8}{r}\langle \epsilon_D, \frac{\bar\delta(\cF_f)}{\|h\|_n}h\rangle_n-\frac{(\bar\delta(\cF_f))^2}{r^2}\le \frac{8}{r}Z(\bar\delta(\cF_f),\cF_f)-\frac{(\bar\delta(\cF_f))^2}{r^2}\le 16(\frac{Z(\bar\delta(\cF_f),\cF_f)}{\bar\delta(\cF_f)})^2.
    \]
    Here, we use the fact that $\frac{\bar\delta(\cF_f)}{\|h\|_n}h\in\cF_f$ and $\|\frac{\bar\delta(\cF_f)}{\|h\|_n}h\|_n=\bar\delta(\cF_f)$. Given the event that $ Z(\bar\delta(\cF_f),\cF_f)\le G(\bar\delta(\cF_f),\cF_f)+\sqrt{\frac{4\log\frac{2\lceil \log_2 T \rceil}{\delta}}{nl}}\bar\delta(\cF_f)$ and $G(\bar\delta(\cF_f),\cF_f)\le \frac{(\bar\delta(\cF_f))^2}{2}$, it holds that
    \[
    q_4\le 16(\frac{\bar\delta(\cF_f)}{2}+\sqrt{\frac{4\log\frac{2\lceil \log_2 T \rceil}{\delta}}{nl}})^2.
    \]
    Therefore, we know that 
    \[
    q_4\lesssim\cO((\bar\delta(\cF_f))^2+\frac{4\log\frac{2\lceil \log_2 T \rceil}{\delta}}{nl}).
    \]

    We now only need to estimate the magnitude of $\bar\delta(\cF_f)$. From \Cref{lem:inner}, we know that
    \[
    G(u,\cF_f)\lesssim\cO(\frac{1}{\sqrt{n}}\int_0^{u}\sqrt{\log N(\cF_f,\|\cdot\|_\infty,x)}dx)\lesssim\cO(\frac{\sqrt{dim}}{\sqrt n}\int_0^{u}1+\sqrt{\log\frac{1}{x}}dx),
    \]
    where we use the fact that $N(\cF_f,\|\cdot\|_\infty,\epsilon)\lesssim\cO((\frac{1}{\epsilon})^{dim+1})$. Notice that for $u\le 1$,
    \[
    \int_0^{u}\sqrt{\log\frac{1}{x}}dx=\int_{\log(1/u)}^\infty \sqrt{x}e^{-x}dx=\Gamma(\frac{3}{2},\log(\frac{1}{u})),
    \]
    where $\Gamma(\cdot,\cdot)$ is the upper incomplete gamma function. From the knowledge of analysis~\citep{abramowitz1968handbook}, we know that $\Gamma(\frac{3}{2},x)\lesssim\cO(\sqrt{x}e^{-x})$ as long as $x\gtrsim\Omega(1)$.
    Hence, we know that
    \[    G(u,\cF_f)\lesssim\cO\left(\frac{\sqrt{dim}}{\sqrt{n}}u\sqrt{\log(\frac{1}{u}})\right).
    \]
    Consequently, there exists an $u\lesssim\cO(\sqrt{\frac{dim}{n}\log n})$ such that $G(u,\cF_f)\le\frac{u^2}{2}$. So, it means that we can choose $\bar\delta(\cF_f)\lesssim\cO(\sqrt{\frac{dim}{n}\log n})$. Here, we only need $n\gtrsim\Omega(dim\log(dim))$.

    Finally, we obtain that
    \[
    q_4\lesssim\cO(\frac{dim}{n}\log n+\frac{\log\frac{\log T}{\delta}}{n}),
    \]
    which yields
    $$\begin{aligned}
        & \quad\; \frac{1}{n}\sum(\hat f(x_i)-f(x_i))^2 \\
        & \lesssim\cO(\frac{dim}{n}\log n+\frac{\eta_S^2\log\frac{\log T}{\delta}\log\frac{ T}{\delta}}{n\sigma_S^4}+\frac{\eta_S^2\log\frac{\log T}{\delta}}{n\sigma_S^4(1-\gamma)} +\frac{\eta_S^2\log^2(\frac{T}{\delta})}{n^2\sigma_S^4(1-\gamma)^2} +\frac{\eta_S^2\log\frac{T}{\delta}}{n^2\sigma_S^4(1-\gamma)^3}).
    \end{aligned}$$

    Since we know $\E \frac{(\hat f(x)-f(x))^2}{4B^2}-\frac{2}{n}\sum\frac{(\hat f(x_i)-f(x_i))^2}{4B^2}\lesssim\cO(\frac{dim\log n+\log\frac{\log T}{\delta}}{n})$, it holds that
    \[
    \E (\hat f(x)-f(x))^2\lesssim\cO(\frac{dim}{n}\log n+\frac{\eta_S^2\log\frac{\log T}{\delta}\log\frac{ T}{\delta}}{n\sigma_S^4}+\frac{\eta_S^2\log\frac{\log T}{\delta}}{n\sigma_S^4(1-\gamma)}
    +\frac{\eta_S^2\log^2(\frac{T}{\delta})}{n^2\sigma_S^4(1-\gamma)^2}
    +\frac{\eta_S^2\log\frac{T}{\delta}}{n^2\sigma_S^4(1-\gamma)^3}),
    \]
    which ends the proof.
\end{proof}

\subsubsection{Proof of Theorem \ref{thm:partially}}
We first bound the regret in the $(m+1)$-th episode, denoted by $\text{Regret}_{m+1}$. The length of this episode is $2^{m+1}=2n$. 
The start round is $s=2^{m+1}-1$ and the end round is $e=2^{m+2}-2$.
Recall that 
\[
\text{Regret}_{m+1}=\sum_{i=1}^{2n}\frac{1}{\alpha_1-\beta}(a_i^*-a_i)^2-\frac{\xi_{Di} a_i}{\alpha_1-\beta},
\]
and $|a_i^*-a_i|\lesssim\cO(|f(x_i)-\hat f(x_i)|)$. 

We first bound $\sum_{i=1}^{2n}|\xi_{Di}a_i|$. Given the success event that $|\epsilon_{Si}|\le\sqrt{2\eta_S^2\log\frac{2T}{\delta}}$, it holds that
\[
|\sum_{i=1}^{2n}\frac{\xi_{Di} a_i}{\alpha_1-\beta}|\lesssim \cO(\frac{\eta_S\sqrt{\log\frac{T}{\delta}}}{1-\gamma}),
\]
due to \Cref{prop:bound_strategy}.

For the other term $\sum_{i=1}^{2n} (f(x_i)-\hat f(x_i))^2$, we first bound its expectation over the feature vector that
$$\begin{aligned}
    \E \sum_{i=1}^{2n} (f(x_i)-\hat f(x_i))^2 & \lesssim\cO\left(dim\log T\log \frac{e}{s}+\frac{\eta_S^2\log\frac{\log T}{\delta}\log\frac{ T}{\delta}\log \frac{e}{s}}{\sigma_S^4}+\frac{\eta_S^2\log\frac{\log T}{\delta}\log \frac{e}{s}}{\sigma_S^4(1-\gamma)}\right. \\
    & \quad \quad \left. +\frac{\eta_S^2\log^2(\frac{T}{\delta})(\frac{1}{s}-\frac{1}{e})}{\sigma_S^4(1-\gamma)^2}
    +\frac{\eta_S^2\log\frac{T}{\delta}(\frac{1}{s}-\frac{1}{e})}{\sigma_S^4(1-\gamma)^3}\right),
\end{aligned}$$
where we use the fact that $2n\le T$. We use $X_i$ to denote $\frac{(f(x_i)-\hat f(x_i))^2}{4B^2}$. Then, we know that $X_1,...,X_{2n}$ are i.i.d. and bounded by $[0,1]$. Also, it holds that $\E X_i\lesssim\cO(\frac{dim}{n}\log n+\frac{\eta_S^2\log\frac{\log T}{\delta}\log\frac{ T}{\delta}}{n\sigma_S^4}+\frac{\eta_S^2\log\frac{\log T}{\delta}}{n\sigma_S^4(1-\gamma)}
    +\frac{\eta_S^2\log^2(\frac{T}{\delta})}{n^2\sigma_S^4(1-\gamma)^2}
    +\frac{\eta_S^2\log\frac{T}{\delta}}{n^2\sigma_S^4(1-\gamma)^3})$. Note that when $X_i\in[0,1]$, we have $\Var(X_i)=\E[X_i^2]-(\E[X_i])^2\le \E[X_i]$. It means that $\Var(X_i)\lesssim\cO(\frac{dim}{n}\log n+\frac{\eta_S^2\log\frac{\log T}{\delta}\log\frac{ T}{\delta}}{n\sigma_S^4}+\frac{\eta_S^2\log\frac{\log T}{\delta}}{n\sigma_S^4(1-\gamma)}
    +\frac{\eta_S^2\log^2(\frac{T}{\delta})}{n^2\sigma_S^4(1-\gamma)^2}
    +\frac{\eta_S^2\log\frac{T}{\delta}}{n^2\sigma_S^4(1-\gamma)^3})$. Then, from \Cref{lem:theorem8}, it holds that with probability at least $1-\delta$,
    $$\begin{aligned}
    &\quad\frac{1}{2n}\sum_{i=1}^{2n}X_i-\E [\frac{1}{2n}\sum_{i=1}^{2n}X_i]\le \sqrt{\frac{2\log\frac{\lceil \log_2 T \rceil}{\delta}\Var(X_i)}{2n}}+\frac{\log\frac{\lceil \log_2 T \rceil}{\delta}}{6n}\\
    &\lesssim\cO\left(
\frac{\sqrt{dim\log\frac{\log T}{\delta}\log T}}{n}+
\frac{\eta_S\log\frac{\log T}{\delta}\sqrt{\log\frac{T}{\delta}}}{n\sigma_S^2}+
\frac{\eta_S\log\frac{\log T}{\delta}}{n\sigma_S^2\sqrt{1-\gamma}}+  
\frac{\eta_S\log\frac{T}{\delta}\sqrt{\log\frac{\log T}{\delta}}}{n^{3/2}\sigma_S^2(1-\gamma)}
+\frac{\eta_S\sqrt{\log\frac{T}{\delta}\log\frac{\log T}{\delta}}}{n^{3/2}\sigma_S^2(1-\gamma)^{3/2}}
    \right).
    \end{aligned}$$
    Hence, we know that 
    $$\begin{aligned}
    &\sum_{i=1}^{2n} (f(x_i)-\hat f(x_i))^2\\
    &\lesssim\cO(dim\log T\log \frac{e}{s}+\frac{\eta_S^2\log\frac{\log T}{\delta}\log\frac{ T}{\delta}\log \frac{e}{s}}{\sigma_S^4}+\frac{\eta_S^2\log\frac{\log T}{\delta}\log \frac{e}{s}}{\sigma_S^4(1-\gamma)}
    +\frac{\eta_S^2\log^2(\frac{T}{\delta})(\frac{1}{s}-\frac{1}{e})}{\sigma_S^4(1-\gamma)^2}
    +\frac{\eta_S^2\log\frac{T}{\delta}(\frac{1}{s}-\frac{1}{e})}{\sigma_S^4(1-\gamma)^3}\\
    &+\frac{\eta_S\log\frac{T}{\delta}\sqrt{\log\frac{\log T}{\delta}}}{n^{1/2}\sigma_S^2(1-\gamma)}
+\frac{\eta_S\sqrt{\log\frac{T}{\delta}\log\frac{\log T}{\delta}}}{n^{1/2}\sigma_S^2(1-\gamma)^{3/2}}+\sqrt{dim\log\frac{\log T}{\delta}\log T}).
    \end{aligned}$$
    Therefore, it holds that
    $$\begin{aligned}
            &\text{Regret}_{m+1}\\
            &\lesssim\cO(dim\log T\log \frac{e}{s}+\frac{\eta_S^2\log\frac{\log T}{\delta}\log\frac{T}{\delta}\log\frac{ e}{ s}}{\sigma_S^4}+\frac{\eta_S^2\log\frac{\log T}{\delta }\log\frac{e}{s}}{\sigma_S^4(1-\gamma)}
    +\frac{\eta_S^2\log^2(\frac{T}{\delta})(\frac{1}{s}-\frac{1}{e})}{\sigma_S^4(1-\gamma)^2}
    +\frac{\eta_S^2\log\frac{T}{\delta}(\frac{1}{s}-\frac{1}{e})}{\sigma_S^4(1-\gamma)^3}\\
    &+\frac{\eta_S\sqrt{\log\frac{T}{\delta}}}{1-\gamma}
    +\frac{\eta_S\log\frac{T}{\delta}\sqrt{\log\frac{\log T}{\delta}}}{2^{m/2}\sigma_S^2(1-\gamma)}
+\frac{\eta_S\sqrt{\log\frac{T}{\delta}\log\frac{\log T}{\delta}}}{2^{m/2}\sigma_S^2(1-\gamma)^{3/2}}+\sqrt{dim\log\frac{\log T}{\delta}\log T}
    )
    \end{aligned}$$
    for all $m\le\lceil \log_2 T \rceil$. Then, we know that the regret due to estimation error is at most
    $$\begin{aligned}
    &\sum_{m=1}^{\lceil \log_2 T \rceil}\text{Regret}_{m}\\
    &\lesssim
    \cO(dim\log^2 T+\frac{\eta_S^2\log\frac{\log T}{\delta}\log\frac{ T}{\delta}\log T}{\sigma_S^4}+\frac{\eta_S^2\log\frac{\log T}{\delta}\log T}{\sigma_S^4(1-\gamma)}
    +\frac{\eta_S^2\log^2(\frac{T}{\delta})}{\sigma_S^4(1-\gamma)^2}
    +\frac{\eta_S^2\log\frac{T}{\delta}}{\sigma_S^4(1-\gamma)^3}\\
    &+\frac{\log T\sqrt{\log\frac{T}{\delta}}}{1-\gamma}+\sqrt{dim\log\frac{\log T}{\delta}\log T}\log T).
    \end{aligned}$$
Note that we also need $2^m\gtrsim\Omega(dim\log(dim)+\frac{\eta_S^2\log\frac{T}{\delta}\log\frac{\log T}{\delta}}{\sigma_S^4}+\frac{\eta_S^2\log\frac{\log T}{\delta}}{(1-\gamma)\sigma_S^4}+\frac{\eta_S\sqrt{\log\frac{T}{\delta}}}{(1-\gamma)\sigma_S^2})$ from \Cref{lem:bound_f_square}, then it holds that
$$\begin{aligned}
&\text{Regret}(T)\\
&\lesssim\cO(dim\log(dim)+dim\log^2 T+\frac{\eta_S^2\log\frac{\log T}{\delta}\log\frac{ T}{\delta}\log T}{\sigma_S^4}+\frac{\eta_S^2\log\frac{\log T}{\delta}\log T}{\sigma_S^4(1-\gamma)}
    +\frac{\eta_S^2\log^2(\frac{T}{\delta})}{\sigma_S^4(1-\gamma)^2}\\
    &+\frac{\eta_S^2\log\frac{T}{\delta}}{\sigma_S^4(1-\gamma)^3}+\frac{\log T\sqrt{\log\frac{T}{\delta}}}{1-\gamma}+\sqrt{dim\log\frac{\log T}{\delta}\log T}\log T),
\end{aligned}$$
as we only suffer from $\cO(1)$ regret every round if some inequalities are violated.

By choosing $\delta=\frac{1}{12T}$, we know that with probability at least $1-\frac{1}{T}$, it holds that
\[
\text{Regret}(T)\lesssim\cO\left(dim\log(dim)+dim\log^2 T+
\frac{\eta_S^2\log^3 T}{\sigma_S^4}
    +\frac{\eta_S^2\log^2T}{\sigma_S^4(1-\gamma)^2}
    +\frac{\eta_S^2\log T}{\sigma_S^4(1-\gamma)^3}\right),
\]
which ends the proof.

\subsubsection{Lemmas Supporting Theorem \ref{thm:no_noise}}
Analogously, we consider the case after the $m$-th episode and recall that there are $n=2^m$ tuples in $\cD$.
\begin{lemma}\label{lem:beta_no_noise}
    Let $\hat\beta$ be the estimate of $\beta$ after the $m$-th episode. Under the success event and \Cref{ass:context}, it holds that with probability at least $1-6\delta$,
    \[
    |\hat \beta-\beta|\lesssim\cO\left(\frac{\sqrt{\log\frac{\log T}{\delta}}}{n^{1/4}}+\frac{\sqrt{\log \frac{T}{\delta}}}{(1-\gamma)n^{3/4}}\right),
    \]
    whenever $2^m\gtrsim\Omega((\frac{\log\frac{T}{\delta}}{(1-\gamma)^2})^{2/3}+(\log\frac{\log T}{\delta})^2)$.
\end{lemma}
\begin{proof}{Proof}
    Here, since there is no noise in the supply, we know that $\hat\alpha_0=\alpha_0$, it then holds that
    \[
    \hat\beta-\beta=\underbrace{\frac{\sum (f(x_i)-\hat f(x_i))\epsilon_i}{\sum Q_i^e \epsilon_i}}_{q_5}+\underbrace{\frac{\sum \epsilon_{Di}\epsilon_i}{\sum Q_i^e \epsilon_i}}_{q_6}+\underbrace{\frac{\sum \xi_{Di}\epsilon_i}{\sum Q_i^e \epsilon_i}}_{q_7},
    \]
    where $\epsilon_i$ is the corresponding artificial randomness.

    For the numerator of $q_5$, we assume that the variance of $\epsilon_i$ is $\sigma_i^2$. It then holds that with probability at least $1-\delta$,
    \[
    |\sum (f(x_i)-\hat f(x_i))\epsilon_i|=|\sum (f(x_i)-\hat f(x_i))\sigma_i\frac{\epsilon_i}{\sigma_i}|\lesssim\cO(\sqrt{\sum \sigma_i^2\log\frac{\log T}{\delta}})\lesssim\cO(\sqrt{\sqrt{n}\log\frac{\log T}{\delta}}),
    \]
    due to \Cref{lem:hoeffding} and $\sigma_i^2=\frac{1}{\sqrt{i}}$.

    For the numerator of $q_6$, we know that due to \Cref{lem:product+subG}, $\epsilon_{Di}\epsilon_i$ is $\sqrt{3}\eta_D\sigma_i$-sub-exponential. Using the independence concerning time index $i$, we know that $\sum \epsilon_{Di}\epsilon_i$ is $\sqrt{\sum 3\eta_D^2\sigma_i^2}$-sub-exponential. Hence, it holds that with probability at least $1-\delta$,
    \[
    |\sum \epsilon_{Di}\epsilon_i|\le\max\{\sqrt{6\eta_D^2\sum\sigma_i^2\log\frac{2\lceil\log_2 T\rceil}{\delta}},2\log\frac{2\lceil\log_2 T\rceil}{\delta}\max_j \sqrt{3}\eta_D\sigma_j\}\lesssim\cO(\sqrt{\sqrt{n}\log\frac{\log T}{\delta}}),
    \]
due to \Cref{lem:subE}. For the last inequality, we need $n\gtrsim\Omega((\log\frac{\log T}{\delta})^2)$. 

    For the numerator of $q_7$, it holds that with probability at least $1-\delta$,
    $$\begin{aligned}
        |\sum \xi_{Di}\epsilon_{i}|\le\sum |\xi_{Di}|\sigma_i\sqrt{2\log\frac{2T}{\delta}} & \le\sum_i \sqrt{\frac{\gamma^{n-i+1} (\alpha_1-\beta)(B^2+\eta_D^2)}{(1-\gamma)|\beta|}}\frac{\sum_j \sigma_j}{n}\sqrt{2\log\frac{2T}{\delta}} \\
        & \lesssim\cO(\frac{\sqrt{\log\frac{T}{\delta}}}{(1-\gamma)n^{1/4}}).
    \end{aligned}$$
    The second inequality holds because the upper bound of $\xi_{Di}$, say $\sqrt{\frac{\gamma^{n-i+1} (\alpha_1-\beta)(B^2+\eta_D^2)}{(1-\gamma)|\beta|}}$, is increasing and $\sigma_i$ is decreasing.

    For the denominator, we know that 
    \[
    -\sum Q_i^e\epsilon_i=\sum\frac{\epsilon_i^2}{\alpha_1-\beta} -\sum\frac{(2f(x_i)-\hat f(x_i)+2\epsilon_{Di}+2\xi_{Di}-\alpha_0)}{2(\alpha_1-\beta)}\epsilon_i.   \]
    Similarly, we know that with probability at least $1-2\delta$ (one $\delta$ is for $\sum (2f(x_i)-\hat f(x_i))\epsilon_i$ and the other is for $\sum \alpha_0\epsilon_i $),
    \[
    |\sum\frac{(2f(x_i)-\hat f(x_i)+2\epsilon_{Di}+2\xi_{Di}-\alpha_0)}{2(\alpha_1-\beta)}\epsilon_i|\lesssim\cO(\sqrt{\sqrt{n}\log\frac{\log T}{\delta}}+\frac{\sqrt{\log\frac{T}{\delta}}}{(1-\gamma)n^{1/4}}).
    \]
    Note that $\sum\epsilon_i^2-\sigma_i^2$ is $\sqrt{256\sum \sigma_i^4}$-sub-exponential due to \Cref{lem:square_subE}, it then holds that with probability at least $1-\delta$,
    \[
    |\sum\epsilon_i^2-\sigma_i^2|\le\max\{\sqrt{512\sum\sigma_i^4\log\frac{2\lceil \log_2 T\rceil}{\delta}},2\log\frac{\lceil 2\log_2 T\rceil}{\delta}\max_j 16\sigma_j^2\}.
    \]
    Consequently, it holds that 
    \[
    \sum\epsilon_i^2\gtrsim\Omega(\sqrt{n}),
    \]
    as long as $n\gtrsim\Omega((\frac{\log\frac{T}{\delta}}{(1-\gamma)^2})^{2/3}+(\log\frac{\log T}{\delta})^2)$.

    Therefore, combining all these parts, it holds that
    \[
    |\hat\beta-\beta|\lesssim\cO\left(\frac{\sqrt{\log\frac{\log T}{\delta}}}{n^{1/4}}+\frac{\sqrt{\log \frac{T}{\delta}}}{(1-\gamma)n^{3/4}}\right),
    \]
    with probability at least $1-6\delta$, which ends the proof.
\end{proof}
\begin{lemma}
    Let $\hat f$ be the estimate of $f$ after the $m$-th episode. Under the success event and \Cref{ass:context}, it holds that with probability at least $1-2\delta$,
    \[
    \E(\hat f(x)-f(x))^2\lesssim \cO(\frac{dim}{n}\log n+\frac{\log\frac{\log T}{\delta}\log\frac{ T}{\delta}}{\sqrt{n}}+\frac{\log\frac{\log T}{\delta}}{\sqrt{n}(1-\gamma)}
    +\frac{\log^2(\frac{T}{\delta})}{n^{3/2}(1-\gamma)^2}
    +\frac{\log\frac{T}{\delta}}{n^{3/2}(1-\gamma)^3}),
    \]
    whenever $2^m\gtrsim\Omega(dim\log(dim)+(\frac{\log\frac{T}{\delta}}{(1-\gamma)^2})^{2/3}+(\log\frac{\log T}{\delta})^2)$. Here, the expectation is taken over the distribution of market features.
\end{lemma}
\begin{proof}{Proof}
    Compared to the proof of \Cref{lem:bound_f_square}, we only need to change the bound for $\frac{2}{n}\sum(\hat\beta-\beta)^2\bar Q^2$. From \Cref{lem:beta_no_noise}, it now holds that
    \[
    \frac{2}{n}\sum(\hat\beta-\beta)^2\bar Q^2\lesssim\cO(
\frac{\log\frac{\log T}{\delta}\log\frac{T}{\delta}}{\sqrt{n}}
    +\frac{\log\frac{\log T}{\delta}}{\sqrt{n}(1-\gamma)}
    +\frac{(\log\frac{T}{\delta})^2}{n^{3/2}(1-\gamma)^2}
    +\frac{\log \frac{T}{\delta}}{n^{3/2}(1-\gamma)^3}
    ).
    \]
    Therefore, we know that
    \[
    \E(\hat f(x)-f(x))^2\lesssim\cO(\frac{dim}{n}\log n+\frac{\log\frac{\log T}{\delta}\log\frac{T}{\delta}}{\sqrt{n}}
    +\frac{\log\frac{\log T}{\delta}}{\sqrt{n}(1-\gamma)}
    +\frac{(\log\frac{T}{\delta})^2}{n^{3/2}(1-\gamma)^2}
    +\frac{\log \frac{T}{\delta}}{n^{3/2}(1-\gamma)^3}
    ),
    \]
    where we need an extra condition that $n\gtrsim\Omega(dim\log (dim))$. Together with the condition in \Cref{lem:beta_no_noise}, we need $n\gtrsim\Omega(dim\log(dim)+(\frac{\log\frac{T}{\delta}}{(1-\gamma)^2})^{2/3}+(\log\frac{\log T}{\delta})^2)$, which finishes our proof.
\end{proof}

\subsubsection{Proof of Theorem \ref{thm:no_noise}}
We also consider the regret in the $(m+1)$-episode and the start point and end point are $s$ and $e$, respectively.
Regret now comes from three sources. The first part is the estimation error of the optimal service fee, while the second part is from the artificial randomness. The last part is from the buyer's misreport. It holds that
\[
\text{Regret}_{m+1}\lesssim\cO(\sum_{i=1}^{2n}(\hat f(x_i)-f(x_i))^2+\epsilon_i^2+|\xi_{Di}a_i|).
\]

Analogous to the proof of \Cref{thm:partially}, we know that 
\[
\sum_{i=1}^{2n}|\xi_{Di}a_i|\lesssim\cO(\frac{\sqrt{\log\frac{T}{\delta}}}{1-\gamma}).
\]
Meanwhile, it holds that
\[
\sum_{i=1}^{2n}\epsilon_i^2\lesssim\cO(\sqrt{n}+\sqrt{\log n\log\frac{\log T}{\delta}}+\log\frac{\log T}{\delta}).
\]
Finally, we bound the term with respect to $\hat f$. It holds that
$$\begin{aligned}
    \sum_{i=1}^{2n}\E (\hat f(x_i)-f(x_i))^2 & \lesssim\cO\left(dim\log T\log\frac{e}{s}+(\log\frac{\log T}{\delta}\log\frac{T}{\delta}+\frac{\log\frac{\log T}{\delta}}{1-\gamma})(\sqrt{e}-\sqrt{s}) \right. \\
& \quad\quad \left. + (\frac{(\log\frac{T}{\delta})^2}{(1-\gamma)^2}
    +\frac{\log \frac{T}{\delta}}{(1-\gamma)^3})(\frac{1}{\sqrt{s}}-\frac{1}{\sqrt{e}})
\right).
\end{aligned}$$
Also, we know from \Cref{lem:theorem8} that with probability at least $1-\delta$,
$$\begin{aligned}
    &\quad \frac{1}{2n}\sum (\hat f(x_i)-f(x_i))^2-\E (\hat f(x_i)-f(x_i))^2\\
    &\lesssim\cO\left(\frac{\sqrt{dim\log\frac{\log T}{\delta}\log T}}{n}+\sqrt{\frac{\log^2(\frac{\log T}{\delta})\log\frac{T}{\delta}}{n^{3/2}}}+\sqrt{\frac{\log^2(\frac{\log T}{\delta})}{(1-\gamma)n^{3/2}}} \right. \\
    & \quad\quad \left.+\sqrt{\frac{\log\frac{\log T}{\delta}(\log\frac{T}{\delta})^2}{(1-\gamma)^2n^{5/2}}}+\sqrt{\frac{\log\frac{\log T}{\delta}\log\frac{ T}{\delta}}{(1-\gamma)^3n^{5/2}}}\right).
\end{aligned}$$

Combining all these terms, it holds that the total regret is bounded by

$$\begin{aligned}
    \sum_{i=1}^{\lceil\log_2 T\rceil}\text{Regret}_m & \lesssim\cO\left(dim\log^2 T+\sqrt{dim\log\frac{\log T}{\delta}\log T}\log T\right. \\
    & \quad\quad \left.+(\log\frac{\log T}{\delta}\log\frac{T}{\delta}+\frac{\log\frac{\log T}{\delta}}{1-\gamma})\sqrt{T} + \frac{(\log\frac{T}{\delta})^2}{(1-\gamma)^2}
    +\frac{\log \frac{T}{\delta}}{(1-\gamma)^3}\right).
\end{aligned}$$
Here, we need $2^{m}\gtrsim\Omega(dim\log(dim)+(\frac{\log\frac{T}{\delta}}{(1-\gamma)^2})^{2/3}+(\log\frac{\log T}{\delta})^2)$. As we've considered regret which originates from artificial randomness, violating this condition only brings $\cO(1)$ regret every round.
It then holds that 
$$\begin{aligned}
    \text{Regret}(T)\lesssim &\cO(dim\log^2 T+dim\log(dim)+\sqrt{dim\log\frac{\log T}{\delta}\log T}\log T\\
    & +(\log\frac{\log T}{\delta}\log\frac{T}{\delta}+\frac{\log\frac{\log T}{\delta}}{1-\gamma})\sqrt{T} + \frac{(\log\frac{T}{\delta})^2}{(1-\gamma)^2}
    +\frac{\log \frac{T}{\delta}}{(1-\gamma)^3}).
\end{aligned}$$

Setting $\delta=\frac{1}{10T}$, we know that with probability at least $1-\frac{1}{T}$,
\[
\text{Regret}(T)\lesssim\cO\left(dim\log(dim)+dim\log^2 T+\log^2 T\sqrt{T}+\frac{\log T\sqrt{T}}{1-\gamma}+\frac{\log^2 T}{(1-\gamma)^2}+\frac{\log T}{(1-\gamma)^3}
\right),
\]
which ends the proof.

\subsubsection{Proof of Theorem \ref{thm:lower_no_noise}}
We assume that the demand curve is now $P_{Dt}=\beta_0+\beta Q+\epsilon_{Dt}$ where $f$ is a constant function and $\epsilon_{Dt}$ is Gaussian. It's a special case for our general demand class. We also assume the buyer is myopic, say $\gamma=0$, as it's the loosest case.

Let's consider the following case. We assume the supply curve is $P_{St}=Q$ and the family of the demand curve parameterized by $\epsilon$ is $\{P_{\epsilon D}=20+5\epsilon-(1+\epsilon)Q+\epsilon_{\epsilon D}\}$ where $\epsilon_{\epsilon D}\sim\cN(0,(\frac{2+\epsilon}{2})^2)$.

The system is depicted completely by $(a_t,P_{\epsilon t}^e)$, where $Q_{\epsilon t}^e=P_{\epsilon t}^e-a_t$. Recall that $a_t$ depends on history up to $t-1$, namely, $\cH_{t-1}=\{h_\tau\}_{\tau=1}^{t-1}$. From \Cref{prop:equ}, we know that for any service fee $a_t$, transacted quantity $Q^e_{\epsilon t}=\frac{20+5\epsilon-a_t+\epsilon_{\epsilon Dt}}{2+\epsilon}$ and cutoff price $P^e_{\epsilon t}=\frac{20+5\epsilon+(1+\epsilon)a_t+\epsilon_{\epsilon Dt}}{2+\epsilon}$.
Then, we know that $P^e_{\epsilon t}\sim\cN(\frac{20+5\epsilon+(1+\epsilon)a_t}{2+\epsilon},\frac{1}{4})$ due to the definition of $\epsilon_{\epsilon Dt}$. Moreover, the ex-ante optimal service fee is $a^*_t=\frac{20+5\epsilon}{2}$. 

We now consider two situations when $\epsilon=0$ and $\epsilon=T^{-\frac{1}{4}}$, and we use $\PP_\epsilon$ to denote associated probability measure. Due to Gaussianity, it holds that
\[KL(\PP_0\|\PP_{\epsilon})=\E_a[KL(\PP_0(\cdot\given a)\|\PP_\epsilon(\cdot\given a))\given a]=\E_a[\frac{\epsilon^2(10-a)^2}{2(2+\epsilon)^2}],\]
where we use the law of iterated expectations.

During the whole $T$ rounds, we calculate the number of rounds with $|10-a_t|\ge 5\epsilon$ denoted by $T_0$. With a slight abuse of notation, we use $[T_0]$ to denote the corresponding index set. Let's consider the total regret associated with $T_0$ case by case.

On the one hand, when $T_0\ge \frac{T}{2}$, it holds that one-round suboptimality for $\epsilon=0$ is at least $\frac{(10-a_t)^2}{2}$ when $t\in[T_0]$. Then, we have
\[
    \text{Regret}_0(T)+\text{Regret}_\epsilon(T)\ge \text{Regret}_0(T)\ge T_0 \frac{25\epsilon^2}{2}\ge \frac{T}{2}\frac{25}{2\sqrt{T}}\gtrsim \Omega(\sqrt{T}).
\]
The first inequality holds due to the positivity of regret while the second inequality holds because when $t$ belongs to $[T_0]$, we have that $|10-a_t|\ge 5\epsilon$. The third inequality holds due to the assumption on $T_0$.

On the other hand, when $T_0<\frac{T}{2}$, there exist more than $\frac{T}{2}$ rounds such that $|10-a_t|\le 5\epsilon$. In these rounds, denoted by $[T-T_0]$, it holds that $KL(\PP_0\|\PP_\epsilon)\le \frac{25\epsilon^4}{2(2+\epsilon)^2}\le \frac{25\epsilon^4}{18}$ when $\epsilon\le 1$. The first inequality holds with no need to consider the distribution of $a_t$ because we use a union bound over all candidate $a_t$. Then, the total $KL$ among these $T-T_0$ rounds is no larger than $(T-T_0)\frac{25\epsilon^4}{18}\le \frac{25}{18}\lesssim\cO(1)$ as $\epsilon=T^{-1/4}$ in leverage of properties of KL divergence.
We define an event $A$ that in more than $\frac{T}{4}$ rounds, $a_t\ge \frac{40+5\epsilon}{4}$. Therefore, $A^c$ contains at least $\frac{T}{4}$ rounds that $a_t<\frac{40+5\epsilon}{4}$. Therefore,
$$\begin{aligned}
    \E_{[\cdot|T_0<\frac{T}{2}]}\left[\text{Regret}_0(T)+\text{Regret}_\epsilon(T)\right]&\ge \PP_0(A\given T_0<\frac{T}{2})\frac{T}{4}\frac{25\epsilon^2}{32}+\PP_\epsilon(A^c\given T_0<\frac{T}{2})\frac{T}{4}\frac{25\epsilon^2}{16(2+\epsilon)}\\
    &\ge \frac{25\epsilon^2 T}{192}(\PP_0(A\given T_0<\frac{T}{2})+\PP_\epsilon(A^c\given T_0<\frac{T}{2}))\\
    &\ge \frac{25\sqrt T}{192}\frac{1}{2}e^{-\frac{25}{18}}\gtrsim\Omega(\sqrt{T}).
\end{aligned}$$
The first inequality holds because the optimal service fee is $\frac{20+5\epsilon}{2}$ and one-round loss is $\frac{[a_t-(20+5\epsilon)/2]^2}{2+\epsilon}$. The second inequality holds due to $\epsilon\le 1$ while the third inequality holds due to \Cref{lem:KL+ineq}.

Therefore, it holds that by combining the two cases

$$\begin{aligned}
    & \quad \max\{\E\text{Regret}_0(T),\E\text{Regret}_\epsilon(T)\}  \ge \frac{1}{2}(\E\text{Regret}_0(T)+\E\text{Regret}_\epsilon(T))\\
    & \ge\min\{\E_{[\cdot|T_0\ge\frac{T}{2}]}\frac{\text{Regret}_0(T)+\text{Regret}_\epsilon(T)}{2},\E_{[\cdot|T_0<\frac{T}{2}]}\frac{\text{Regret}_0(T)+\text{Regret}_\epsilon(T)}{2}\}\\
    & \ge \frac{1}{2}\min\{\frac{25\sqrt{T}}{4},\frac{25\sqrt T}{384}e^{-\frac{25}{18}}\} \gtrsim\Omega(\sqrt{T}),
\end{aligned}$$
due to $\epsilon=T^{-1/4}$. Therefore, for any algorithm against $\epsilon=0$ and $\epsilon=T^{-1/4}$, at least one regret is no smaller than $\Omega(\sqrt{T})$, which ends the proof.
\subsection{Omitted Proof in Section \ref{sec:adaptive}}
\Cref{ass:eta} is equivalent to the following assumption. We present it only for the convenience of proof.
\begin{assumption}
    Assume that 
    there exists a constant $C$ such that $\eta_S\le C\sigma_S$.
\end{assumption}
\subsubsection{Lemmas Supporting Theorem \ref{thm:upper_ada}}
We first state the following lemma on hypothesis testing. We assume we output $\HH_0$ when the sample variance is smaller than $\frac{2}{\sqrt{T}}$ and $\HH_1$ otherwise.
\begin{lemma}\label{lem:T0}
    With probability at least $1-\frac{1}{2T}$, by setting $T_0=2048C^4\log(8T)$, it holds that
     when $\sigma_S^2\in[0,\frac{1}{\sqrt{T}}]$ will the Hypothesis Test output $\HH_0$. Similarly, when $\sigma_S^2\ge \frac{5}{\sqrt{T}}$, the Hypothesis Test will output $\HH_1$.
\end{lemma}
\begin{proof}{Proof}
    Since the threshold of the empirical variance is $\frac{2}{\sqrt{T}}$, we only need to prove that
    \[
    \frac{1}{T_0}\sum_{t=1}^{T_0}(\epsilon_{St}-\frac{1}{T_0}\sum\epsilon_{St})^2-\sigma_S^2\le\frac{1}{\sqrt{T}}.
    \]
    Note that $\E \epsilon_{St}=0$, it holds that we only need to prove when $\sigma_S^2\in[0,\frac{1}{\sqrt{T}}]$,
    \[
    \frac{1}{T_0}\sum_{t=1}^{T_0}\epsilon_{St}^2-\sigma_S^2+(\frac{1}{T_0}\sum_{t=1}^{T_0}\epsilon_{St})^2\le\frac{1}{\sqrt{T}}.
    \]
    For the first term, we know that $\epsilon_{St}^2-\sigma_S^2$ is $16\eta_S^2$-sub-exponential due to \Cref{lem:square_subE}. Therefore, it holds that with probability at least $1-\frac{1}{4T}$,
    \[
    |\frac{1}{T_0}\sum_{t=1}^{T_0}\epsilon_{St}^2-\sigma_S^2|\le \frac{1}{2\sqrt{T}},
    \]
    when $T_0\ge \max\{2\log(8T)*256\eta_S^4*4T,2\log(8T)*16\eta_S^2*2\sqrt{T}\}$ because of \Cref{lem:subE}.

    For the second term, it holds that when $T_0\ge 4\sqrt{T}\eta_S^2\log (8T)$, with probability at least $1-\frac{1}{4T}$,
    \[
    |\frac{1}{T_0}\sum_{t=1}^{T_0}\epsilon_{St}|\le\frac{1}{2\sqrt{T}}.
    \]
    Here, we use \Cref{lem:hoeffding}.

    Combining these two parts, we know that the error is at most $\frac{1}{\sqrt{T}}$. Since the threshold is $\frac{2}{\sqrt{T}}$, we know that as long as $\sigma_S^2\le\frac{1}{\sqrt{T}}$, it will be categorized correctly. Replacing $\sigma_S^2$ by its upper bound $\frac{1}{\sqrt{T}}$ and using the fact that $\eta_S\le C\sigma_S$, it then holds that when $T_0= 2048C^4\log(8T)$, the statement in \Cref{lem:T0} is correct with probability at least $1-\frac{1}{2T}$.

    Analogously, we know that the lower bound of the empirical variance is at least 
    \[
        \sigma_S^2-\frac{2\eta_S^2\log(8T)}{T_0}-\max\left\{\sqrt{\frac{2*256\eta_S^4\log(8T)}{T_0}},\frac{2*16\eta_S^2\log(8T)}{T_0}\right\}.
    \]
    Considering $C\ge 1$ and $T_0=2048C^4\log(8T)$, it holds that when $\sigma_S^2\ge \frac{5}{\sqrt{T}}$, the empirical variance will be larger than $\frac{2}{\sqrt{T}}$ with probability at least $1-\frac{1}{2T}$ and the test will output $\HH_1$.
\end{proof}

\subsubsection{Proof of Theorem \ref{thm:upper_ada}}
From \Cref{lem:T0}, we know that when $\sigma_S^2\in[0,\frac{1}{\sqrt{T}}]$, the test will output $\HH_0$; when $\sigma_S^2\in[\frac{1}{\sqrt{T}},\frac{5}{\sqrt{T}}]$, the test will output either $\HH_0$ or $\HH_1$; and when $\sigma_S^2\ge\frac{5}{\sqrt{T}}$, the test will output $\HH_1$.

When the outcome of the hypothesis testing is $\HH_1$, we know that from \Cref{thm:partially}, with probability at least $1-\frac{1}{2T}$, the regret of the following $T-T_0$ rounds is at most $\cO(dim\log(dim)+dim\log^2 T+
\frac{\log^3 T}{\sigma_S^2}
    +\frac{\log^2T}{\sigma_S^2(1-\gamma)^2}
    +\frac{\log T}{\sigma_S^2(1-\gamma)^3})$ (since $\sigma_S\simeq\eta_S$).
    For the first $T_0$ rounds, if we set service fee as $\frac{B-\alpha_0-\epsilon_{St}}{2}$, the one-round regret is at most $\cO(B^2)$.
    Therefore, it holds that
    $$
    \begin{aligned}
    \text{Regret}(T) 
    & \lesssim\cO\left(dim\log(dim)+dim\log^2 T + \frac{\log^3 T}{\sigma_S^2} +\frac{\log^2T}{\sigma_S^2(1-\gamma)^2} + \frac{\log T}{\sigma_S^2(1-\gamma)^3}\right)+\cO(T_0) \\
    & \lesssim\cO(\frac{\log^3 T}{\sigma_S^2}).
    \end{aligned}
    $$
The total probability of the success event is at least $1-\frac{1}{T}$.

We now consider the situation when the outcome is $\HH_0$. We know that with probability at least $1-\frac{1}{2T}$, the supply variance $\sigma_S^2\le\frac{5}{\sqrt{T}}$. Given this, we first bound $|\hat\beta-\beta|$. It holds that
$$
\begin{aligned}
\hat\beta-\beta = &\; 
\frac{\frac{1}{n}\sum (f(x_i)-\hat f(x_i))(\epsilon_{Si}+\alpha_0-\hat\alpha_0)}{\frac{1}{n}\sum Q_i^e(\epsilon_{Si}+\alpha_0-\hat\alpha_0-2\epsilon_i)}
+
\frac{\frac{1}{n}\sum \epsilon_{Di}(\epsilon_{Si}+\alpha_0-\hat\alpha_0)}{\frac{1}{n}\sum Q_i^e(\epsilon_{Si}
+\alpha_0-\hat\alpha_0-2\epsilon_i)}
\\
& +
\frac{\frac{1}{n}\sum \xi_{Di}(\epsilon_{Si}+\alpha_0-\hat\alpha_0)}{\frac{1}{n}\sum Q_i^e(\epsilon_{Si}+\alpha_0-\hat\alpha_0-2\epsilon_i)}
+\frac{\frac{1}{n}\sum (f(x_i)-\hat f(x_i))(-2\epsilon_i)}{\frac{1}{n}\sum Q_i^e(\epsilon_{Si}+\alpha_0-\hat\alpha_0-2\epsilon_i)}
\\
& +
{\frac{\frac{1}{n}\sum \epsilon_{Di}(-2\epsilon_i)}{\frac{1}{n}\sum Q_i^e(\epsilon_{Si}+\alpha_0-\hat\alpha_0-2\epsilon_i)}}+
\frac{\frac{1}{n}\sum \xi_{Di}(-2\epsilon_i)}{\frac{1}{n}\sum Q_i^e(\epsilon_{Si}+\alpha_0-\hat\alpha_0-2\epsilon_i)}.
\end{aligned}
$$
From the proof of \Cref{lem:beta_noise}, we know that the sum of the numerators of the first three terms is bounded by $\cO(\sqrt{\frac{\log T}{n\sqrt{T}}})$. Here, we use the fact that $\eta_S^2\le C^2\sigma_S^2\le\frac{5C^2}{\sqrt{T}}$. From the proof of \Cref{lem:beta_no_noise}, it holds that the sum of the numerator of the last three terms is at most $\cO(\sqrt{n^{-3/2}\log T})$.

We now turn to calculate the denominator $\frac{1}{n}\sum Q_i^e(\epsilon_{Si}+\alpha_0-\hat\alpha_0-2\epsilon_i)$. We only need to replace $\epsilon_{Si}$ in the proof of \Cref{lem:beta_noise} by $\epsilon_{Si}-2\epsilon_i$ which is a $\sqrt{\sigma_S^2+4\sigma_i^2}$-sub-Gaussian random variable. Similarly, we know the denominator is at least $\Omega(\sigma_S^2+\frac{1}{\sqrt{n}})$.

Consequently, we know that 
\[
|\hat\beta-\beta|\lesssim\cO(\sqrt{n^{-1/2}\log T}),
\]
with probability $1-\Theta(\frac{1}{T})$.

Second, we employ the same process as the proof of \Cref{thm:no_noise}. It then holds that 
\[
\text{Regret}(T)\lesssim\cO(T_0)+\cO(\sqrt{T}\log^2 T)\lesssim\cO(\sqrt{T}\log^2 T).
\]
By adjusting the constants hidden in the $\cO(\cdot)$, we can achieve the first bound with probability at least $1-\frac{1}{2T}$. Also, we know the probability related to the second term is at least $1-\frac{1}{2T}$. The total success event happens with probability at least $1-\frac{1}{T}$, which ends the proof.

Recall that when $\sigma_S^2\in[\frac{1}{\sqrt{T}},\frac{5}{\sqrt{T}}]$, it can be categorized to either $\HH_0$ or $\HH_1$. So, we summarize the relationship between regret and noise variance $\sigma_S^2$ as follows. With probability at least $1-\frac{1}{T}$, it holds that
\[
\text{Regret}(T)\lesssim\left\{\begin{array}{cc}
\cO(\sqrt{T}\log^2 T) & \text { when } \sigma_S^2\in[0,\frac{1}{\sqrt{T}}] \\
\cO(\sqrt{T}\log^2 T) & \text { when } \sigma_S^2\in[\frac{1}{\sqrt{T}},\frac{5}{\sqrt{T}}]\text{ and } \HH_0 \\
\cO(\sqrt{T}\log^3 T) & \text { when } \sigma_S^2\in[\frac{1}{\sqrt{T}},\frac{5}{\sqrt{T}}]\text{ and } \HH_1 \\
\cO(\frac{\log^3 T}{\sigma_S^2}) & \text { when } \sigma_S^2\in[\frac{5}{\sqrt{T}},\infty).
\end{array}\right.
\]

Notice that when $\sigma_S^2\in[\frac{1}{\sqrt{T}},\frac{5}{\sqrt{T}}]$, adding artificial randomness or not results in slightly different results. It may inspire us to change the threshold of hypothesis testing from $\Theta(\frac{1}{\sqrt{T}})$ to $\Theta(\frac{\log T}{\sqrt{T}})$. However, since the logarithmic terms highly depend on the form of used concentration inequalities, we ignore this little mismatch for clearness. Interested readers may refer to \citet{audibert2009minimax} for a method to remove the extraneous logarithmic factor.

\subsubsection{Proof of Theorem \ref{thm:lower}}
Similarly, we assume that the demand curve is a linear function and function class $\cF$ contains the intercept, namely, $f(x)=\beta_0$. Let $\tilde\beta=(\beta_0, \beta)$ and $\epsilon_{Dt}\sim\mathcal N(0, \sigma_D^2)$ with $\sigma_D=(\alpha_1-\beta)\sigma$ for some $\sigma>0$. To express concisely, we use subscript $\tilde\beta$ and superscript $\pi$ to denote parameter and policy, respectively. We know that
$$
    a_t^*(\tilde\beta) = \frac{\beta_0-\alpha_0-\epsilon_{St}}{2}.
$$
Also, when the decision maker sets any $a_t$, we observe
$$
    Q_t^e=\frac{\beta_0+\epsilon_{Dt}-a_t-\alpha_0-\epsilon_{St}}{\alpha_1-\beta}
$$
Note that the only unknown randomness comes from $\epsilon_{Dt}$ since $\alpha_0+\epsilon_{St}$ is always observable. Meanwhile, upon deciding $a_t$ and observing $Q_t^e$ and $\alpha_0+\epsilon_{St}$, the price from both the demand and supply side can be uniquely decided. Therefore, the log-likelihood prior to time $t$ can be calculated as:
$$
    \mathcal L_{t-1}(\tilde\beta) = \sum_{i=1}^{t-1}-\frac{1}{2\sigma^2}\left(Q_i^e-\frac{\beta_0-a_i-\alpha_0-\epsilon_{Si}}{\alpha_1-\beta}\right)^2 + C,
$$
where $C$ is a constant only dependent on $\{a_i, Q_i^e\}_{i=1}^{t-1}$ and \emph{not} dependent on $\tilde\beta$. The Fisher information matrix prior to time $t$ can be calculated as
$$
    \mathcal I_{t-1}^\pi(\tilde\beta) = \mathbb E_{\tilde\beta}^\pi\left[-\partial^2\mathcal L_{t-1}(\tilde\beta) / \partial \Tilde{\beta}^2\right] = \frac{1}{\sigma^2}\mathbb E_{\Tilde{\beta}}^\pi\left[\sum_{i=1}^{t-1}\left[
    \begin{matrix}
        \frac{1}{(\alpha_1-\beta)^2} & -\frac{\beta_0-a_i-\alpha_0-\epsilon_{Si}}{(\alpha_1-\beta)^2} \\
        -\frac{\beta_0-a_i-\alpha_0-\epsilon_{Si}}{(\alpha_1-\beta)^2} & \frac{\left(\beta_0-a_i-\alpha_0-\epsilon_{Si}\right)^2}{(\alpha_1-\beta)^2}
    \end{matrix}
    \right]\right].
$$

Let $\lambda$ be an absolutely continuous density on $\Theta$, taking positive values on the interior of $\Theta$ and zero on its boundary (see, e.g., \citet{keskin2014dynamic}). Then, the multivariate Van Trees inequality
\citep{gill1995applications} implies that
\begin{equation} \begin{aligned} \label{eq:van-trees}
    \mathbb E_\lambda\left[\mathbb E_{\tilde\beta}^\pi\left[(a_t-a_t^*(\tilde\beta))^2\right]\right] \geq \frac{\left(\mathbb E_\lambda\left[C(\tilde\beta)(\partial a_t^*(\tilde\beta)/\partial \tilde\beta)^\top\right]\right)^2}{\mathbb E_\lambda\left[C(\tilde\beta)\mathcal I_{t-1}^\pi(\tilde\beta)C(\tilde\beta)^\top\right] + {\mathcal I}(\lambda)},
\end{aligned} \end{equation}
where ${\mathcal I}(\lambda)$ is the Fisher information corresponding to $\lambda$.  

Let $C(\tilde\beta) = [(\beta_0-\alpha_0)/2, 1]$. Then we have
$$
    C(\tilde\beta)(\partial a_t^*(\tilde\beta)/\partial \tilde\beta)^\top = [(\beta_0-\alpha_0)/2, 1][1/2, 0]^\top = (\beta_0-\alpha_0)/4 \eqsim \Omega(1)
$$
and
$$
\begin{aligned}
    C(\tilde\beta)\mathcal I_{t-1}^\pi(\tilde\beta)C(\tilde\beta)^\top & = \frac{1}{\sigma^2(\alpha_1-\beta)^2}\mathbb E_{\tilde\beta}^\pi\left[\sum_{i=1}^{t-1}\left(\frac{\beta_0-\alpha_0}{2} - (\beta_0-a_i-\alpha_0-\epsilon_{Si})\right)^2\right] \\
    & = \frac{1}{\sigma^2(\alpha_1-\beta)^2}\sum_{i=1}^{t-1}\mathbb E_{\tilde\beta}^\pi\left[\left(a_i-\frac{\beta_0-\alpha_0-\epsilon_{Si}}{2} + \frac{\epsilon_{Si}}{2}\right)^2\right] \\
    & \lesssim \cO\left(\sum_{i=1}^{t-1}\left(\mathbb E_{\tilde\beta}^\pi\left[(a_i-a_i^*(\tilde\beta))^2\right] + \sigma_S^2\right)\right).
\end{aligned}
$$
Plugging the inequalities above into Inequality \eqref{eq:van-trees}, we know that there exists a positive constant $c$ such that
$$
    \mathbb E_{\lambda}\left[\mathbb E_{\tilde\beta}^\pi\left[(a_t-a_t^*(\tilde\beta))^2\right]\right] \geq \frac{2c}{1+\sum_{i=1}^{t-1}\mathbb E_{\lambda}\left[\mathbb E_{\tilde\beta}^\pi\left[(a_i-a_i^*(\tilde\beta))^2\right]\right] + (t-1)\sigma_S^2}.
$$
From now on, for notation simplicity, we denote $\Delta^{[t, T]}=\sum_{i=t}^T\mathbb E_{\lambda}\left[\mathbb E_{\tilde\beta}^\pi\left[(a_i-a_i^*(\tilde\beta))^2\right]\right]$. Summing the formula above from $t$ to $T$, we know that
\begin{equation} \begin{aligned} \label{eq:van-trees-simplified}
    \Delta^{[t, T]} \geq \sum_{i=t-1}^{T-1}\frac{2c}{1+\Delta^{[1, i]}+\sigma_S^2\cdot i}.
\end{aligned} \end{equation}
Then $\sup_{\tilde\beta}{\text{Regret}}_{\tilde\beta}^\pi(T)\eqsim\Theta(\Delta^{[1, T]})$. We show the lower bound in two cases.

\noindent \underline{(a). $T\leq 1+1/\sigma_S^4$.} Then from Inequality \eqref{eq:van-trees-simplified} we have
$$
    \Delta^{[1, T]} \geq \sum_{i=1}^{T-1}\frac{2c}{1+\Delta^{[1, i]}+\sigma_S^2\cdot i} \geq \frac{2c(T-1)}{1+\Delta^{[1, T]} + \sqrt{T}} \geq \frac{cT}{\Delta^{[1, T]} + 2\sqrt{T}}.
$$
Therefore, 
$$
    (\Delta^{[1, T]}+\sqrt{T})^2 \geq (1+c)T,
$$
which indicates that there exists a constant $c_0=\sqrt{1+c}-1>0$ such that $\Delta^{[1, T]}\geq c_0\sqrt{T}$.

\noindent \underline{(b). $T > 1+1/\sigma_S^4$.} Let $K$ be the smallest positive integer such that $2^K \geq 1+1/\sigma_S^4$. Then $2^{K-1} < 1 + 1/\sigma_S^4$, which indicates that
$$
    K+1 = K-1 + 2 < \log_2(1+1/\sigma_S^4) + 2 \leq 6(1+\log _+(1/\sigma_S)),
$$
where we use $\log_+(\cdot)$ to denote $\max\{0,\log(\cdot)\}$.

From (a), we know that for $t=2^K$, we have
$$
    \Delta^{[1, t]} \geq c_0 \sqrt{1+1/\sigma_S^4} \geq c_0\frac{1}{\sigma_S^2} \geq \frac{c_1}{1+\log _+(1/\sigma_S)}\frac{K}{\sigma_S^2},
$$
where $c_1$ is a small positive constant irrelevant with $K$ and $\sigma_S$ such that $c_1(2+6c_1) \leq c$.

Now we use induction. Suppose we have
\begin{equation} \begin{aligned} \label{eq:lower-noise}
    \Delta^{[1, 2^k]} \geq \frac{c_1}{1+\log _+(1/\sigma_S)}\frac{k}{\sigma_S^2}
\end{aligned} \end{equation}
for some $k\geq K$. From Inequality \eqref{eq:van-trees-simplified} we have
$$
    \Delta^{[2^k+1, 2^{k+1}]}\geq \frac{(2^{k+1}-2)c}{1+\Delta^{[1, 2^{k+1}]}+\sigma_S^2\cdot (2^{k+1}-1)} \geq \frac{c}{2^{-k}+\frac{\Delta^{[1, 2^{k+1}]}}{2^k} + \frac{3}{2}\sigma_S^2}\geq \frac{c}{\frac{\Delta^{[1, 2^{k+1}]}}{2^k} + 2\sigma_S^2}.
$$
Note that we have utilized the following inequality:
$$
    2^{-k} \leq \frac{\sigma_S^4}{1+\sigma_S^4}\leq \frac{\sigma_S^4}{2\sigma_S^2} = \sigma_S^2 / 2.
$$
As a result, 
$$
    \Delta^{[1, 2^{k+1}]}\geq \frac{c_1}{1+\log _+(1/\sigma_S)}\frac{k}{\sigma_S^2} + \frac{c}{\frac{\Delta^{[1, 2^{k+1}]}}{2^k} + 2\sigma_S^2}.
$$
If $\Delta^{[1, 2^{k+1}]} < \frac{c_1}{1+\log _+(1/\sigma_S)}\frac{k+1}{\sigma_S^2}$, then
$$
    \frac{c}{\frac{\Delta^{[1, 2^{k+1}]}}{2^k} + 2\sigma_S^2}  > \frac{c}{\frac{c_1}{1+\log _+(1/\sigma_S)}\frac{k+1}{2^k\sigma_S^2} + 2\sigma_S^2} \geq \frac{c}{\frac{c_1}{1+\log _+(1/\sigma_S)}\frac{K+1}{2^K\sigma_S^2} + 2\sigma_S^2}  \geq \frac{c}{(2+6c_1)\sigma_S^2}  \geq \frac{c_1}{\sigma_S^2}.
$$
This causes a contradiction. Therefore, for any $k\geq K$, Inequality \eqref{eq:lower-noise} holds. Now select any $T > 1+ 1/\sigma_S^4$. If $T< 2^K$, then
$$
    \Delta^{[1, T]}\geq c_0\sqrt{1+1/\sigma_S^4} \geq \frac{c_1}{1+\log _+(1/\sigma_S)}\frac{K}{\sigma_S^2} \geq \frac{c_1}{1+\log _+(1/\sigma_S)}\frac{\log  T}{\sigma_S^2}.
$$
If $T\geq 2^K$, choose $k=\lfloor \log_2 T\rfloor$, then
$$
    \Delta^{[1, T]} \geq \Delta^{[1, 2^k]} \geq \frac{c_1}{1+\log _+(1/\sigma_S)}\frac{k}{\sigma_S^2} \geq \frac{c_1/2}{1+\log _+(1/\sigma_S)}\frac{\log  T}{\sigma_S^2}.
$$
Combining all these cases finishes the proof.

\section{Omitted Proof in Section \ref{sec:practice}}
\subsection{Omitted Proof in Section \ref{sec:linear}}
\subsubsection{Proof of Theorem \ref{thm:linear_class}}
Note that $x_t\in\cX=[0,1]^d$, when $\|\theta_f\|\le \frac{B}{\sqrt{d}}$, it holds that
\[
|f(x)|\le\|x\|*\|\theta_f\|\le\sqrt{d}*\frac{B}{d}=B,
\]
because of $f(x)=\langle x,\theta_f\rangle$ and the triangle inequality.

For the effective dimension, let's consider the following set 
\[
\cS=\left\{-\lceil \frac{B\sqrt{d}}{\epsilon}\rceil*\frac{\epsilon}{d},...,-\frac{\epsilon}{d},0,\frac{\epsilon}{d},\frac{2\epsilon}{d},...,\lceil \frac{B\sqrt{d}}{\epsilon}\rceil*\frac{\epsilon}{d}\right\}.
\]
We restrict all entries of $\theta_f$ can only be a lattice point in $\cS$. Therefore, the cardinality is at most $(1+2\lceil \frac{B\sqrt{d}}{\epsilon}\rceil)^d$. Then, for every unrestricted $\theta_f$, there exists a restricted $\tilde\theta_f$ such that the distance between their $i$-th components is at most $\frac{\epsilon}{d}$. Therefore, we have $\|\langle x,\theta_f\rangle-\langle x,\tilde\theta_f\rangle\|_\infty\le d*\frac{\epsilon}{d}=\epsilon$. It yields that $\log(N(\cF,\|\cdot\|_\infty,\epsilon))\lesssim \cO(d\log(1/\epsilon))$.

For the lower bound, notice that for a $d$-dimensional ball with radius $r$, the volume is proportional to $r^d$. Therefore, we know that $N(\cF,\|\cdot\|_\infty,\epsilon)\gtrsim\Omega((1/\epsilon)^d)$, finishing our proof.

We remark that this is the worst-case effective dimension. When $\cX=[0,1]^d$, this bound is tight. However, when $\cX$ has volume 0 in this $d$-dimensional space, the effective dimension may be smaller than $d$.

\subsection{Omitted Proof in Section \ref{sec:MLP}}
\subsubsection{Proof of Theorem \ref{thm:MLP}}
Since $x_t\in[0,1]^d$, we know that $\|x_t\|_\infty\le 1$. Then, it holds that $\|h_1(x_t)\|_\infty=\|W_1x_t+b_1\|_\infty\le d_0b+b$. Recall that $ReLU$ is a truncation towards zero, so $\|ReLU\circ h_1(x_t)\|_\infty\le (d_0+1)b\le (W+1)b$. Consequently, we've contained the upper bound for 1-layer networks. Assuming for $l$-layer networks, we have $\|h_l\circ ReLU\circ\cdots\circ ReLU\circ h_1(x_t)\|_\infty\le (W+1)^lb^l$, it then holds that for a $(l+1)$-layer network, $\|h_{l+1}\circ ReLU\circ h_l\circ ReLU\circ\cdots\circ ReLU\circ h_1(x_t)\|_\infty\le d_l*b*(W+1)^lb^l+b\le (W+1)^{l+1}b^{l+1}$. We hence conclude that for a $L$-layer network, it holds that $\|f\|_\infty=\|h_L\circ ReLU\circ\cdots\circ ReLU\circ h_1(\cdot)\|_\infty\le (W+1)^Lb^L\le B$.

As for the covering number, when $\epsilon\in(0,\frac{1}{2})$, we know from Theorem 2.1 in \citet{ou2024covering} that $\log(N(\cF,\|\cdot\|_\infty,\epsilon))\le 30W^2 L\log(\frac{B}{\epsilon})$. Note that all functions in $\cF$ are bounded by $B$, when $\epsilon\ge 2B$, the covering number will be one. When $\epsilon\in[\frac{1}{2},2B]$, it holds that $\log(N(\cF,\|\cdot\|_\infty,\epsilon))\le \log(N(\cF,\|\cdot\|_\infty,\frac{1}{4}))\le 30W^2L\log(\frac{8B^2}{\epsilon})$. Therefore, we know that $dim\le 30W^2L$.

When both $W$ and $L$ are no less than 60, we know that from Theorem 2.1 in \citet{ou2024covering}, $\log(N(\cF,\|\cdot\|_\infty,\epsilon))\ge \frac{1}{92160000}W^2 L\log(\frac{B}{\epsilon})$. Hence, $dim\ge\frac{W^2L}{92160000}$, yielding $dim\eqsim\Theta(W^2L)$, and this finishes the proof.
\section{Details of Experiments}
\subsection{Simulation Setup}\label{app:exp}
In the simulation, we choose the supply curve $P_{St}=Q+\epsilon_{St}$, where $\epsilon_{St}\sim\cN(0,\sigma_S^2)$. We randomly choose a $f\in\cF$ and set $P_{Dt}=f(x_t)-Q+\epsilon_{Dt}$, where $\epsilon_{Dt}\sim\cN(0,1)$. To initialize $f$, we randomly sample from $\cN(0,1)$ for the weights of the first two layers. To guarantee the cutoff price is positive most of the time, we sample the last-layer weights from $\text{Unif}(0,1)$ without bias. Moreover, we sample $x_t$ uniformly between $[0,1]$ along each dimension. 
For the \texttt{Adam}, we choose the learning rate to be 0.001 and optimize 100 epochs every time.
Since the misreport only brings exponentially small nuisances, we set $\gamma=0$ and focus on the relationship between regret and supply noise level. 
We also test the case with $\gamma = 0.9$. Following the proof of \Cref{prop:bound_strategy}, we add a perturbation $\xi_{Dt} \sim \text{Unif}(-\gamma^{s/2},\gamma^{s/2})$ to the demand, and obtain similar results as shown in \Cref{fig:gamma}.
Finally, we set the random seed to 42 to ensure reproducibility.

\begin{figure}[!ht]
\centering
\begin{minipage}[t]{0.465\textwidth}
\centering
    \includegraphics[width=\linewidth]{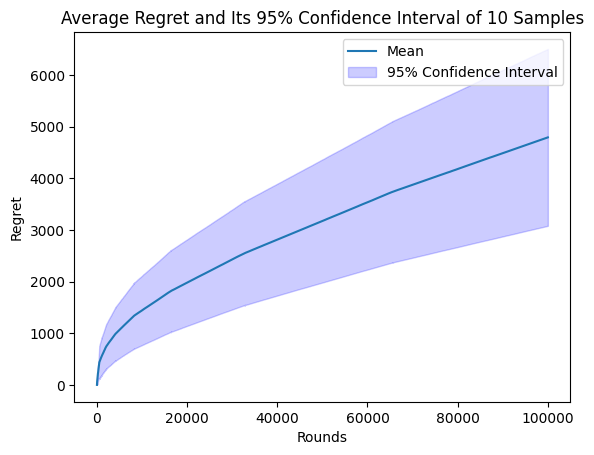}
\end{minipage}
\hspace{0.5cm}
\begin{minipage}[t]{0.485\textwidth}
\centering
    \includegraphics[width=\linewidth]{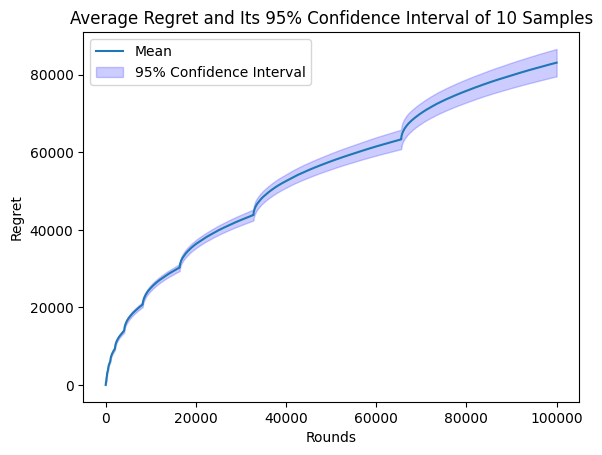}
\end{minipage}
\caption{Regret of \Cref{alg:AaaIV} when $\gamma=0.9$. Left: $\sigma_S^2 = 1$; Right: $\sigma_S^2 = 0$. }
\label{fig:gamma}
\end{figure}

In the first simulation (c.f. \Cref{fig:log,fig:sqrt}), since we know $\sigma_S^2$ in advance, there is no need to do hypothesis testing. Hence, we only use 10 rounds to warm up. Later, when we update our service fee policy, we need at least 100 data points to preserve stability. We use the fact that $\beta$ is negative, and we use zero to truncate it when the estimated $\hat\beta$ is positive.

In the second simulation (c.f. \Cref{fig:phase_sigmaS}), we now set $T=10{,}000$ and, as we don't know $\sigma_S^2$ now, we set $T_0=100$, which is around $10\log(T)$, to do hypothesis testing. When the empirical variance is larger than $\frac{10}{\sqrt{T}}$, we output $\HH_1$ and otherwise $\HH_0$. When categorized in $\HH_0$, we add artificial randomness sampled from $\cN(0,\frac{10^2}{\sqrt{\text{size}(\cD)+1}})$. 
Consistent with the high-probability guarantee, 198 of the 200 simulated trajectories (99\%) exhibit the predicted phase-transition pattern. The two exceptional trajectories, at $\sigma_S^2=0.1560$ and $\sigma_S^2=0.2704$ with regrets 56435.92 and 75731.14, are omitted from the visualization to preserve the scale of the typical trajectories.
Besides, for the LOWESS, we choose moderate 0.3 as the hyperparameter, namely, fraction. 
This follows the standard LOWESS convention in \citet{cleveland1979robust}.

\subsection{Empirical Study Setup}\label{app:emp}
The Talabat dataset contains individual orders in 21 Egyptian cities, including Al Mahallah Al Kubra, Alexandria, Assiut, Banha, Cairo, Damanhour, Damietta, Damietta New, Hurghada, Ismailia, Mansoura, Minya, Port Said, Ras El Bar, Shebeen El Koom, Suez, Suhag, Talkha, Tanta, Zagazig and ``Other''. The dataset includes both successful and failed orders. In our study, we retain only successful orders and exclude failed ones.
We gather over hours the successful orders to simulate the market transacted $Q_t^e$. We use the amount a customer pays for an order as the price and the difference between it and the total value of the goods included in the order, including taxes, before any discounts, as the service fee.
To account for occasional extreme values in real-world transaction data, we approximate both the price and the service fee used in our analysis by averaging them within each hour.
Note that the service fee includes a platform fee (as of May 2026, E\pounds8 per order), delivery partner fee, restaurant charges, and any applicable coupons. We also use simple averages to construct features. For example, if in one hour, 30\% of orders are in Alexandria while 70\% are in Cairo, the Alexandria dummy will be 0.3. Although Talabat is a food delivery app, it actually contains seven kinds of items, including pharmacy, cosmetics, electronics, flowers, grocery, pet shop and food (constituting 83.47\%). We also include dummy variables indicating discounts and first-time customers. Besides, we control for the delivery time and drop-off distance. Finally, we remove observations with missing values.

We use the service fee as an instrument to estimate the demand curve. For ordinary least squares, the estimated $\hat\beta$ is -0.02 with p-value 0.81. It is unreliable, highlighting the importance of using actions as instrumental variables in demand learning. Moreover, to calibrate the influence of the service fee on quantity, we conduct a linear regression using the test set. It serves as a robustness check for our algorithm under a potential distribution shift between the training and test sets. Since, for example, the platform fee is fixed, we regard the confounding to be negligible, as little strategic pricing is observed in the dataset between service fees and transacted quantities. The estimated result is $Q_t^e=-0.12 a_t+\text{controls}$, and we also perform a calibration using the residuals. This aligns with the economic fact that higher service fees are shared between sellers and buyers. 

Since we do not have access to supply data, we construct an offline, model-based counterfactual within the maintained linear market-clearing model. We calibrate the supply intercept to match the observed test-set average quantity (28.92) and choose counterfactual service fees to maximize model-implied revenue subject to the observed average service fee (E\pounds4.90). Under this linear specification, market-clearing quantity is affine in the service fee, so fixing the average fee also fixes average quantity. The exercise therefore isolates the redistributive component of the fitted model, but does not isolate revenue causally attributable to demand learning.


Under the fitted model, average hourly counterfactual revenue is approximately 600 rather than approximately 100 in the observed data.
More specifically, our algorithm raises the service fee during peak demand hours such as dinnertime, which is consistent with our theoretical framework and thereby improves overall revenue. One potential explanation is that Talabat might be competing for customer traffic, so revenue may not be the platform’s primary concern during that period. This may also account for the relatively low proportion of the service fee. 
Moreover, variations in supply, such as increased availability during dinner hours, may affect service fees; ignoring such factors could lead to an overestimation of our algorithm's impact.
Also, our definition of service fee includes both the platform’s revenue and a portion of the restaurant’s income, which may lead to an overestimation of the algorithm’s performance due to the underlying allocation.

\section{Additional Experiments} \label{sec:add-exp}

\subsection{Ablation: How Doubly-Robust Machine Learning Helps}\label{app:ablation}

In \Cref{sec:method}, we show that we can use either \Cref{eq:IV}  $\E [(P_{t}^e-\beta Q_t^e-f(x_t))a_t]=0$ or \Cref{eq:IV_robust}   $ \E [(P_{t}^e-\beta Q_t^e-f(x_t))(a_t-\E[a_t\given x_t])]=0$  as the structural function. The latter method utilizes the idea of so-called doubly-robust machine learning that, as long as one estimate is asymptotically correct, the estimation of $\hat\beta$ will be accurate. Here, we set $T=10,000$, so we use 4096 samples to estimate $\hat\beta$ in the last update. The true value of $\beta$ is -1. 


For the ablation study (c.f. \Cref{fig:beta_doubly_no,fig:beta_action}), for each setup, we sample 1000 estimates $\hat\beta$ to gather the empirical distribution. We choose $\sigma_S^2=1$ and use 10 rounds to warm up.
When we draw the histogram of $\hat\beta$, we use 50 bins. For the kernel density estimation, we directly use the default \texttt{sns.kdeplot} without further fine-tuning. We also test the case using zero truncation but without doubly-robust machine learning (c.f. \Cref{fig:beta_action_trun}) and find that the empirical distribution is biased towards zero and inferior to the one with doubly-robustness.


\begin{figure}[!ht]
\begin{minipage}[t]{0.48\textwidth}
\centering
    \includegraphics[width=\linewidth]{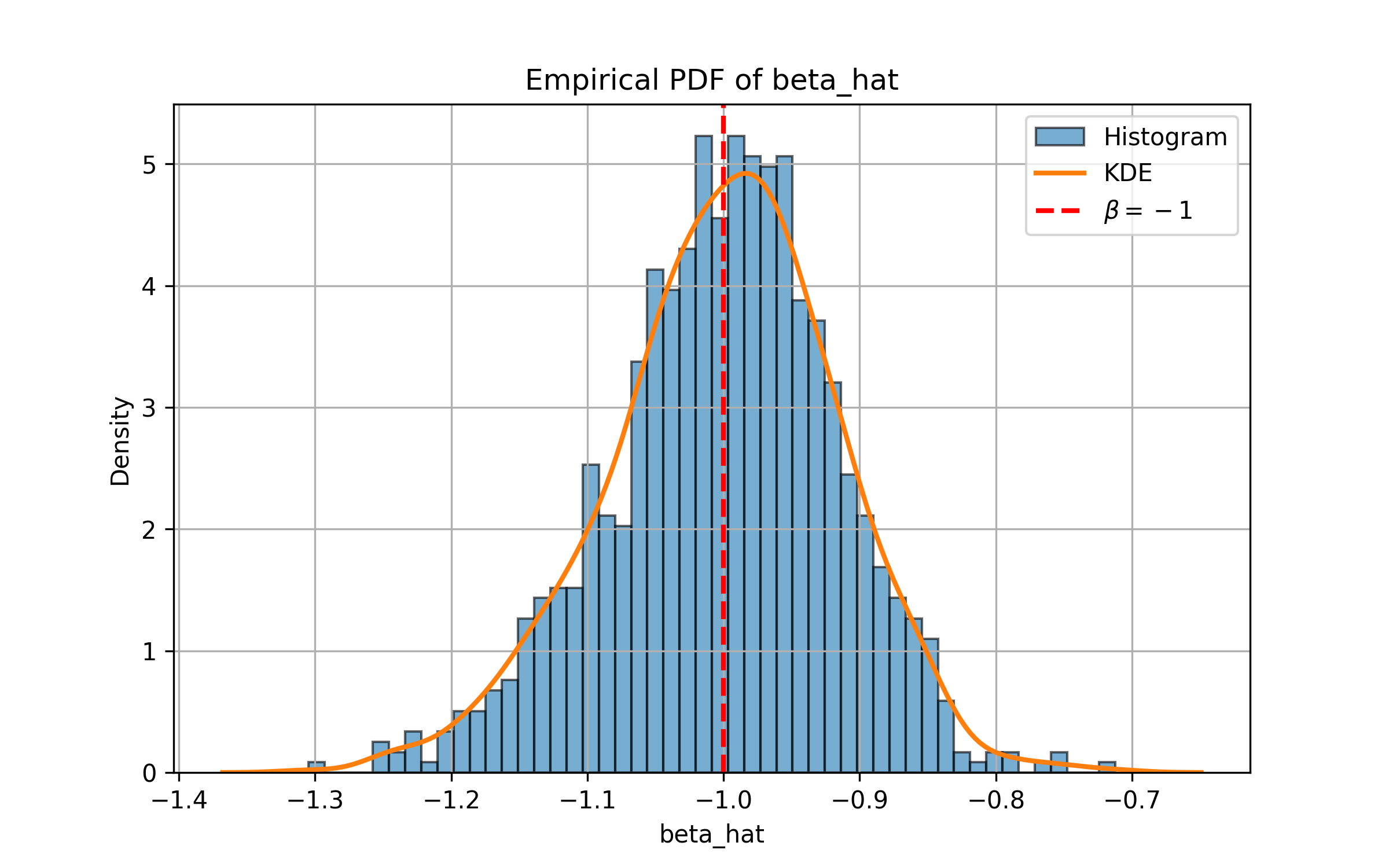}
    \caption{Estimating $\hat\beta$ with doubly-robust ML.}
    \label{fig:beta_action}
\end{minipage}
\hspace{0.5cm}
\begin{minipage}[t]{0.48\textwidth}
\centering
    \includegraphics[width=\linewidth]{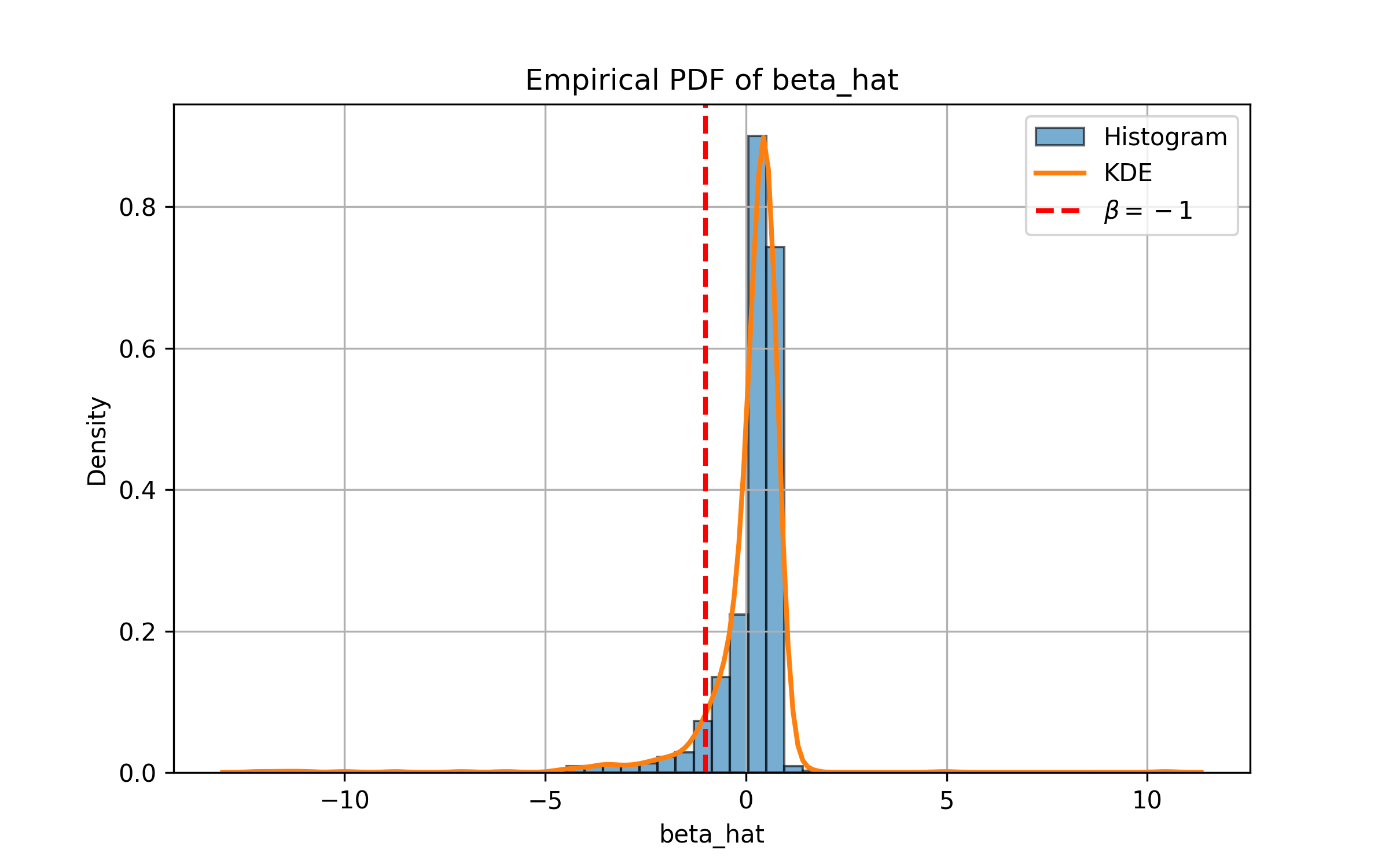}
    \caption{Estimating $\hat\beta$ without doubly-robust ML.}
    \label{fig:beta_doubly_no}
\end{minipage}
\end{figure}

We visualize the empirical distribution of $\hat\beta$ using doubly-robust machine learning in \Cref{fig:beta_action} and conduct the kernel density estimation (KDE). It shows that the empirical distribution is approximately a normal distribution centered at the true value -1, aligned with our theoretical analysis.  
Without the doubly-robust correction, the estimation of $\beta$ is unstable and biased.
As shown in \Cref{fig:beta_doubly_no}, the mode of the empirical $\hat\beta$ is even positive. Hence, in practice, adopting doubly-robust machine learning in demand learning is necessary and quite beneficial. 

\begin{figure}[!ht]
\begin{minipage}[t]{0.48\textwidth}
\centering
    \includegraphics[width=\linewidth]{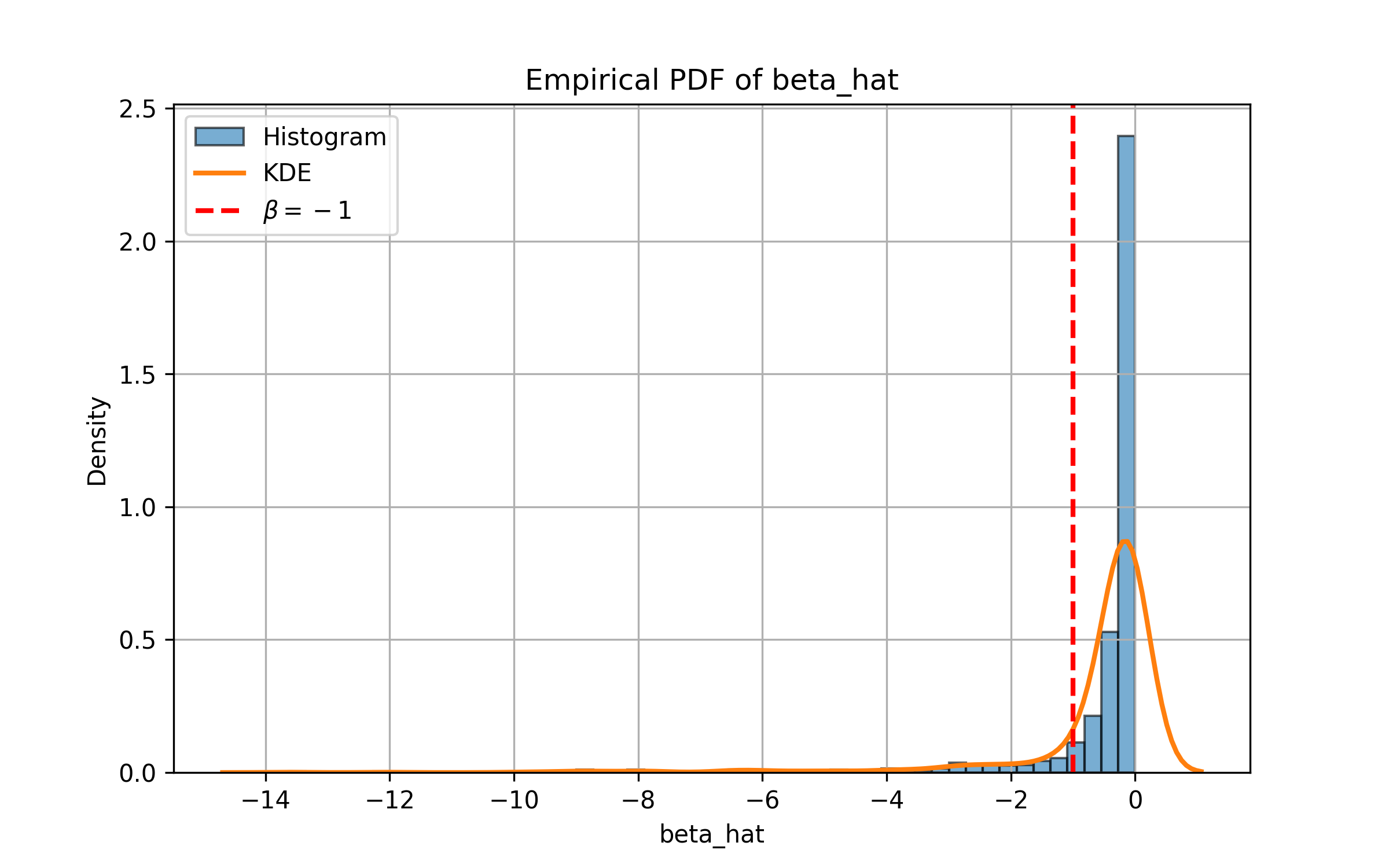}
    \caption{Estimating $\hat\beta$ with zero truncation but without doubly-robust ML.}
    \label{fig:beta_action_trun}
\end{minipage}
\hspace{0.5cm}
\begin{minipage}[t]{0.48\textwidth}
\centering
    \includegraphics[width=\linewidth]{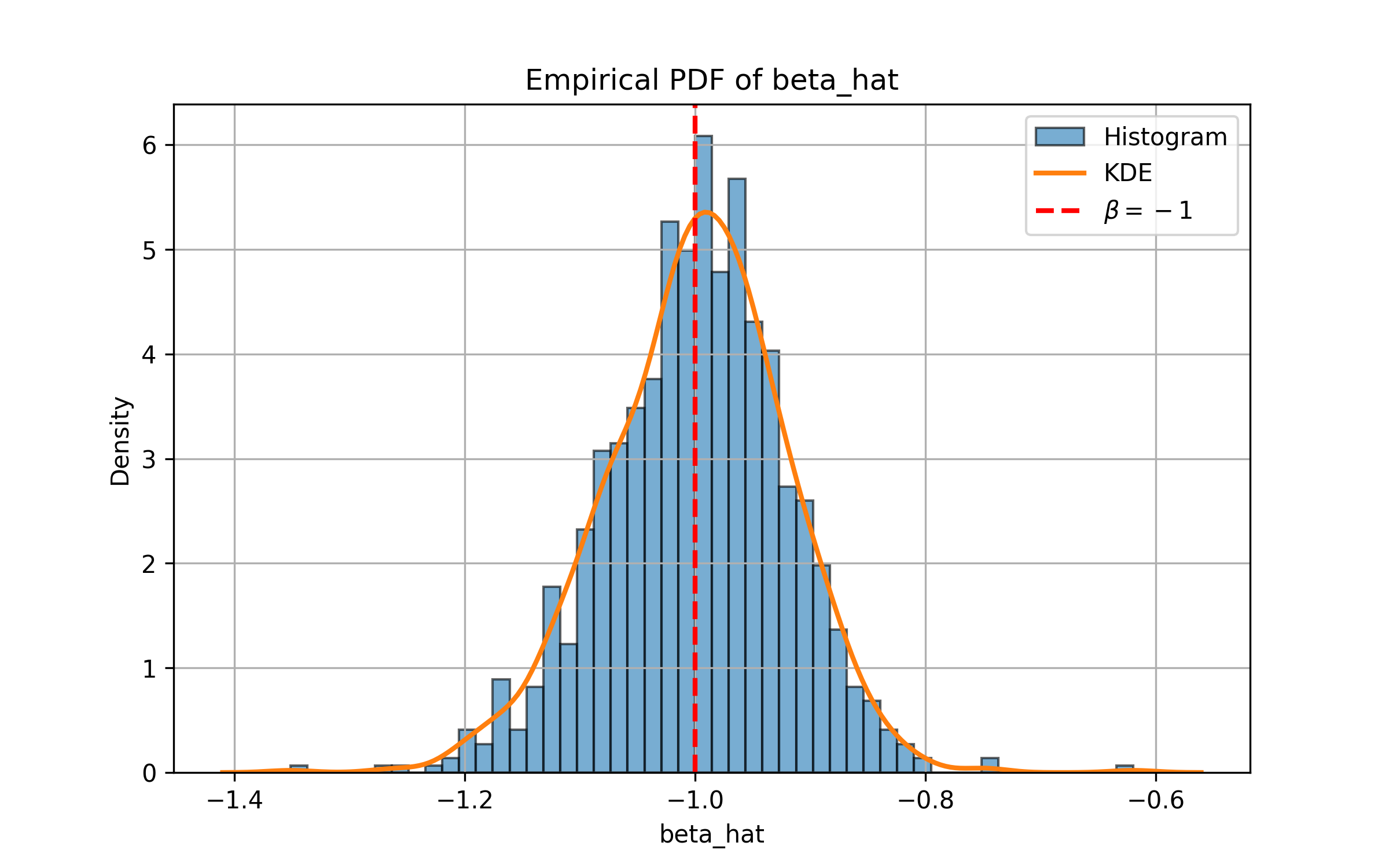}
    \caption{Estimating $\hat\beta$ with doubly-robust ML and zero truncation.}
    \label{fig:beta_doubly_trun}
\end{minipage}
\end{figure}

We also conduct another ablation study on the sign of $\beta$. Since we know $\beta$ is negative, we can use 0 as a truncation for $\hat\beta$. However, as we notice that the doubly-robust estimator is quite accurate, none of them is positive. 
Thus, adopting zero truncation in \Cref{fig:beta_doubly_trun} has little effect on the inference.
Although there are some theoretical assumptions on the network, we actually don't limit the parameters of the network used. In the meantime, 
we replace ERM with \texttt{Adam}, a more practical optimization algorithm compatible with common deep-learning libraries.
We observe strong performance, demonstrating the broad applicability of deep neural networks in real-world dynamic pricing.


\subsection{Empirical Study on Ride Hailing Platforms}\label{sec:ride}
Finally, we use Lyft data from November 26, 2018, to December 18, 2018, in Boston, Massachusetts\footnote{\url{https://www.kaggle.com/datasets/brllrb/uber-and-lyft-dataset-boston-ma}.}, to provide a deep learning workflow with the idea in \Cref{sec:ext}. Ride-hailing platforms typically also charge a booking fee or other forms of service fees per trip as their revenue mechanism. We aggregate the number of orders with the same origin and destination on an hourly basis, for example, from Haymarket Square to North Station, as a proxy for the market total volume of transactions. We use the weather information included in the dataset as features for demand prediction. Since many of the weather descriptions are in textual form, simply using dummy variables would not yield good results, as terms like Drizzle and Light Rain are clearly more similar to each other than to Clear. We first embed the weather descriptions using GloVe~\citep{pennington2014glove}, and then use a Transformer architecture~\citep{vaswani2017attention,joseph2021pytorch} to model the function $f(Q,x_t)$. 

For the Lyft study, we treat it as a time series to avoid information leakage in the test set. We group the data based on time and trip origin, and use group means to construct features. 
The data report the realized trip fare and surge multiplier. We recover the implied base fare by dividing the realized fare by the surge multiplier, and define the service-fee proxy as the difference between the realized fare and this recovered base fare.
We don't use Uber data as it doesn't contain this type of multiplier. For textual information, we use a 25-dimensional GloVe embedding, resulting in a 154-dimensional vector $x_t$. In the first stage, we use a transformer to learn how $x_t$ and the service fee affect the transacted quantity on the training set, namely $F(Q_t^e|x_t,a_t)$. In the second stage, we learn to predict the expected price $f(Q,x_t)$ leveraging the estimated quantity along with $x_t$. Finally, we use a third transformer to predict how the transacted quantity changes on the test set after adjusting the service fee. 
Similar to the Talabat instance, due to the lack of supply-side information and the presence of goals beyond revenue, we keep the average service fee unchanged and report an illustrative model-implied counterfactual from the nonlinear workflow.
Because of the non-linearity of transformers, the average $Q_t^e$ will now change with the linear approximation, compared to the original.

For each transformer, we use the following hyperparameters for training. The input embedding dimension is set to 32, and the embedding weights are initialized using the He initialization method. We apply dropout with a rate of 0.1 to the attention weights, the residual connections after layer normalization, and the feedforward network to prevent overfitting. The hidden dimension of the feedforward network is set to be four times the embedding size, and we set the batch size to be 32. We show the total parameters of the used neural networks in \Cref{tab:para}.

\begin{table}[htbp]
\centering
\begin{minipage}{0.52\textwidth}
\centering
\begin{tabular}{c|cccc}
\toprule
 head\textbackslash block & 2 & 4 & 6 & 8 \\
\midrule
2  & 41.6  &  82.9  & 124 &  165 \\
4  & 58.0 &  115 & 173  & 230 \\
8  & 90.7 & 181 &  271  &  362  \\
16 & 156 & 312  & 468  &   624 \\
\bottomrule
\end{tabular}
\caption{Total parameters of different network structures (in K).}
\label{tab:para}
\end{minipage}\hfill
\begin{minipage}{0.48\textwidth}
\centering
\begin{tabular}{c|cccc}
\toprule
 head\textbackslash block & 2 & 4 & 6 & 8 \\
\midrule
2  & 7.66 & 7.75 & 7.22 & 7.04 \\
4  & 7.19 & 7.53 & 7.58 & 7.53 \\
8  & 7.16 & 7.68 & 7.56 & 7.63 \\
16 & 7.23 & 7.09 & 7.21 & 7.58 \\
\bottomrule
\end{tabular}
\caption{Average revenue with different network structures.}
\label{tab:attention}
\end{minipage}
\end{table}

Besides the concern raised in the Talabat example, there are two additional concerns in this experiment. First, since the number of parameters far exceeds the number of samples, training may suffer from overfitting or get stuck in local optima. Second, the surge may not be a valid instrumental variable. For example, when there are sporting events in Boston, such as games by the Celtics or the Red Sox, certain areas experience not only a reduction in driver supply due to traffic congestion but also an increase in rider demand. As a result, the service fee we use may act as a confounder, which could explain the limited improvement in revenue.

We have a total of 23,540 data points. We use the first 80\% of the data as the training set and the remaining 20\% as the test set. We experiment with various network architectures, with the number of attention blocks ranging from 2 to 8, and the number of attention heads ranging from 2 to 16. The total number of parameters ranges from 41.7 K to 624 K. For a fair comparison, we also fix the average service fee. Please refer to \Cref{tab:attention} for the average revenue of our algorithm on the test set. The average revenue in the real data is 7.09. 
Among the 16 network architectures tested, 15 yield higher model-implied counterfactual revenue, with relative differences ranging from -0.65\% to 9.37\%.
Compared to the Talabat case, the model-implied Lyft differences are more modest, consistent with a mature U.S. ride-hailing market in which existing surge pricing already performs substantial dynamic adjustment. The exercise complements the Talabat analysis by illustrating how the nonlinear workflow operates in a more developed pricing environment.

Thus, a typical dynamic service fee pricing workflow involves first encoding various types of market information, such as images, text, and structured data, followed by using the service fee as an instrument to model the market demand curve with a neural network. Finally, based on the observed supply-side information, the service fee is adjusted to maximize revenue.

\end{document}